\documentclass[12pt]{article}
\usepackage{amsmath}
\usepackage{times}
\usepackage{graphicx}
\usepackage{color}
\usepackage{multirow}
\usepackage[authoryear]{natbib}
\usepackage{rotating}
\usepackage{bbm}
\usepackage{latexsym}

\usepackage{mathrsfs}
\usepackage{natbib}
\usepackage{mathtools}
\usepackage{tikz,subfig}
\usepackage{setspace}
\usepackage{hyperref}
\usetikzlibrary{tqft}
\usepackage{tikz-cd}
\usepackage{tikz,subfig}
\usepackage{tqft}
\usetikzlibrary{shapes.geometric,tqft}
\usepackage{graphicx,import}
\usepackage{framed}
\usepackage{outlines}

\usepackage{setspace}
\usepackage{amssymb}
\usepackage{mathtools}
\usepackage{tikz,subfig}
\usepackage{bm}
\usepackage{hyperref}

\usepackage{amsthm} % For math fonts, symbols and environments\
\usepackage{bm}
\usepackage{cleveref}

\definecolor{SAEblue}{rgb}{0, .62, .71}
\definecolor{bulgarianrose}{rgb}{0.28, 0.02, 0.03}
\definecolor{uablue}{rgb}{0.0, 0.2, 0.67}
\hypersetup{colorlinks, citecolor=blue, linkcolor=black, urlcolor=black}

\usepackage{outlines}
  \newcounter{cenumi}
  \newcounter{cenumisaved}
  \newcommand{\labelcenumi}{\arabic{cenumi}.}
    {\begin{list}{\labelcenumi}{\usecounter{cenumi}\partopsep=4pt\topsep=4pt\itemsep=4pt\parsep=4pt}%
    \setcounter{cenumi}{\value{cenumisaved}}}%
      {\setcounter{cenumisaved}{\value{cenumi}}%
    \end{list}}

  \newcounter{cenumii}
  \newcommand{\labelcenumii}{(\alph{cenumii})}
    {\begin{list}{\labelcenumii}{\usecounter{cenumii}\partopsep=0pt\topsep=0pt\itemsep=0pt\parsep=0pt}%
    }%
      {%
    \end{list}}

\newcommand{\BlackBox}{\rule{1.5ex}{1.5ex}}  % end of proof
\ifdefined\proof
    \renewenvironment{proof}{\par\noindent{\bf Proof\ }}{\hfill\BlackBox\\[2mm]}
\else
    \newenvironment{proof}{\par\noindent{\bf Proof\ }}{\hfill\BlackBox\\[2mm]}
\fi
\newtheorem{example}{Example} 
\newtheorem{theorem}{Theorem}
\newtheorem{lemma}[theorem]{Lemma} 
\newtheorem{proposition}[theorem]{Proposition} 
\newtheorem{remark}[theorem]{Remark}
\newtheorem{corollary}[theorem]{Corollary}
\newtheorem{definition}[theorem]{Definition}

\newcommand\colorAutoref[1]{{\hypersetup{linkcolor=blue}\autoref{#1}}}  %% allows for calls to \autoref{} that have a different color from other links.

\Crefname{algocf}{Algorithm}{Algorithms}
\usepackage{enumitem}

\usepackage[ruled,vlined]{algorithm2e}

\usepackage{float}
\usepackage{wrapfig}

\usepackage{enumitem}

\newcommand{\captionfonts}{\normalsize}

\makeatletter  
\long\def\@makecaption#1#2{%
  \vskip\abovecaptionskip
  \sbox\@tempboxa{{\captionfonts \textbf{#1}: #2}}%
  \ifdim \wd\@tempboxa >\hsize
    {\captionfonts \textbf{#1}: #2\par}
  \else
    \hbox to\hsize{\hfil\box\@tempboxa\hfil}%
  \fi
  \vskip\belowcaptionskip}
\makeatother

\begin{document}
\hspace{13.9cm}1

\begin{center}
\vspace{20mm}

\begin{spacing}{1.5}
    {\Large Associative Memory and Generative Diffusion in the Zero-noise Limit} \\[4mm]
\end{spacing}

{\large  Joshua Hess$^{1,2}$, Quaid Morris$^{1}$} \\[2mm]
\end{center}
$^1$Computational \& Systems Biology, 
       Sloan Kettering Institute, Memorial Sloan Kettering Cancer Center, New York, NY, USA. \\
$^2$Tri-Institutional Training Program in Computational Biology \& Medicine, \\
       New York, NY, USA. \\
\ \\[-2mm]

\thispagestyle{empty}
\markboth{}{NC instructions}
\ \vspace{-0mm}\\
%
%Abstract
\begin{center} {\bf Abstract} \end{center}

This paper shows that generative diffusion processes converge to associative memory systems at vanishing noise levels and characterizes the stability, robustness, memorization, and generation dynamics of both model classes. Morse-Smale dynamical systems are shown to be universal approximators of associative memory models, with diffusion processes as their white-noise perturbations. The universal properties of associative memory that follow are used to characterize a generic transition from generation to memory as noise diminishes. Structural stability of Morse-Smale flows -- that is, the robustness of their global critical point structure -- implies the stability of both trajectories and invariant measures for diffusions in the zero-noise limit. The learning and generation landscapes of these models appear as parameterized families of gradient flows and their stochastic perturbations, and the bifurcation theory for Morse-Smale systems implies that they are generically stable except at isolated parameter values, where enumerable sets of local and global bifurcations govern transitions between stable systems in parameter space. These landscapes are thus characterized by ordered bifurcation sequences that create, destroy, or alter connections between rest points and are robust under small stochastic or deterministic perturbations. The framework is agnostic to model formulation, which we verify with examples from energy-based models, denoising diffusion models, and classical and modern Hopfield networks. We additionally derive structural stability criteria for Hopfield-type networks and find that simple cases violate them. Collectively, our geometric approach provides insight into the classification, stability, and emergence of memory and generative landscapes.

\newpage
\setcounter{tocdepth}{2}
% Black links for TOC
% \hypersetup{linkcolor=black}
% Redefine TOC link color to black using \AtBeginDocument
\hypersetup{linkcolor=black} % seems to allow for getting the TOC to be black then links back to blue
% \newpage
\tableofcontents
\hypersetup{linkcolor=blue} % controls footnote color and section referencing

%---------------------------------------
% Terminology 
%---------------------------------------
\subsection*{Terminology}
\label{subsection:terminology}
By \textit{manifold} we mean a closed smooth finite-dimensional manifold $M$. By \textit{smooth} we mean $C^{\infty}$ unless otherwise indicated. The boundary of $M$ is denoted by $\partial M$. By $\textit{closed}$ manifold we mean that $\partial M$ is empty. The space of $C^r$ vector fields on $M$ is denoted by $\chi^r(M)$ and the space of $C^r$ diffeomorphisms by $\text{Diff}^r(M)$. A subset of these spaces being \textit{open} or \textit{dense} is understood with respect to the compact-open $C^r$ topology. Vector fields and flows are assumed at least $C^1$ and are defined on closed manifolds or compact subsets $D \subset \mathbb{R}^n$ diffeomorphic to a closed $n$-dimensional disc with flows pointing inwards and intersecting the boundary transversally. Real-valued functions are assumed at least $C^2$ unless specified. The tangent bundle on a manifold $M$ is denoted $TM$ and the tangent space at a point $x\in M$ by $T_xM$. 

Let $(M,g)$ be a Riemannian manifold and $(M, \mathcal{B}, \mu)$ a Borel probability space. We take the natural Borel $\sigma$-field $\mathcal{B}$ on $(M,g)$ generated by Borel sets $A \in \mathcal{B}$ that coincide with the open sets of $M$. Borel probability measures are assumed absolutely continuous with respect to the volume element $d\text{vol}_g = \sqrt{|g|} \, dx^1 \wedge ...\wedge dx^m$, where $|g| = \text{det}(g_{ij})$ is the determinant of the metric tensor. This does not apply to atomic measures.

%---------------------------------------
% Introduction
%---------------------------------------
\newpage
\linespread{1}
\section{Introduction}
\label{section:Introduction}

Associative memory models have been used as abstractions of neural memory for over fifty years. These models, originally inspired by Hebbian learning theory, recall through energy descent attractor dynamics (\cite{dayan2005theoretical}), capturing key features of biological memory retrieval, which operates robustly despite intrinsic noise and variability. They were popularized in a binary state setting with limited storage capacity (\cite{hopfield1982neural}; however, recent models, like dense associative memory or modern Hopfield networks, have removed this limitation (\cite{krotov2016dense, demircigil2017model}). They are now central to modern deep learning (\cite{krotov2023new}), with equivalent formulations to the attention mechanism of the transformer, an architecture behind state-of-the-art natural language processing (\cite{vaswani2017attention,ramsauer2020hopfield}), extensions to full transformer blocks (\cite{hoover2023energy}) and biologically plausible implementations (\cite{kozachkov2025neuron}). Meanwhile, diffusion models have become best-in-class generative models for many tasks (\cite{rombach2022high,ramesh2022hierarchical,cao2024survey}), but their history dates to the Boltzmann machine (\cite{ackley1985learning}), a probabilistic version of the Hopfield model. Although analogies between these model classes are known, interest in their mathematical connections has highlighted a deeper continuum between memory and generation (\cite{hoover2023memory, ambrogioni2024search}).

Continuous-state associative memory and diffusion models define deterministic and random dynamical systems. Recent work suggests that critical point and bifurcation analysis can identify memorization and generation phases in these models, as well as transitions between them, as noise diminishes or as memory load exceeds a critical threshold. Spurious memories arise in the latter case as unintended fixed points. Their behavior is well studied probabilistically (e.g., \cite{gayrard2025mixed}), and \cite{pham2025memorization} identified their appearance as an intermediate phase in a memory-to-generation transition in the infinite memory limit. Analogously, symmetry breaking and memorization phases have been identified during generation in diffusion models (\cite{raya2023spontaneous, biroli2024dynamical}), which proceeds through progressive noise titration. It is known that metastable states emerge and disappear in random systems, and random-to-deterministic transitions are similarly well studied at vanishing noise levels, owing largely to the Freidlin-Wentsell (FW) theory of random dynamical systems (\cite{ventsel1970small}), among others (e.g., \cite{bovier2001metastability,bovier2004metastability}). 

These results are rigorous but largely local -- derived in neighborhoods of critical points (symmetry breaking), concerned with transitions between them (FW theory), or, in the generation-to-memory transition at infinite memory, are not \textit{generic}, or typical, in parameter space due to nongeneric, degenerate critical manifolds that appear. Global descriptions of the robustness, stability, learning and generation dynamics, and orbit structure of these models are less developed. All of these topics are aligned with the notion of genericity. A generic or universal property holds for "almost all" systems --  here, on a dense open subset of an appropriate function or parameter space (\Cref{subsection:genericity-morse-smale-models-associative-memory}). A property holding on a dense subset reflects its approximation power\footnote{See, e.g., classic work in \cite{cybenko1989approximation}, Theorem 1 for a proof of universal approximation using density explicitly; see also \cite{funahashi1989approximate}, Theorems 1 and 2 for proofs implicitly using density.}; a property holding on an open subset reflects its robustness and stability. 

Search for generic properties historically progressed both geometric and probabilistic approaches to dynamical systems. On the geometric side, determining whether global notions of stability are generic stimulated work, largely driven by Smale's school, which laid the basic theory of the subject. Major results pertain to hyperbolic (e.g., Axiom A) and \textit{structurally stable} systems (\cite{smale1967differentiable}). A special subset satisfying the \textit{Morse-Smale} conditions are closely related to the topology of manifolds (\cite{smale1960morse}), making the generic properties of these models well suited for geometric and topological arguments. A probabilistic approach to studying a system's asymptotic behavior, rooted in ergodic theory, developed in parallel to counterexamples showing that structurally stable systems are not generic (e.g., \cite{newhouse1970nondensity}; see \cite{palis2000global, palis2005global}). Here, the geometric and probabilistic approaches are combined with the FW theory to globally detail the generic properties of these models, a generation-to-memory transition at diminishing noise levels, and sequences of bifurcations associated with the loss of structural stability that characterize their learning and generation processes. 

Specifically, Morse-Smale gradient systems are shown to be a generic class of associative memory models and diffusions are described as their white-noise perturbations. Zero-noise-limiting descriptions of invariant measures and trajectories of diffusion models are shown to converge to those of associative memory systems within this class, and universal properties of associative memory implied by the Morse-Smale conditions are combined with results from Morse theory to describe their stability and robustness in the zero-noise limit. These conditions are equivalent to the existence of a neighborhood of a gradient field in which all nearby models are topologically equivalent to it (structural stability). As a result, all nearby models have the same critical point structure and global connections between them, which can be organized into invariant directed acyclic graphs (DAGs) that characterize each memory landscape. Although associative memory models naturally correct noisy and error-prone initial conditions, the theory takes this local robustness to the global topological structure of the system and its invariant measures under small stochastic or deterministic perturbations.

Breakdowns of the Morse-Smale conditions mark bifurcations that change a model's topological type. The learning and generation processes of these models appear as one- and two-parameter families of gradients, which have well-studied, enumerable sets of generic bifurcations that we detail. Along with additional local bifurcations, global bifurcations occur that are not described by local theories. While local bifurcations describe the creation and destruction of fixed points, global bifurcations alter connections between fixed points and therefore the trajectories of recall and generation. Sequences, or cascades of generic bifurcations thus globally characterize memory formation and the learning and generation dynamics of diffusion models, and can be visualized as discrete graph edits between DAGs. Examples from Hopfield networks, energy-based models, and denoising diffusion models are given to verify the theory, and conditions under which Hopfield-type networks satisfy the Morse-Smale assumptions are derived. Our results are summarized in \Cref{subsection:overview}.

% --- Models of associative memory --- 

\subsection{Dynamical systems and associative memory}
Associative memories are often understood as attractors of a dynamical system with basins of attraction containing points that evolve asymptotically to them. Memory recall consists of evolving the network to these fixed points given initial conditions. Therefore, small perturbations within a neighborhood of a memory decay over time, making these models capable of pattern completion and error correction. 

Widely known models, including all models of Hopfield type, define gradient dynamical systems on compact subsets of $\mathbb{R}^n$. More generally, they can be defined on a manifold $M$. If $V \in C^2(M, \mathbb{R})$ is a smooth "energy" function and $M$ is given a Riemannian metric $g = \sum g_{ij} \, dx^i \otimes dx^j$, the derivative $DV(x): T_xM \rightarrow \mathbb{R}$ defines a linear function on the tangent space $T_xM$ at $x \in M$ and hence a smooth 1-form on $M$. The (inverse) metric converts this 1-form to a vector field and defines a gradient system $X := - \nabla_g V$. This vector field generates a flow $\phi^X_t: \mathbb{R} \times M \rightarrow M$ with $t \in \mathbb{R}$ (\Cref{subsection:general-aspects-of-associative-memory}), and in this formulation, memories are \textit{asymptotically stable} points $p \in M$, meaning that there is a neighborhood $U(p)$ of $p$ such that for all $q \in U$, $\phi_t^X(q) \rightarrow p$ as $t \rightarrow \infty$. Asymptotic stability is an implicit defining property that can be found with the introduction of the classic models in \cite{hopfield1982neural,hopfield1984neurons}. 

It can be checked that off critical points, where $X(p) = 0$, the energy $V\left(\phi_t^X(p)\right)$ decreases as $t$ increases. Therefore, if a memory $p \in M$ is an isolated minima of $V$, then it is asymptotically stable. Consequently, gradient systems are natural models of associative memory, and memory storage consists of finding an energy function whose attractors are the desired memory states.

% --- Generative diffusion models --- 

\subsection{Generative diffusion models}
Generative models aim to learn a probability distribution ${\mu_{\text{data}}(dx) = p_{\text{data}}(x) \, dx}$ over a set of data in $\mathbb{R}^n$, with density \( p_{\text{data}}(x) \) absolutely continuous with respect to Lebesgue measure $dx$, by inferring an optimal set of parameters \( \theta^{*} \) for a family of distributions ${\mu_{\theta}(dx) = p_{\theta}(x) \, dx}$ (\cite{cotler2023renormalizing}). Rather than explicitly defining distributions, diffusion models are defined implicitly as solutions to stochastic differential equations. Consider random perturbations to a dynamical system $\dot{x_t} = b(x_t)$ in $\mathbb{R}^n$:
\begin{equation}
\label{equation:stochastic-ode}
    d{x_t^{\epsilon}} = b(x_t^{\epsilon}) \,dt + \epsilon \sigma(x_t^{\epsilon}) \,dw_t \, ,
\end{equation}
where $b$ is a vector field, $w_t$ is the standard $n-$dimensional Wiener process (Brownian motion), $\sigma(x)$ is an $n\times n$ matrix, and $\epsilon>0$ is a scalar noise level. 

Solutions to \eqref{equation:stochastic-ode} define a diffusion Markov process $\mathcal{X}^{\epsilon}=\left\{x^{\epsilon}_t \right\}^T_{t=0}$, with $t \in [0,T]$, given by transition probabilities $\{p^{\epsilon}(\cdot \, | \, x) : x\in \mathbb{R}^n\}$ that satisfy the Chapman-Kolmogorov equation,
\begin{equation*}
\label{equation:transition-probabilities-diffusion-markov}
    p^{\epsilon}(y, s  \, | \, x,t_0) = \int_{\mathbb{R}^n} p^{\epsilon}(y_1, t_1  \, | \, x,t_0) \, p^{\epsilon}(y, s \, | \, y_1,t_1 ) \, dy_{1} \, ,
\end{equation*}
for $t_0 < t_1 < s$ and intermediate states $y_1$. A diffusion process like $\mathcal{X}^{\epsilon}$ is associated with a second-order differential operator, called the backward Kolmogorov operator, given by
\begin{equation}
\label{equation:Euclidean-backward-Kolmogorov}
    \mathcal{L}^{\epsilon} = \sum b^i \frac{\partial}{\partial x^i} + \frac{\epsilon^2}{2}\sum a^{ij} \frac{\partial^2}{\partial x^i \partial x^j} \, ,
\end{equation}
where $a^{ij} = \sigma(x)\sigma^*(x)$ and $\sigma^*(x)$ is the adjoint of $\sigma (x)$ (\cite{ventsel1970small}). The forward Kolmogorov operator, denoted $\mathcal{L}^{*\epsilon}$, is the adjoint of $\mathcal{L}^{\epsilon}$. We consider its action on probability densities, giving the Fokker-Planck equation,
\begin{equation}
\label{equation:Euclidean-Fokker-Planck-equation}
    \frac{\partial p(x,t)}{\partial t} = \sum \frac{\partial}{\partial x^i}\left[b^i p(x,t)\right] + \frac{\epsilon^2}{2}\sum \frac{\partial^2}{\partial x^i \partial x^j}\left[a^{ij} p(x,t)\right] \, ,
\end{equation}
which induces a flow $\partial_t p_t = \mathcal{L}^{*\epsilon}p_t$ solved by the transition density (given initial conditions)\footnote{We assume $b(\cdot)$ and $\sigma(\cdot)$ satisfy regularity and growth conditions to ensure a unique strong solution to \eqref{equation:stochastic-ode} and for transition densities to satisfy the forward and backward Kolmogorov equations; see, e.g., \cite{anderson1982reverse}, Section 3 for such conditions.}. 

Diffusion models sample from $p_{\text{data}}$ using the flow $\partial_t p_t$ induced by \eqref{equation:stochastic-ode} and \eqref{equation:Euclidean-Fokker-Planck-equation}. One class of models, which includes energy-based models and Boltzmann machines, uses forward diffusion processes that evolve an initial density $p_0$, to a density $p_T$ with $0<T\leq\infty$ and $p_{\text{data}} \approx p_{T}$. Another class generates data by learning the reverse of a diffusion process that iteratively adds noise to data, motivated by the main theorem in \cite{anderson1982reverse}, which shows that diffusions admitting a transition density also admit a reverse-time process, which induces a flow from $p_T$ to $p_0 := p_{\text{data}}$. Both model classes are determined by fixed points of the operator $\mathcal{L}^{*\epsilon}$. 

Let $\mathcal{M}$ be the set of Borel probability measures on $\mathbb{R}^n$ with the weak topology. 

\begin{definition}[Stationary and equilibrium distributions]
\label{definition:stationary-or-invariant-measures}
    A probability measure $\mu^{\epsilon} \in \mathcal{M}$ is a \textbf{\textit{stationary}} or \textbf{\textit{invariant measure}} for $\mathcal{X}^{\epsilon}$ if for each Borel set $A$ and $x \in \mathbb{R}^n$: ${\mu^{\epsilon}(A) = \int p^{\epsilon}(A, t \, | \, x) \, \mu^{\epsilon}(dx)}$ with $t>0$; that is, if it is a fixed point of $\mathcal{L}^{*\epsilon}$. It is called an \textbf{\textit{equilibrium distribution}} if ${\int \mu^{\epsilon}(dx)=1}$ and ${p^{\epsilon}(A, t \, | \, x)\rightarrow \mu(A)}$ as $t\rightarrow \infty$.
\end{definition}

If an equilibrium distribution exists, it is unique. Stationary distributions are also equilibrium distributions up to a constant factor; see \cite{kent1978time}, Section 5. 

More generally, the support of the density \( p(x) \) is a Riemannian $n-$manifold $(M,g)$. In local coordinates $(x^1,...,x^n)$ on $M$, equations like \eqref{equation:stochastic-ode} become
\begin{equation}
\label{equation:sde-on-manifold}
    d{x_t^{\epsilon}} = b^{\epsilon}(x_t^{\epsilon}) \,dt + \epsilon \sigma(x_t^{\epsilon}) \,dw_t \, ,
\end{equation}
where $b^{\epsilon}(x) = b(x) + \frac{\epsilon^2}{2}\widetilde{b}(x)$, with
\begin{equation*}
    \widetilde{b}^i = \sqrt{\det(g^{ij})}\sum_j\frac{\partial}{\partial x^j}\left [ g^{ij}\sqrt{\det(g_{ij})} \right]
\end{equation*}
and $g^{ij} = \sigma(x)\sigma^*(x)$, which represents a drift correction taking into account the local curvature of $M$. The operator generating a diffusion process $\mathcal{X}^{\epsilon}$ solving \eqref{equation:sde-on-manifold} is
\begin{equation*}
\label{equation:perturbations-morse-smale-gradient-operator}
    L^{\epsilon}f(x) = \frac{\epsilon^2}{2}\Delta f(x) + \left<b(x), \nabla_g f(x) \right > \, ,
\end{equation*}
where $f \in C^{2}(M)$, $\Delta$ is the Laplace-Beltrami operator corresponding to the metric $g$, and $\nabla_g$ is the Riemannian gradient. 

It is shown in \Cref{section:applications-examples} that the drift fields of many diffusion models are gradients. That is, the term $b(x_t^{\epsilon})$ in \eqref{equation:sde-on-manifold} is the gradient of a function $V\in C^2(M, \mathbb{R})$, so
\begin{equation}
\label{equation:riemannian-small-random-perturbation-equation}
    d{x_t^{\epsilon}} = -\nabla_g V ^{\epsilon}(x_t^{\epsilon}) \, dt + \epsilon \sigma(x_t^{\epsilon}) \, dw_t.
\end{equation}
The notation $\nabla_g V^{\epsilon}$ emphasizes that the drift is the Riemannian gradient of $V$ with drift correction. It is known that if diffusions solving \eqref{equation:riemannian-small-random-perturbation-equation} do not escape to infinity (i.e., they are \textit{nonexplosive}), which is ensured if $M$ is compact\footnote{See also \cite{kent1978time}, Theorem 4.2 for $\mathcal{X}^{\epsilon}$ being reversible, which implies the same result.}, then their stationary measures are Boltzmann-Gibbs distributions:
\begin{equation}
    \mu(dx)=\frac{1}{Z} e^{\frac{-V(x)}{\epsilon^2}} d\text{vol}_{g}(x) \hspace{1em} \text{with } \hspace{1em} Z = \int_M e^{\frac{-V(x)}{\epsilon^2}} d\text{vol}_{g}(x) \, ,
\end{equation}
where $\mu(dx)$ is absolutely continuous with respect to Lebesgue measure and $Z$ ensures that the density $p(x) = {Z}^{-1}\cdot e^{{-V(x)}/{\epsilon^2}}$ is normalized with $\int_M p(x) \, dx = 1$.

%---------------------------------------
% Overview 
%---------------------------------------
\subsection{Outline and overview of main results}
\label{subsection:overview}

Exact samples can be drawn from stationary distributions of diffusion models as solutions to stochastic differential equations like \eqref{equation:riemannian-small-random-perturbation-equation} when $t \rightarrow \infty$. Likewise, associative memory models evolve network states to fixed points in this limit. Like in the stochastic case, this asymptotic behavior actually encodes information about invariant probability measures. Moreover, when $\epsilon = 0$, \eqref{equation:sde-on-manifold} is the equation of a dynamical system and \eqref{equation:riemannian-small-random-perturbation-equation} a gradient system. Consequently, it is of interest to study the limiting behavior of diffusion processes with gradient drift as noise levels vanish, $\epsilon \rightarrow 0$, and when time reaches $\pm \infty$. The trajectories of diffusions with vanishing, but small positive noise are studied when $\epsilon \rightarrow 0$, while their invariant measures are studied when $t\rightarrow \pm\infty$. 

\paragraph{Trajectories and stationary measures.} 
Given an initial condition $x = x_0$, the forwards Kolmogorov equation implies that the transition density of a diffusion process is increasingly governed by the drift field as $\epsilon \rightarrow 0$, making diffusion models with gradient drift \textit{small random perturbations} of associative memory models:

\begin{definition}[Small random perturbation, adapted from \cite{cowieson2005srb}]
\label{definition:small-random-perturbation}
    Let $\phi: M \rightarrow M$ be a $C^{r}$ diffeomorphism for $r \geq 1$. A one-parameter family of Markov chains $\{\mathcal{X}^{\epsilon}\}_{\epsilon>0}$, parameterized by $\epsilon$ and given by transition probabilities $\{p^\epsilon(\cdot \, | \, x): x \in M\}$ is a \textbf{\textit{small random perturbation}} of $\phi$ if $p^\epsilon(\cdot \, | \, x) \rightarrow \delta_{\phi(x)}$ uniformly as $\epsilon \rightarrow 0$. 
\end{definition}

The flow $\phi_t^X$ of a gradient system $X = -\nabla_g V$ is downhill off its set of critical points, since $DV_pX(\cdot) = -|| X(\cdot) ||^2$. Consequently, it has no \textit{closed orbits}, i.e., trajectories diffeomorphic to the $1$-sphere, $S^1$, and its only critical elements are singularities. By the Poincaré recurrence theorem (\cite{sinai1989dynamical}, Theorem 2.1), any probability measure invariant under $\phi_t^X$ must assign zero measure to Borel sets $A$ with $\phi^X_t(A) \cap A = \emptyset$ when $t\rightarrow \infty$, implying invariant measures for gradient flows concentrate on their singularities; however, they are not absolutely continuous with respect to Lebesgue measure, in contrast to invariant measures for diffusions -- the two are weakly related:

\begin{definition}[Zero-noise limit, \cite{cowieson2005srb}]
    Let $M$ be a Riemannian manifold and $\mathcal{M}$ the set of Borel probability measures on $M$. A probability measure $\mu \in \mathcal{M}$ is a \textbf{\textit{zero-noise limit}} of a small random perturbation $\mathcal{X}^{\epsilon}$ of $\phi: M \rightarrow M$ if $\mu$ is a limit point of the sequence $\{\mu^{\epsilon}\}$ as $\epsilon \rightarrow 0$ in the weak topology on $\mathcal{M}$, where $\mu^{\epsilon}$ are stationary measures of $\mathcal{X}^{\epsilon}$. That is, $\int_M f(x) \,d\mu^{\epsilon}(x) = \int_M f(x) \,d\mu(x)$ as $\epsilon \rightarrow 0$ for any continuous bounded function $f \in C^0_b(M)$.
\end{definition}

\paragraph{Stability and universal properties.}
The elements of a generic set of \textit{Morse-Smale} gradient systems have a finite number of critical points and satisfy two conditions that describe their behavior at critical points (Morse condition) and off critical points (Smale condition). That these systems are generic implies that they are universal approximators for gradient-based associative memory models (\Cref{subsection:genericity-morse-smale-models-associative-memory}). Several universal properties of associative memory are derived from the Morse-Smale conditions in \Cref{subsection:structural-stability-models-associative-memory}, and they are used to study a generic generation-to-memory transition through the zero-noise limiting behavior of invariant measures and trajectories of small random perturbations of gradient flows in \Cref{section:diffusion-zero-noise-limit,section:generic-arcs}.

Local and global stability are two properties that follow from the Morse-Smale conditions. The Morse condition implies that the critical points of these systems are nondegenerate and isolated, and that each only displays three possible types of behavior characterized by the \textit{index}, or number of negative eigenvalues of the Hessian matrix of the potential at that point. An index $0$ critical point is a local minimum and acts as an \textit{attractor} (i.e., a memory), while a point of maximal index is a \textit{repellor}. Critical points of intermediate index are \textit{saddles}, with both stable and unstable directions. The asymptotic stability of attractors follows from these properties, making the Morse condition fundamental for associative memory systems (\Cref{subsubsection:asymptotic-stability}).

The Smale condition is necessary for a global form of stability, called \textit{structural stability}, that is relevant to describe robustness to model parameter perturbation. Two vector fields $X,Y \in \chi^r(M)$ are \textit{\textbf{topologically equivalent}} if there exists a homeomorphism $h: M \rightarrow M$ that maps orbits of $X$ to orbits of $Y$ preserving orientation. That is, there exist times $t_1, t_2 >0$ such that $h(\phi^X_{t_{1}}(p)) = \phi^Y_{t_{2}}\left( h(p)\right)$ for any $p\in M$. If $h$ preserves the time parametrization, it is called a \textit{conjugacy}. Topologically equivalent systems necessarily have the same critical points (modulo homeomorphism) whose indices agree. Consequently, topological equivalence captures when associative memory models are qualitatively the same. Structural stability captures the notion that small perturbations to a dynamical system should not change its topological type. 

Let $X$ be a $C^r$ vector field on a manifold $M$. A vector field $\delta X$ is called a \textit{${C^k}$ perturbation of size ${\epsilon}$} for $k \leq r$ if the difference between $X$ and $\delta X$, and their $k^{\text{th}}$ order partial derivatives are uniformly less than $\epsilon$ at all points $p \in M$ (\cite{guckenheimer2013nonlinear}, Chapter 1.7)\footnote{A general definition can be made with respect to the compact-open topology on the set $\chi^r(M)$ of $C^r$ vector fields on a manifold $M$.}.

\begin{definition}[Structural stability]
\label{definition:structural-stability}
    A vector field $X \in \chi^r(M)$ is \textbf{\textit{structurally stable}} if there exists an $\epsilon>0$ so that $C^1$ perturbations of size less than or equal to $\epsilon$ are topologically equivalent to $X$. That is, there is a neighborhood $U$ of $X$ in $\chi^r(M)$ such that all $Y \in U$ are topologically equivalent to $X$.
\end{definition}

Unlike asymptotic stability, structural stability is global, defined by the dynamical system itself. "Robustness" of a vector field means that all nearby systems (in a neighborhood) are isomorphic to it. Importantly, a gradient system is structurally stable exactly when it is Morse-Smale. 

The following example illustrates small random perturbations and zero-noise limits using a bistable, Morse-Smale memory model. Examples throughout are motivated by normal form polynomials (those with minimal parameters) for illustrative purposes, and they are vector fields and small random perturbations on globally attracting regions diffeomorphic to the closed $n$-dimensional disc $D \subset \mathbb{R}^n$; that is, the set $\{ x \in \mathbb{R}^n \, : \, ||x|| \leq 1 \}$ with a spherical boundary $\partial D \cong \{ x \in \mathbb{R}^n \, : \, ||x|| = 1 \}$ that intersects these vector fields transversally. The transversality condition is generic (\Cref{subsubsection:assumption-boundary}).

\begin{example}[Bistable associative memory, \colorAutoref{fig:figure-intuition-example-zero-noise}]
\label{example:bistable-memories}
Consider storing two patterns $\beta_1, \beta_2$ using a dual-well potential in two variables $v_1, v_2$ with quadratic terms in all other variables. Add quartic terms on $v_1,v_2$ to bound the dynamics with repellors at $\pm \infty$ so that the flow points inward. Neuronal states evolve via gradient dynamics: 
\begin{equation*}
\label{equation:example-bistable-associative-memory-dual-well}
    {\dot{v^i}} = - \sum_{j=1}^{n} \delta^{ij} \frac{\partial V(v)}{\partial v^i} \hspace{1em} 
    \mathrm{with} \hspace{1em}  V(v) = \dfrac{1}{4} v_{1}^{4} + \dfrac{1}{4} v_{2}^{4} - v_{1}^{2} + v_{2}^2 + \sum_{i=3}^{n}v_{i}^{2} \,.
\end{equation*}
Given a noise level $\epsilon > 0$, a diffusion process $\mathcal{X}^{\epsilon}=\left\{v^{\epsilon}_t \right\}^T_{t=0}$ with $t \in [0,T]$ can be constructed as solutions to \eqref{equation:riemannian-small-random-perturbation-equation}. As $\epsilon \rightarrow 0$, the trajectories of $\mathcal{X}^{\epsilon}$ converge to those of the deterministic gradient flow. This intuition is captured by saying that $\{\mathcal{X}^{\epsilon}\}_{\epsilon>0}$ is a small random perturbation of the gradient flow $\phi_t^X$. Invariant measures of $\mathcal{X}^{\epsilon}$ are Boltzmann-Gibbs distributions, $\mu^{\epsilon}(dv)=({Z^{\epsilon}})^{-1} \cdot e^{{-V(v)}/{\epsilon^2}} d\mathrm{vol}_{g}(v)$ with ${Z^{\epsilon} = \int_{D} e^{{-V(v)}/{\epsilon^2}} d\mathrm{vol}_{g}(v)}$.

\begin{figure}[t!]
% \vspace{-1.5em}
        \centering
        \includegraphics[width=1\textwidth]{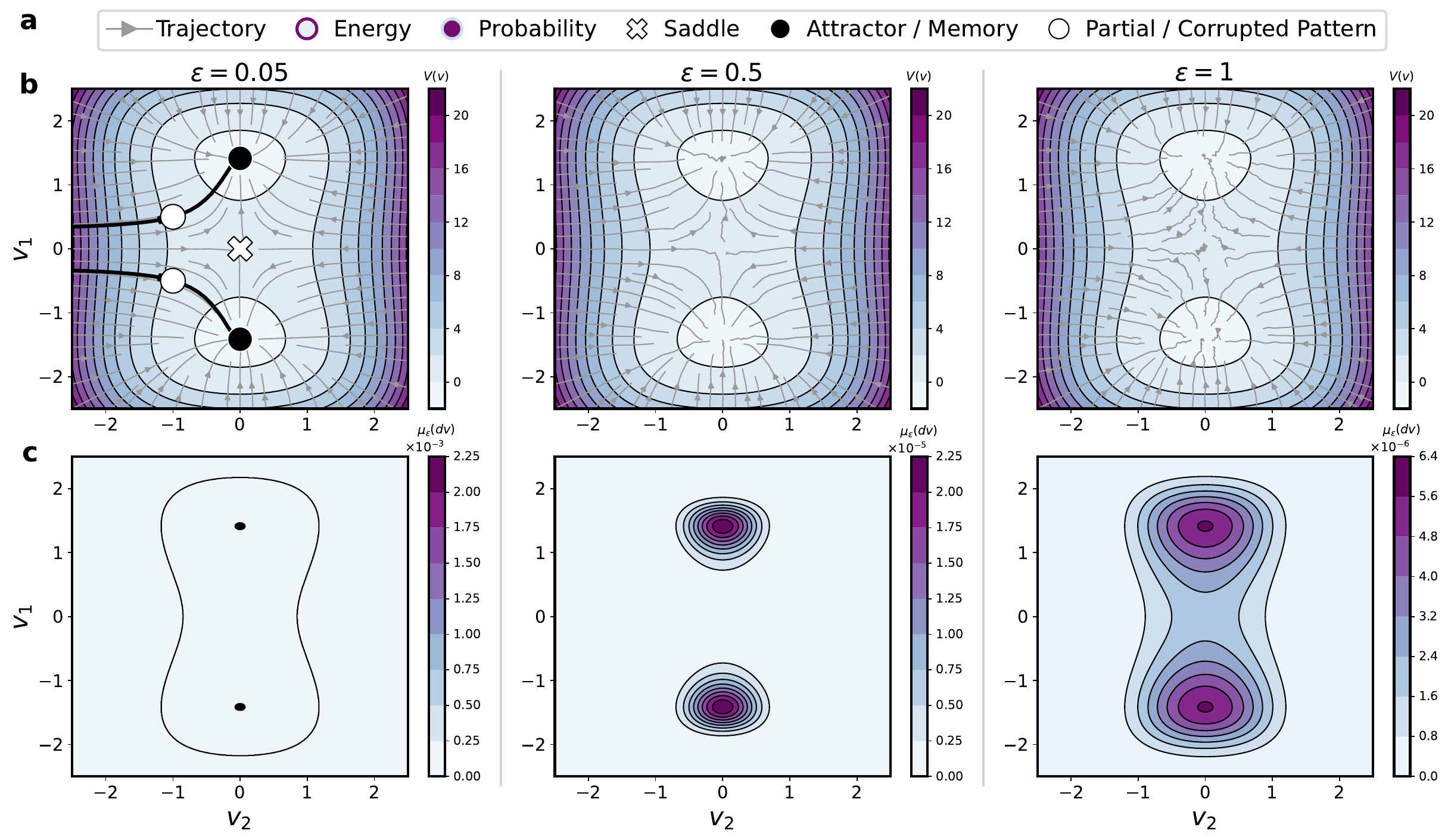}
        \caption{\textbf{\textit{Small random perturbations} and \textit{zero-noise limits}.} {\textbf{(a)} Symbols representing phase portraits and probability measures throughout. \textbf{(b)} Trajectories (grey) of diffusion processes $\mathcal{X}^{\epsilon}$ at varying noise levels (left to right) overlayed on the energy surface (light blue interior, magenta exterior) of the associative memory model in \Cref{example:bistable-memories}. As $\epsilon \rightarrow 0$ trajectories approach those of the deterministic system -- $\{\mathcal{X}^{\epsilon}\}$ are \textit{small random perturbations} of the gradient flow. Attractors are memories (black circles). Recall from partial/corrupted patterns (white circles) is given by the asymptotic behavior of trajectories (black lines). \textbf{(c)} Invariant measures $\mu^{\epsilon}(dv)$ of $\mathcal{X}^{\epsilon}$ are Boltzmann-Gibbs distributions. As $\epsilon \rightarrow 0$, the sequence $\{ \mu^{\epsilon}\}$ converges to the \textit{zero-noise limit} of $\{\mathcal{X}^{\epsilon}\}$}.}
        \label{fig:figure-intuition-example-zero-noise}
% \vspace{-1.5em}
\end{figure}
\end{example}

%---------------- Summary of main results -----------------------

\paragraph{Learning and memory consolidation.}  Let $(M,g)$ be a Riemannian manifold and $X = -\nabla_g V$ an associative memory model. In applications, the potential and the metric can be parameterized by neural networks with parameters $\theta_V$ and $\theta_g$. Denote the parameter set by $\theta = \{ \theta_g, \theta_V\}$. Let $\mathcal{L}: \theta \rightarrow \mathbb{R}$ be a smooth loss function, and denote by $\theta^*$ a parameter set that solves $\min_{\theta} \mathcal{L}(\theta)$ by, for example, gradient descent. It is natural to ask whether this learning process can be characterized.

A \textbf{\textit{$\bm{k-}$parameter family of gradients}} is a family of vector fields $\{ X_{\eta_1,...,\eta_k} \}$ on $M$ with $\eta \in N^k$, where $N$ is a $k-$manifold, for which there is a family of functions $\{V_{\eta_1,...,\eta_k} \}$ and family of metrics $\{ g_{{\eta_1,...,\eta_k}}\}$ with $X_{\eta_1,...,\eta_k} = -\nabla_{g_{\eta_1,...,\eta_k}} V_{\eta_1,...,\eta_k}$. Memory consolidation, or storing new memories, appears as a one-parameter family, $\{X_{\eta}\}_{\eta \in \mathbb{R}}$, where $\eta \in \mathbb{R}$ is a time parameter of the learning process. Similarly, let $X_t = -\nabla_{g_{t}} V_t$ with $t\in \mathbb{R}$ be a time-varying model with potential and metric that smoothly vary with $t$. Gradient descent then produces a two-parameter family. 

Shortly after the structural stability theory for gradients was established, Thom asked about classifying bifurcations associated with the loss of stability of their parameterized families (\cite{thom1969topological,thom1972stabilite}; see \Cref{subsubsection:universal-unfoldings}). In contrast to Thom's local approach, which focused on parametrized potential functions, additional global bifurcations can occur, making the classification of bifurcations for arbitrary $k-$parameter families of gradients difficult. However, the one- and two-parameter cases are relatively simple (\cite{palis1983stability,dias1989bifurcations}). Only two bifurcations must be considered for one-parameter families; the two-parameter case is more complicated, but the bifurcations are still enumerable (\Cref{subsubsection:one-parameter-families-gradients,subsubsection:two-parameter-families-gradients}). 

Importantly, bifurcations occur at isolated parameter values, which implies that memory formation is characterized by ordered sequences of these topology-changing moves. Moreover, associated with each Morse-Smale system is an invariant directed acyclic graph (DAG) that globally organizes the rest points of the system and the connections between them (\Cref{subsubsection:DAG}). Sequential bifurcations are visible as successive alterations to these DAGs.

\paragraph{Trajectories at vanishing noise levels.}
Let $\mathcal{X}^{\epsilon}$ be a small random perturbation of $\phi_t^X$ by white-noise perturbations as in \eqref{equation:riemannian-small-random-perturbation-equation}. Likewise, let $\mathcal{X}^{\epsilon,t}$ be obtained from a stochastic differential equation with time-varying drift. A similar question as above is: \textit{Can the generation and learning process of diffusion models be characterized?}

The bifurcations of one- and two-parameter families topologically characterize memory formation. However, topological equivalence of flows and structural stability do not apply to diffusions in the vanishing noise limit. An analogue of topological equivalence for diffusions in this limit is defined in \Cref{subsection:paths-large-deviations} to address this shortcoming.

\paragraph{Zero-noise limits.}
The hope is that the Boltzmann-Gibbs distributions of diffusion models encode the asymptotic behavior of their corresponding associative memory model in the zero-noise limit. Generically, this is not the case (\Cref{subsubsection:issues-with-globally-defined-measures}). However, gradients satisfying the Morse condition have well-defined \textit{stable} and \textit{unstable manifolds}. Given a critical point $p \in M$, the sets
\begin{align*}
    & W^s(p) = \{x \in M \, : \, \phi^X_t(x) \rightarrow p \textnormal{ as }t \rightarrow \infty \} \hspace{1em} \text{and} \\
    & W^u(p) = \{x \in M \, : \, \phi^X_t(x) \rightarrow p \textnormal{ as }t \rightarrow -\infty \}
\end{align*} 
that consist of those points that flow to $p$ as $t \rightarrow \pm\infty$, are called the \textbf{\textit{stable manifold}} and \textit{\textbf{unstable manifold}} of $p$, respectively. By focusing on stable manifolds, the inconsistency above can be relieved, and some information is retained (\Cref{subsection:robustness-diffusion-models}).

The zero-noise limits of families of diffusions with Morse-Smale gradient drift show more interesting behavior. By structural stability, the image of these measures can be studied globally under homeomorphisms between topologically equivalent systems. We detail the behavior of zero-noise limits under these maps, as discussed below. 

\vspace{1em}
\noindent
The main contributions of this paper can now be stated. In addition to the generic Morse and Smale conditions, one assumption and a generic condition (\Cref{section:assumptions}) are imposed to obtain results (ii)-(iv). An additional generic condition (\Cref{subsubsection:trajectory-spaces-negative-gradients}) is assumed in applications to obtain (v):
\begin{enumerate}[label=(\roman*), itemsep=0pt]

    \item Morse-Smale dynamical systems are universal associative memory models, and this class is natural for reliable and structurally stable memory (\Cref{subsection:general-aspects-of-associative-memory}). Associated with each is an invariant DAG that classifies the memory landscape.
    
    \item If $V \in C^2(M,\mathbb{R})$ is a Morse function, then on stable manifolds of critical points, zero-noise limits of families $\{ \mathcal{X}^{\epsilon}\}_{\epsilon >0}$ vary continuously with the weak topology on Borel probability measures and the compact-open $C^r$ topology on real-valued functions (\Cref{subsection:robustness-diffusion-models}, \Cref{lemma:zero-noise-continuous-dependence}). That is, on these invariant sets, zero-noise limits of diffusions with gradient drift are robust to perturbations to drift terms.
    
    \item If $V \in C^2(M,\mathbb{R})$ is a Morse-Smale function, regions of convergence of zero-noise limits of families $\{ \mathcal{X}^{\epsilon}\}_{\epsilon >0}$ vary continuously with the weak topology on Borel probability measures and the compact-open $C^1$ topology on flows (\Cref{subsection:robustness-diffusion-models}, \Cref{proposition:triangle-commutativity-zero-noise}). This is a global form of (i) obtained from the map between measures induced by a homeomorphism between topologically equivalent systems. Caveats on the absolute continuity of the induced measure, or image under this map, must be considered. Necessary conditions for this measure to also be a zero-noise limit are given.

    \item Sequences of bifurcations of one- and two-parameter families of gradients describe memory formation and the learning and generation dynamics of diffusion models in the zero-noise limit (\Cref{section:generic-arcs}). These bifurcations correspond to sequential graph edits that offer a combinatorial language for the emergence and restructuring of memory landscapes, and we detail all generic bifurcations that can occur in parameter space. Along with additional local bifurcations, global bifurcations appear that are not described by local bifurcation theory.
    
    \item The generality of the framework is verified with examples from energy-based models, Hopfield networks and Boltzmann machines, modern Hopfield networks, and denoising diffusion models. Structural stability criteria for classic and modern Hopfield networks are also derived (\Cref{section:applications-examples}).
\end{enumerate}

%---------------- Overview -----------------------

\paragraph{Outline.}
The outline of this paper is as follows. \Cref{subsection:general-aspects-of-associative-memory} states the assumptions made for stable and reliable associative memory, followed by the universal approximation of these models by Morse-Smale gradients in \Cref{subsection:genericity-morse-smale-models-associative-memory}. Generic properties of associative memory that follow from the Morse-Smale assumptions are described in \Cref{subsection:structural-stability-models-associative-memory}. In \Cref{section:diffusion-zero-noise-limit}, generic zero-noise limits of diffusions with gradient drift are studied along with their stability. In \Cref{section:generic-arcs} trajectories in the zero-noise limit are studied along with parameterized families of gradients and their bifurcations. \Cref{section:applications-examples} verifies the theory with examples. Finally, \Cref{section:discussion-and-conclusion} is a discussion and summary.

%%%%%%%%%%%%%%%%%%%%%%%%%%%%%%%%%%%%%%%%%%%%%%%%
% Associative memory and Morse-Smale systems %
%%%%%%%%%%%%%%%%%%%%%%%%%%%%%%%%%%%%%%%%%%%%%%%%
\section{Associative memory and Morse-Smale gradients}
\label{section:Morse-Smale-gradient-assocaiative-memory}

The assertion that Morse-Smale gradients universally approximate energy-based associative memory models is now detailed, followed by generic properties of associative memory that follow from this characterization. First, it is shown that the Morse-Smale conditions also arise naturally from basic properties that reliable and stable associative memory models should satisfy.

%------------------------------------------
% Mathematical models of associative memory
%------------------------------------------
\subsection{Stable associative memory}
\label{subsection:general-aspects-of-associative-memory}

Let $M$ be a smooth $n$-manifold. In the dynamical systems view of associative memory, any point $x \in M$ is an instantaneous condition of a dynamical system $X$ defined by a set of differential equations. Suppose that an item $p\in M$ is stored in memory. The following are natural to impose for stable, reliable memory recall: 
\begin{enumerate}[itemsep=0pt]
    \item The flow $\phi_t^X$ generated by $X$ is globally defined for all time;
    \item If $p \in M$ is a memory and $U \ni p$ is an open neighborhood of $p$, then for any $q\in U$, the orbit $\phi_t^X(q)$ should intersect $U$ for all time $t > 0$. In other words, the dynamics are asymptotically confined to $U$ if $p$ is a stored memory;
    \item The number of neurons and memories are finite (applicable even to models with large storage capacity, e.g., \cite{krotov2016dense,ramsauer2020hopfield}); 
    \item The critical elements and trajectories of $X$ are stable under small perturbations;
    \item Memories are singular points.
\end{enumerate}

% Compactness %
\paragraph{Compactness (Condition 1).}
The time integration of a differential equation on a compact manifold produces a one-parameter set of diffeomorphisms $\phi^X: \mathbb{R} \times M \rightarrow M$, or flow, $\phi_t^X: M \rightarrow M$ given by a left action of the additive group of real numbers sending $p \mapsto \phi_t^X(p)$. Given a time $t \geq 0$, the map $\phi^X_t(p)$ solves the differential equation
\begin{equation}
\label{equation:vector-field-ode-definition}
    \frac{d}{dt}\phi^X_t(p) = X\left(\phi^X_t(p) \right), \hspace{0.5em}  \text{with }\phi^X_0(p)=p \,.
\end{equation}
Since $M$ is compact, by (1), the vector field $X$ is \textit{complete}\footnote{If $M$ is compact, any smooth vector field is complete; see, e.g., \cite{lee2013smooth}, Theorem 9.16 or \cite{palis2012geometric}, Proposition 1.3.}; that is, $\phi^X_t(M)$ is globally defined for all time $t\in \mathbb{R}$. Without a loss of generality\footnote{Recall that any smooth manifold admits a Riemannian metric; see, e.g., \cite{lee2006riemannian}, Chapter 3.}, $M$ is given the additional structure of a Riemannian metric $g = \sum g_{ij} \, dx^i \otimes dx^j$; the pair is denoted $(M,g)$.

% Non-wandering set %
\paragraph{Non-wandering sets (Condition 2).}
This is understood more precisely as follows:

\begin{definition}[Non-wandering set]
    A point $p \in M$ is called a \textbf{\textit{non-wandering point}} if there exists a neighborhood $V \ni p$ for which $\phi_t^X(V) \cap V \neq \emptyset$ for $|t|> t_0$. The set of non-wandering points is denoted $\Omega(X)$.
\end{definition}

Condition (2) states that the non-wandering set, $\Omega(X)$, should contain stored memories. These need not be the only elements of $\Omega(X)$. It is reasonable to assume that any point $p \in \Omega(X)$ that is \textit{not} a stored memory is a critical point. This forbids complicated orbits, such as quasiperiodic motions that may have unpredictable dynamics.

% Compactness %
\paragraph{Finiteness (Condition 3).}
By (2), the non-wandering set $\Omega(X)$ reduces to the set of critical elements of a vector field $X$, and by (3), the number of all such elements is finite\footnote{Compactness implies that a Morse function has a finite number of critical points, so (1) and (3) are not independent; see \Cref{subsubsection:asymptotic-stability}.}. Therefore, dynamical systems that serve as models of associative memory are restricted -- they are those with finite numbers of critical elements.

% Structural stability %
\paragraph{Structural and $\bm{\Omega}$-stability (Condition 4).}
For systems satisfying conditions (1)-(3), there are a well-studied set of conditions to ensure that $X$ is structurally stable, satisfying (4). These systems are known as Morse-Smale (\cite{smale1960morse}). The set of all Morse-Smale systems on $M$ is denoted by $S(M)$. 

Requiring that critical elements are stable to perturbation is evidently weaker than demanding stability of trajectories. This weaker notion is well-captured by \textit{${\Omega}-$stability}. Let $X \in \chi^r(M)$ be a $C^r$ vector field and $\Omega(X)$ its non-wandering set. The flow $\phi_t^X$ is called $\mathbf{\Omega}-$\textbf{stable} if there is a neighborhood $\mathcal{U}$ of $X$ such that for all $Y \in \mathcal{U}$, the restriction $\phi_t^X|_{\Omega(X)}$ of $\phi_t^X$ to its non-wandering set is topologically equivalent to $\phi_t^Y|_{\Omega(Y)}$. It is generally clear that structural stability implies $\Omega$-stability.

% Structural stability %
\paragraph{Memories as singularities (Condition 5).}
The non-wandering set $\Omega(X)$ of a Morse-Smale vector field $X$ consists of a finite number of critical elements; however, these can be singularities or closed orbits. It is difficult to imagine a widely applicable scenario where closed orbits are useful for reliable associative recall. Such behavior would introduce ambiguity, as the dynamics would not converge to a single point. 

Reliable, structurally stable associative memory models are therefore Morse-Smale systems whose critical elements are a finite set of singularities. These vector fields, referred to as \textit{Morse-Smale gradient fields}, and denoted $G(M)$, are defined as follows:

\begin{definition}[Morse-Smale gradient]
\label{definition:Morse-Smale-vector-field}
A vector field $X \in G(M)$ and is called \textbf{Morse-Smale gradient field}\footnote{The definition of $S(M)$ allows $X$ to also have a finite number of closed orbits.} if it satisfies the following:
    \begin{enumerate}[itemsep=0pt]
    \item $\Omega(X)$ has finite cardinality and is equal to the union of the critical elements of $X$, each of which is a singular point;
    \item (Morse condition) the critical elements of $X$ are all hyperbolic;
    \item (Smale condition) if $\beta_1$ and $\beta_2$ are critical elements of $X$, then $W^u(\beta_1)$ is transverse to $W^s(\beta_2)$.
\end{enumerate}
\end{definition}
\noindent

Given a vector field $X \in \chi^r(M)$, a critical point $p \in M$ is a \textit{hyperbolic singularity} if the spectrum of $DX_p: T_pM \rightarrow T_pM$ is disjoint from the imaginary axis. Let $X = -\nabla_g V$ be a gradient field. Since gradient fields have no closed orbits, this condition means that $DX_p$ does not have a zero eigenvalue. This condition also implies that the potential function $V \in C^2(M,\mathbb{R})$ is \textit{Morse}, i.e., all its critical points are nondegenerate -- equivalently, hyperbolic. That the intersection of stable and unstable manifolds of distinct critical elements is transversal means that their tangent spaces span the tangent space of $M$ at their intersection. That is, if $x \in W^u(\beta_i) \, \cap \,W^s(\beta_j)$ for $\beta_i, \beta_j \in \Omega(X)$ critical elements, then $T_xM = T_x W^u(\beta_i) + T_x W^s(\beta_j)$.

%------------------------------------------
% Mathematical models of associative memory
%------------------------------------------
\subsection{Universal approximation}
\label{subsection:genericity-morse-smale-models-associative-memory}
\Cref{subsection:general-aspects-of-associative-memory} indicates that Morse-Smale gradients are stable, reliable models of associative memory. The assertion that they are generic in the compact-open topology, or universal, is now detailed. This topology captures the notion that two vector fields are close if the vector fields and their $r^{\text{th}}$ order partial derivatives are close at all points $x \in M$ on compact sets. The following definition is from \cite{hirsch2012differential}, Chapter 2.1.

\begin{definition}[Compact-open $C^r$ topology]
\label{definition:compact-open-topology}
    Let $M,N$ be $C^r$ manifolds with $0 \leq r < \infty$ and $C^r(M,N)$ the space of $C^r$ maps from $M$ to $N$. The \textbf{\textit{compact-open $\mathbf{C^r}$ topology}} on $C^r(M,N)$ has as a basis the following elements. Let $(\phi,U),(\psi,V)$ be charts on $M$ and $N$ and $f \in C^r(M,N)$; let $K\subset U$ be a compact subset with $f(K) \subset V$ and $0\leq \epsilon < \infty$. A weak subbasic neighborhood $\mathcal{N}^r(f;(\phi,U),(\psi,V),K,\epsilon)$ is defined to be a set of $C^r$ maps $g:M \rightarrow N$ such that $g(K) \subset V$ and
\begin{equation*}
    || D^k(\psi \circ f \circ \phi^{-1})(x) - D^k(\psi\circ g\circ \phi^{-1})(x) || < \epsilon
\end{equation*}
for all $x \in \phi(K)$ and $k = 0,...,r$. A basis for $C^{\infty}$ is obtained by taking the union of the topologies induced by the inclusions $C^{\infty}(M,N) \rightarrow C^{r}(M,N)$.
\end{definition}

The set of $C^r$ diffeomorphisms, $\text{Diff}^r(M)\subset C^r(M,M)$ form an open subset of the space of $C^r$ maps from $M$ to itself and are given the subspace topology. Likewise, a vector field $X \in \chi^r(M)$ is a $C^r$ map $X: M \rightarrow TM$ which assigns to each point $p \in M$ a vector $X(p) \in T_pM$, so \Cref{definition:compact-open-topology} applies to vector fields as maps ${X \in C^r(M,TM)}$. The set $C^{r}(M,\mathbb{R})$ of real-valued functions can also be given the compact open topology. 

It is a classic result that Morse functions form a dense open subset of $C^{r}(M,\mathbb{R})$ for $r \geq 2$; see \Cref{appendix:Morse-theory}. Given any Riemannian metric on a compact $n-$manifold $M$, Smale showed that $S(M) \, \cap \, \text{Grad}(M) = G(M)$ forms a dense open set in the space of gradient vector fields, $\text{Grad}(M)$, with the $C^1$ topology (\cite{smale1961gradient,PALIS1969385}). In fact, the Morse-Smale gradients form a generic subset of $\text{Grad}^r(M)$ for any $r \geq 1$\footnote{See \cite{banyaga2004lectures}, Theorem 6.6 or \cite{audin2014morse}, Theorem 2.2.5 for accessible proofs that the Morse-Smale gradients are generic.}. 

Let $(U,h)$ be a coordinate chart at a point $p \in M$. In local coordinates $(x^1,...,x^n)$, these systems are written as an inverse metric $g^{ij} = (g)^{-1}_{ij}$ times the gradient of a Morse potential function $V\in C^2(M,\mathbb{R})$: $\dot{x^{i}} = -\sum_{j=1}^{n} g^{ij}\frac{\partial V}{\partial x^j}$. This equation is standard for computing Riemannian gradients. The inverse metric establishes an isomorphism $TM \cong T^*M$ between the tangent bundle $TM$ and its dual $T^*M$. The last term is a covector, ${\partial V}/{\partial x^j} \in T^*_xM$. Contracting it with the inverse metric ensures that the left-hand side is defined independent of coordinates. Consequently, the flows defined by generic associative memory models are captured by a metric and a potential.

%------------------------------------------
% Mathematical models of associative memory
%------------------------------------------
\subsection{Generic properties}
\label{subsection:structural-stability-models-associative-memory}
A number of generic properties of associative memory follow from the Morse-Smale conditions, which are now described -- we are not aware of such a discussion elsewhere. 

\subsubsection{Asymptotic stability}
\label{subsubsection:asymptotic-stability}
In \cite{hopfield1982neural}, asymptotic stability appears as a defining property of associative memory -- if a sufficiently small perturbation is applied to neuronal states $x \in U$, where $U \ni p$ a neighborhood of a stored memory $p \in M$, the network evolves to restore $p$ as $t\rightarrow \infty$. Given a Morse function $V \in C^2(M, \mathbb{R})$ and a nondegenerate critical point $p \in M$, the Morse lemma implies that, at a nondegenerate critical point $p \in M$, there is a coordinate chart $(U,h)$, called a Morse chart, with $U \ni p$ and $h(p)=0$, so that 
\begin{equation*}
    (V \, \circ h^{-1})(x_1,...,x_n) = V(p) - x_1^2 - x_2^2... - x_{\lambda}^2 + x_{\lambda+1}^2 + ... + x_n^2
\end{equation*}
for $x\in U$, where $h(x) = (x_1,...,x_m) \in h(U)$ and $\lambda$ is the index of $V$ at $p$ (see, e.g., \cite{milnor2016morse}, Lemma 2.2).
In these coordinates, the gradient of $V$ has a unique zero at $p$. Consequently, the critical points of Morse functions are isolated, and the attractors of these gradient systems are asymptotically stable.

\subsubsection{Structure of invariant manifolds} 
\label{subsubsection:structure-invariant-manifolds}
The stable and unstable manifolds of gradients of Morse functions have two key properties that are used to describe generic zero-noise limits of small random perturbations of associative memory models. First, they provide a finite decomposition of a compact manifold into their disjoint union. In addition, they have simple characterizations as embedded discs that vary continuously with perturbations to the Morse potential.

\paragraph{Decomposition.} Let $X = -\nabla_g V$ be a gradient of a Morse function $V \in C^2(M, \mathbb{R})$. Since $M$ is compact and the critical elements of $V$ are isolated, there are finitely many of them: $\text{Crit}(V) = \bigcup_{i=1}^n\beta_i$ for $n>0$ finite. Moreover, the energy $V(\phi^X_t(x))$ is strictly decreasing off critical points, and by compactness, it is bounded below. Therefore, the $\omega-$limit set of each $x\in M$, that is, the set $\omega(x)=\{y \in M \, | \, \phi_t^X(x) \rightarrow y \text{ as }t\rightarrow \infty \}$, is a single point. 

Consequently, every point $x\in M$ lies on the stable manifold of a unique critical point, and the union of these stable manifolds cover $M$. By uniqueness, any two stable manifolds are disjoint: $W^s(\beta_i) \cap W^s(\beta_j) = \emptyset$ for distinct critical points $\beta_i,\beta_j \in \text{Crit}(V)$. It follows that $M$ decomposes into a disjoint union $M = \bigcup_{i=1}^n W^s(\beta_i)$ (\colorAutoref{fig:figure-generic-properties}). The analogous result holds for unstable manifolds by considering asymptotic orbits as $t\rightarrow -\infty$.

\paragraph{Stable Manifold Theorem.} By the Hartman-Grobman theorem, if $p \in M$ is a hyperbolic fixed point of a vector field $X$, then $X$ is locally (topologically) equivalent to its linearization $DX_p$ (\cite{palis2012geometric}, Chapter 2.4). That is, $DX_p$ defines a hyperbolic linear vector field $L(T_pM)$. Any such vector field induces a splitting of the tangent space $T_pM$ at $p$ into stable and unstable directions, $T_pM = E^s \oplus E^u$ where the stable directions have eigenvalues with negative real part and the unstable directions have eigenvalues with positive real part (\cite{palis2012geometric}, Proposition 2.15).

Denote the splitting of the tangent space at $p$ under $L(T_pM)$ by $T_pM = E^s \oplus E^u$. Then for any $q \in E^s$, $L^t(q) \rightarrow 0$ as $t \rightarrow \infty$ and for $q \in E^u$ as $t \rightarrow -\infty$. That is, there is an adapted norm with $\max{\{|| L|_{E^s}||,|| L^{-1}|_{E^u}|| \}}<1$ so that $L|_{E^s}$ is contracting and $L|_{E^u}$ is expanding. For any other $q \notin E^s \cup E^u$, the magnitude $|| L^t(q) || \rightarrow \infty$, so $W^s(0) = E^s$ and $W^u(0)=E^u$. 

The local version of the \textit{Stable Manifold Theorem} says that this behavior of $DX_p$ is captured in a neighborhood of the critical point $p$. That is, for a sufficiently small $\delta > 0$, there are \textit{local stable and unstable manifolds of $p$ of radius $\delta$} defined by the sets
\begin{equation*}
    \begin{split}
    W_{B_{\delta}}^s(p) &= \{x \in B(p,\delta) \, : \, \phi^X_t(x) \in B(p,\delta) \textnormal{ for }t \geq 0 \}  \hspace{1em }\text{and}\\
    W_{B_{\delta}}^u(p) &= \{x \in B(p,\delta) \, : \, \phi^X_{-t}(x) \in B(p,\delta) \textnormal{ for }t \geq 0 \} \, ,
\end{split}
\end{equation*}
\noindent
where $B(p,\delta)$ is a ball of radius $\delta$ centered at $p$, for which the following hold (adapted from \cite{palis2012geometric}, Chapter 2.6):
\begin{enumerate}
    \item $W_{B_{\delta}}^s(p) \subset W^s(p)$ and $W_{B_{\delta}}^u(p) \subset W^u(p)$ are embedded topological discs in $M$ with the same regularity as $X$ whose dimension is equal to the dimensions of the stable and unstable subspaces, $E^s$ and $E^u$, respectively;
    \item $W^s(p) = \bigcup_{t\geq0}\phi_{-t}^X(W_{B_{\delta}}^s(p))$ and $W^u(p) = \bigcup_{t\geq0}\phi_{t}^X(W_{B_{\delta}}^u(p))$. Therefore, there exists injective topological immersions $\varphi_s:E^s \rightarrow M$ and $\varphi_u:E^u \rightarrow M$ whose images are $W^s(p)$ and $W^u(p)$, respectively. 
\end{enumerate}
A standard proof of the global Stable Manifold Theorem establishes the first property and extends the results using the second. It is long but well-documented (e.g., \cite{palis2012geometric} and \cite{shub2013global}); see \cite{smale1967differentiable} for more references. 

If $X$ is the gradient of a Morse function, the global version can be sharpened. Suppose that $\delta > 0$ is such that the local stable manifold theorem holds. Then for any $x \in W_{B_{\delta}}^s(p)$, the set $\bigcup_{t\geq0}\phi_{t}^X(x)$ contains no critical points except for $p$. A similar result holds for the local unstable manifold $W_{B_{\delta}}^u(p)$. A classic result in Morse theory implies that these sets extend smoothly to all of $W^s(p)$ and $W^u(p)$. Moreover, this extension, and thus $W^s(p)$ and $W^u(p)$, vary continuously in the $C^r$ topology with the same regularity as $X$. These results give the following theorem.

\begin{theorem}[Stable Manifold Theorem for Morse functions]
\label{theorem:stable-manifold-theorem-morse}
    Let $X = -\nabla_g V$ be a gradient field on a Riemannian $n$-manifold $(M,g)$ with $V \in C^{r+1}(M,\mathbb{R}), r \geq 1$ a Morse function. Let $p \in M$ be a hyperbolic fixed point of $V$ with index $\lambda_p$, and $E^s$ (resp. $E^u$) the stable (resp. unstable) subspace of $DX_p = L$. Then the following hold,
    \begin{enumerate}
        \item $W^s(p)$ is an embedded differentiable manifold in $M$ with the same regularity as $\phi^X_t$, and the tangent space $T_p(W^s(p))$ at $p$ is $E^s$; hence, 
        $W^s(p)$ and $W^u(p)$ are embedded (open) topological discs with dimensions $n-\lambda_p$ and $\lambda_p$, respectively.
        \item Let $D \subset W^s(p)$ be an embedded disc containing $p$. Then there is a neighborhood $\mathcal{N} \subset \text{Diff}^r(M)$ in which all $g \in \mathcal{N}$ have a unique hyperbolic fixed point $p_g$ contained in a neighborhood $U \ni p$. Hence, for any $\epsilon > 0$, there is a neighborhood $\tilde{\mathcal{N}} \subset \mathcal{N}$ of $\phi_t^X$ such that, for each $g \in \tilde{\mathcal{N}}$, there exists a disc $D_g \subset W^s(p_g)$ that is $\epsilon$-$C^r$ close to $D$.
    \end{enumerate}
\end{theorem}
By \textit{\textbf{$\bm{\epsilon}$-$\bm{C^r}$ close}}, we mean that there exists a $C^r$ diffeomorphism $h:D \rightarrow D_g \subset M$ such that $i_g \circ h$ is contained within an $\epsilon$-neighborhood of $i$ in the $C^r$ topology where $i: D \hookrightarrow M$ and $i_g:D_g \hookrightarrow M$ are inclusions. 

\Cref{theorem:stable-manifold-theorem-morse} is applied in \Cref{section:diffusion-zero-noise-limit}, and a few lemmas, adapted from \cite{milnor2025lectures}, Section 2, are used to prove it in \Cref{appendix:Morse-theory}. Our proof of (1) is motivated by \cite{cohennotes}, Chapter 6.3. Refer to \cite{banyaga2004lectures}, Chapter 4 for a detailed account of (1). The proof of (2) essentially shows that Morse functions are generic, as in \cite{milnor2025lectures}, Theorem 2.7.

\subsubsection{Structural and $\Omega-$stability}
\label{subsubsection:structural-and-omega-stability}
The introduction of Morse-Smale systems in \cite{smale1960morse} was accompanied by a conjecture that its elements were structurally stable. This is proven in the main theorem of \cite{smale1970structural} and applies to Morse-Smale gradients. In particular, \textit{a gradient vector field is structurally stable if and only if it is Morse-Smale}. 

From the definition of non-wandering set, a homeomorphism $h:M \rightarrow M$ establishing a topological equivalence between orbits of $X,Y\in \chi^r(M)$ restricts to a homeomorphism between non-wandering sets: $h|_{\Omega(X)} = \Omega(Y)$. It is instructive to see that, for Morse-Smale systems, since $\Omega(X)$ is the union of the set of critical elements, denoted $\text{Crit}(X)$, the map $h$ restricts to a homeomorphism, $h|_{\text{Crit}(X)} = \text{Crit}(Y)$. This restriction is used in \Cref{subsection:robustness-diffusion-models}

\begin{figure}[t!]
% \vspace{-1.5em}
        \centering
        \includegraphics[width=1\textwidth]{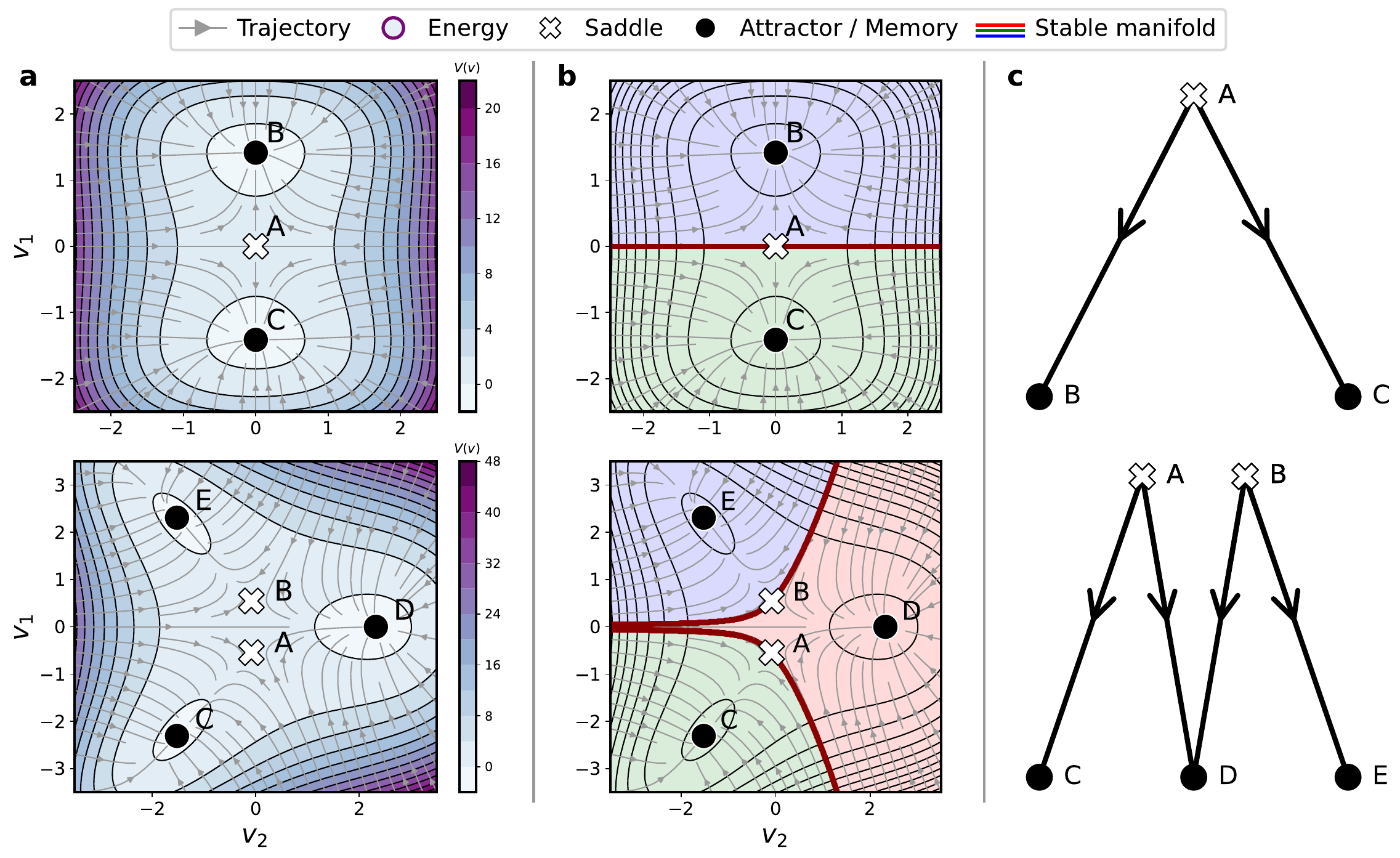}
        \caption{\textbf{Generic phase space decomposition and hierarchical organization (DAG) of associative memory.} {\textbf{(a)} Phase portraits of Morse-Smale gradients overlayed on their energy surfaces. The two attractor system (top) is the dual-well model of \Cref{example:bistable-memories}. A dual cusp geometry describes the three attractor system (bottom) and corresponds to the potential $V= \frac{1}{10} \left( v_1^4 + v_2^4 - 3v_2^3 + 7v_2v_1^2 + \frac{1}{10}v_2^2 - 2v_2 \right)$. Saddles (white crosses) and attractors (black circles) are labeled alphabetically. \textbf{(b)} Decomposition of the phase space into disjoint stable manifolds of each critical element (red, green, and blue shades). \textbf{(c)} Invariant DAGs of the dual-well model (top) and dual cusp model (bottom). Nodes correspond to critical elements and are orded by their index. Top layer nodes correspond to index 1 saddle points with edges to index 0 attractors (memories).}}
        \label{fig:figure-generic-properties}
% \vspace{-1.5em}
\end{figure}

\subsubsection{Hierarchical organization of associative memory}
\label{subsubsection:DAG}
Finally, associated with a Morse-Smale system is a directed acyclic graph (DAG) whose vertices are critical elements $\beta_1,...,\beta_n$. Two vertices are connected by a directed edge, denoted $\beta_i \succ \beta_j$, if there is a trajectory from $\beta_i$ to $\beta_j$ (\colorAutoref{fig:figure-generic-properties}). Equivalently, $\beta_i \succ \beta_j$ if $ W^u(\beta_i) \cap W^s(\beta_j) \neq \emptyset$. The DAG defines a partial order on the set of critical elements with the following properties (\cite{meyer1968energy,smale1960morse}): 
\begin{enumerate}[itemsep=0pt]
    \item[\textbullet] (No self-connections) $\beta_i \nsucc \beta_i$;
    \item[\textbullet] (Transitive edges) if $\beta_i \succ \beta_j$ and $\beta_j \succ \beta_k$, then $\beta_i \succ \beta_k$;
    \item[\textbullet](Ordered by index) if $\beta_i \succ \beta_j$, then $\text{dim}(W^u(\beta_i)) \geq \text{dim}(W^u(\beta_j))$. Equality can occur only if $\beta_j$ is a closed orbit, which is excluded for gradient flows.
\end{enumerate}

Maximal-index repellors occupy the top level of the DAG, which connects hierarchically to intermediate-index saddle points and finally to minimal-index attractors (memories). These DAGs are topological invariants in the following sense. Denote the partial order of $X \in S(M)$ by $P(X)$ and define a \textit{diagram isomorphism} as a bijective, order- and index-preserving map $\rho: P(X) \rightarrow P(Y)$. If two Morse-Smale flows are topologically equivalent, then they are diagram isomorphic -- their stable and unstable manifolds agree (modulo homeomorphism), and their critical elements are in one-to-one correspondence with the same indices. In addition, $P(X)$ is stable up to diagram isomorphism under small $C^r$ perturbations to $X$ (\cite{PALIS1969385}).

%%%%%%%%%%%%%%%%%%%%%%%%%%%%%%%%%%%%%%%%%%%%%%%%
% Morse-Smale models as zero-noise limits %
%%%%%%%%%%%%%%%%%%%%%%%%%%%%%%%%%%%%%%%%%%%%%%%%
\section{Zero-noise limits and stability}
\label{section:diffusion-zero-noise-limit}
Attention is now turned to generic zero-noise limits of small random perturbations of gradient flows like \eqref{equation:riemannian-small-random-perturbation-equation}. The potential will always be assumed to be a Morse function. The Smale condition is not always needed, but it is used to obtain global results on the stability of zero-noise limits in \Cref{subsection:robustness-diffusion-models}.

%---------------------------
% Zero noise limit
%---------------------------
\subsection{Generic zero-noise limits}
\label{subsection:diffusion-models-zero-noise-limits}
A simple hypothesis, apparent in \colorAutoref{fig:figure-intuition-example-zero-noise}, is that zero-noise limits of a family $\{\mathcal{X}^{\epsilon}\}_{\epsilon>0}$ concentrate on attractors of $X$. This is essentially correct, but more can be said. If $V$ is Morse and $M$ is compact, then $M$ decomposes into a union $M = \bigcup_{i=1}^n W^s(\beta_i)$, where $\{\beta_1,...,\beta_n\} = \Omega(X)$ are the isolated critical elements of $X$. Moreover, if $\mu$ is $\phi_t^X$-invariant, then it must concentrate on the non-wandering set $\Omega(X)$ by Poincaré recurrence, i.e., $\mu(\Omega(X))=1$. Since $\Omega(X)$ is a discrete set, $\mu$ should be a weighted sum of Dirac delta functions centered on $\{\beta_1,...,\beta_n\}$. This is the content of \Cref{proposition:zero-noise-limit-MS}.

\begin{proposition}[Generic zero-noise limits.]
\label{proposition:zero-noise-limit-MS}
    Let $(M,g)$ be a Riemannian manifold and $X = -\nabla_g V$ a gradient vector field on $M$ with $V\in C^{r}(M,\mathbb{R})$ a Morse function where $r \geq 2$; let $\{\mathcal{X}^{\epsilon}\}_{\epsilon>0}$ be a family of small random perturbations of $\phi_t^X$ with invariant measures $\mu^{\epsilon}$ for $\mathcal{X^{\epsilon}}$ at a given $\epsilon>0$. Then as $\epsilon \rightarrow 0$, $\{\mu^{\epsilon}\}_{\epsilon>0}$ converges weakly to an invariant measure $\mu$ of $\phi_t^X$ consisting of a weighted sum $\mu=\sum_i^n w_i\delta_{\beta_i}$, where $\{ \beta_1,...,\beta_n \} = \Omega(X)$ is the non-wandering set of $X$.
\end{proposition}
Since $M$ is compact, $\mathcal{X}^{\epsilon}$ are non-explosive and admit invariant measures $\mu^{\epsilon}$. Moreover, $X$ is a gradient system, so the measures $\mu^{\epsilon}$ can be studied analytically through an application of a convenient inequality in \cite{hwang1980laplace} that bounds the measure assigned by a Boltzmann-Gibbs distribution to high-energy states. The following proof is straightforward -- apply this inequality on a stable manifold and extend it globally using the stable manifold decomposition induced by the Morse condition.
\vspace{0.5em}
\begin{proof}
    Since $M$ is compact, each $\mathcal{X^{\epsilon}}$ admits an invariant measure and $\mathcal{M}$, the space of Borel probability measures on $M$ with the weak topology, is compact\footnote{See, e.g., \cite{parthasarathy2005probability}, Page 45.}. Since $\mathcal{M}$ is compact, $\{\mu^\epsilon\}$ is \textit{tight}\footnote{By Prokhorov's theorem; see \cite{rezakhanlou2015lectures}, Chapter 2.4.}. If the sequence $\{\mu^{\epsilon} \}$ does not converge, simply pass to a convergent subsequence $\{\mu^{\epsilon_{n}} \}$. Therefore, $\{\mathcal{X}^{\epsilon}\}$ has a zero-noise limit, but it may not be unique. Call this limit $\mu'$. It remains to show that $\mu'$ is $\mu$ as asserted. By definition of small random perturbation, $p^{\epsilon}(\cdot, x) \rightarrow \delta_{\phi_t^X(x)}$ as $\epsilon \rightarrow 0$. Let $\mu'$ be an invariant probability measure on a stable manifold $W^s(\beta_i)$. From the definition of invariant measure and small random perturbation:
    \begin{align}
        \mu_{\epsilon}(W^s(\beta_i)) &= \int_M p^{\epsilon}(t,x,W^s(\beta_i)) \,\mu_{\epsilon}\left(d\text{vol}_g(x)\right) \qquad t>0,x\in M \nonumber \\ 
        &\quad \rightarrow \int_M \delta_{\phi_t^X(x)} \left( W^s(\beta_i) \right) \, d \mu'(x) \nonumber \\ 
        &\quad \rightarrow \int_{W^s(\beta_i)} \, d \mu'(x)
        \qquad \text{as } \epsilon \rightarrow 0 \nonumber \,.
    \end{align} 
    Suppose $\mu'$ does not concentrate on $\{\beta_i\}$. Let $\{a_{m} \}$ be a sequence converging to $\inf_x V(x) = \beta_i$ on $W^s(\beta_i)$. By assumption, ${P_{\epsilon}(V(x) \geq a_m ) \rightarrow P(V(x) \geq a_m )}$. But,
    \begin{align}
        P_{\epsilon}(V(x) \geq a_{m} ) &= \left[ \int_{W^s(\beta_i)} e^{\frac{-({V(x)})}{\epsilon^2}} d\text{vol}_g(x)\right]^{-1} \left[ \int_{V(x) \geq a_{m}} e^{\frac{{-(V(x))}}{\epsilon^2}} d\text{vol}_g(x)\right] \nonumber \\
        &\quad \leq \left[ \int_{W^s(\beta_i)} e^{\frac{-({V(x)})}{\epsilon^2}} d\text{vol}_g(x)\right]^{-1} \left[ \int_{V(x) < a_{m}} e^{\frac{{- (V(x) - a_{m})}}{\epsilon^2}} d\text{vol}_g(x)\right] \nonumber \\
        &\quad \rightarrow 0 \qquad \text{as } \epsilon \rightarrow 0 \nonumber \, ,
    \end{align} 
    which implies that $P(V(x) \geq a_{m} ) = 0$. The complement of the set $\{ x \in W^s(\beta_i) \, : \, V(x) \geq \beta_i\}$ is $\{ \beta_i\}$, implying that $P(W^s(\beta_i))=0$, which is a contradiction, since $P(W^s(\beta_i))=1$, i.e., $\mu'$ is a valid probability measure on $W^s(\beta_i)$. Therefore $\mu'(W^s(\beta_i))$ has full measure on $\{ \beta_i \}$. That is, the complement $W^s(\beta_i) \setminus \{ \beta_i\}$ has measure zero, and it holds that $\mu'(W^s(\beta_i)) = \mu'(\{ \beta_i \})$. We may write,
    \begin{align}
        \int_{W^s(\beta_i)} \, d \mu'(x) = \int_{W^s(\beta_i)} w_i \delta_{\beta_{i}} \left(d\text{vol}_g(x) \right) \nonumber \, .
    \end{align}
    The argument holds for the stable manifold of each critical element $\beta_i$. Since the union $\bigcup_{i=1}^n W^s(\beta_i)$ covers $M$, the complement has measure zero. Therefore, 
    \begin{equation*}
        \mu' = \sum_i^n w_i \delta_{\beta_i}  \qquad \text{and } \qquad \sum_i^n w_i = 1 \, ,
    \end{equation*}
     where the constraint $\sum_i^n w_i = 1$ ensures that $\mu'$ is a valid probability measure.
\end{proof}

% --- Physical measures and Morse-Smale gradients in standard forms. --- 

\subsection{Physical measures and resonance}
\label{subsubsection:issues-with-globally-defined-measures}
Let $X = -\nabla_g V$ be a gradient system with $V\in C^{r}(M,\mathbb{R}), r\geq 2$ a Morse function. Suppose that $\mu = \sum_i^n w_i \delta_{\beta_i}$ is the zero-noise limit of a family of small random perturbations $\{\mathcal{X^{\epsilon}}\}_{\epsilon >0}$ of the gradient flow $\phi_t^X$ generated by $X$, according to \Cref{proposition:zero-noise-limit-MS}. 

A naïve hope is that the coefficients $w_i$ are determined by the relative volumes of the stable manifolds $W^s(\beta_i)$ of $X$. If a set $A$ of positive Lebesgue measure were chosen uniformly at random, then the expected value of an observable $f\in C^0_b(M,\mathbb{R})$ with respect to this measure could be obtained from the asymptotic behavior of the flow $\phi_t^X(A)$ when $t\rightarrow\infty$. This normalization would make $\mu$ as "close" as possible to normalized Lebesgue measure on $(M,g)$. The idea of a \textit{physical measure} makes this precise and more general:
\begin{definition}[Physical measure, adapted from \cite{young2002srb}]
    Let $\phi:M \rightarrow M$ be measurable and $\mu \in \mathcal{M}$ an invariant probability measure under $\phi$. Then $\mu$ is called a \textbf{\textit{physical measure}} if there exists a positive Lebesgue measure set $A \subset M$ such that for any continuous observable $f\in C^0_b(M,\mathbb{R})$ and $x \in A$,
    \begin{equation*}
        \lim_{T \rightarrow \infty} \frac{1}{T}\int_0^T f(\phi_t(x))\,dt \rightarrow \int f \, d\mu
    \end{equation*}
\end{definition}
The interpretation, commonly attributed to Kolmogorov, is that physical measures are those that are actually observable. In some cases, zero-noise limits are physical measures -- refer to \cite{young2002srb}, Section 2 and Theorem 1 for further information. However, there is no constraint that $w_i > 0$ for the zero-noise limit $\mu = \sum_i^n w_i \delta_{\beta_i}$. Actually, the inequality in the proof implies that $\mu$ concentrates on the attractor(s) with uniformly minimal energy; see \cite{hwang1980laplace} for a related discussion.

Evidently, if $\mu(W^s(\beta_i)) > 0, i=a,b$ for stable manifolds $W^s(\beta_a), W^s(\beta_b)$ of distinct critical points $\beta_a, \beta_b$, then the Morse function $V$ is \textit{resonant}\footnote{\colorAutoref{fig:figure-intuition-example-zero-noise} depicts a resonant Morse function and the zero-noise limit concentrating on the two attractors of the corresponding gradient system.}. That is, there are at least two equal critical values $V(\beta_a)$ and $V(\beta_b)$. In this case, the zero-noise limit is a uniform probability measure with equal weights $w_a$ and $w_b$.

Nonresonant, or \textit{excellent}, Morse functions are generic -- they form an open and dense subset of $C^{2}(M, \mathbb{R})$ (\cite{milnor2025lectures}, Lemma 2.8, \cite{nicolaescu2007invitation}, Chapter 1.2). Therefore, a zero-noise limit $\mu|_{W^s(\beta_i)}$ of $\{ \mathcal{X}^{\epsilon}\}$ \textit{restricted} to a stable manifold $W^s(\beta_i)$ is generically not physical. For example, if $\beta_i$ is not the unique critical point with uniformly minimal critical value (energy), and $A \subset W^s(\beta_i)$ is a subset of positive Lebesgue measure, then the integral $\int f \, d\mu$ will be zero, which is not equal to the time average of an observable $f: M \rightarrow \mathbb{R}$ under the flow $\phi_t^X(A)$ unless $f(\beta_i)$ is zero.

%------------------------------------------
% Perturbation robustness
%------------------------------------------
\subsection{Stability}
\label{subsection:robustness-diffusion-models}
In view of \Cref{subsubsection:issues-with-globally-defined-measures}, the stability of zero-noise limits is now studied on stable manifolds of a gradient system of a Morse function and then globally, enabled by structural stability. On stable manifolds, these zero-noise limits are physical measures. First, it is known that zero-noise limits are closely related to the \textit{stochastic stability} of Borel probability measures (\cite{cowieson2005srb}). 

\begin{definition}[Stochastic stability]
\label{definition:stochastic-stability}
    Let $(M,g)$ be a Riemannian manifold and $\mathcal{M}$ the set of Borel probability measures on $M$ with the weak topology. A probability measure $\mu \in \mathcal{M}$ is \textbf{\textit{stochastically stable}} with respect to small random perturbations from a class of Markov chains $\mathcal{F}$ if for every family of small random perturbations $\{\mathcal{X}^{\epsilon}\}_{\epsilon>0}$ of $f$, with $\mathcal{X}^{\epsilon} \in \mathcal{F}$, the collection of stationary measures $\{\mu^{\epsilon}\}_{\epsilon>0}$ converges weakly to $\mu$ as $\epsilon \rightarrow 0$.
\end{definition}

The following corollary is a consequence of \Cref{definition:stochastic-stability}. It says that zero-noise limits for small random perturbations of generic associative memory models are stochastically stable with respect to white-noise perturbations.
\begin{corollary}[Stochastic stability of zero-noise limits]
\label{corollary:stochastically-stable}
    Let $(M,g)$ be a Riemannian manifold and $X=-\nabla_g V$ be a gradient field with $V \in C^r(M, \mathbb{R}), r\geq 2$ a Morse function. Let $\mathcal{F}$ be the collection of Markov diffusion processes generated by Brownian motion, and $\{\mathcal{X}^{\epsilon}\}_{\epsilon>0}$ be a family of small random perturbations of the gradient flow $\phi_t^X$ with $\mathcal{X}^{\epsilon}\in \mathcal{F}$. Then the zero-noise limit $\mu$ of $\{\mathcal{X}^{\epsilon}\}_{\epsilon>0}$ is stochastically stable.
\end{corollary}
\begin{proof}
    The small random perturbations $\{\mathcal{X}^{\epsilon}\}_{\epsilon>0}$ are arbitrary in \Cref{proposition:zero-noise-limit-MS}. Apply the definition of stochastic stability.
\end{proof}

The following proposition is motivated by \cite{bowen1975ergodic}, Proposition 5.4 and details the stability, or dependence, of zero-noise limits to Morse function perturbations. The proof follows from the Stable Manifold Theorem for Morse functions (\Cref{theorem:stable-manifold-theorem-morse}) combined with the explicit form of zero-noise limits in \Cref{proposition:zero-noise-limit-MS}. 

\begin{proposition}
[Continuous dependence of zero-noise limits]
\label{lemma:zero-noise-continuous-dependence}
    Let $\phi^X_t$ be a Morse gradient flow generated by $X = -\nabla_g V$ for $V \in C^{r+1}(M,\mathbb{R}), r\geq1$ a Morse function. Denote the set of critical elements of $X$ by $\{ \beta_i,...,\beta_n\}$ and let $W^s(\beta_a)$ be a stable manifold for $\beta_a$, $a \in 1,...,n$. If $\{\mathcal{X}^{\epsilon}\}_{\epsilon>0}$ is a family of small random perturbations of $\phi_t^X$. Then the zero noise limit of $\{\mathcal{X}^{\epsilon}\}_{\epsilon>0}$ on $W^s(\beta_a)$ depends continuously on the $C^r$ flow $\phi^X_t$ for the weak topology on measures and the compact-open $C^{r+1}$ topology on real-valued functions.
\end{proposition}

\begin{proof}
    Let $\mu$, and $\mu'$ be zero-noise limits for small random perturbations $\{ \mathcal{X^{\epsilon}}\}_{\epsilon >0}$ and $\{ \mathcal{X'^{\epsilon}}\}_{\epsilon >0}$ of the gradient flows $\phi_t^X$ and $\phi_t^{X'}$, respectively, with $X' \in \text{Grad}^r(M)$. We have to show that the sequence $\{\mu'\}$ converges to $\mu$ weakly on $W^s(\beta_a)$ when $V' \rightarrow V$ in the $C^{r+1}$ sense, where $X' = -\nabla_g V'$. Let $(U_1,\psi_1),...,(U_k,\psi_k)$ be a finite covering of $M$ (by compactness). Choose compact sets $K_i \subset U_i$ such that $\bigcup_i K_i$ covers $M$. From \Cref{lemma:good-critical-points-milnor}, there exists an $\alpha>0$ so that all 
    functions $V' \in C^{r+1}(M,\mathbb{R})$ with
    \begin{equation*}
        || D^j (V \circ \psi^{-1})(x) - D^j(V' \circ \psi^{-1})(x)|| < \alpha
    \end{equation*}
    have nondegenerate critical elements on $M$, where $x \in \psi(K)$ with $K = K_i$ for some $i$ and $j = 0,...,r+1$. This $\alpha$-neighborhood of $V$ induces a neighborhood of $\nabla_gV$ in $C^r(M,TM)$ and similarly on flows by the uniqueness of the solutions to \eqref{equation:vector-field-ode-definition}. That is, for $V'$ within $\alpha$ of $V$ in the $C^{r+1}$ sense above, the flows $\phi_t^X$ and $\phi_t^{X'}$ are within some $\eta >0$. For any $\delta \leq \eta$, non-degeneracy of critical points is preserved since $V'$ is within $\alpha$ of $V$. Call this neighborhood $\mathcal{U}$. Convergence in the compact-open $C^r$ topology is uniform. Therefore for any $\delta>0$, there exists an $N \in \mathbb{Z}^{+}$, such that for all $n \geq N$,
    \begin{equation*}
        || D^l (\psi \circ \phi_t^{X'_{n}} \circ \psi^{-1})(x) - D^l (\psi \circ \phi_t^{X} \circ \psi^{-1})(x) || < \delta \, ,
    \end{equation*}
    for $x \in \psi(K_i)$ for some $K_i$ and $l=0,...,r$. Call this $\delta$-neighborhood $\mathcal{V}$ and let $\mathcal{N = \mathcal{U} \, \cap \, \mathcal{V}}$. Then for any $\phi_t^{X'_{n}}\in \mathcal{N}$, the stable manifolds of its critical elements are well-defined, since the set $\text{Crit}(V'_{n})$ contains no degenerate critical points. By \Cref{proposition:zero-noise-limit-MS}, the zero-noise limit $\mu_n'$ is a weighted sum of Dirac delta functions. By the Stable Manifold Theorem for a Morse function (\Cref{theorem:stable-manifold-theorem-morse}), $\delta$ can be chosen sufficiently small so the stable manifold $W^s(\beta_a^{'n})$ is $\epsilon$-$C^r$ close to $W^s(\beta_a)$ for any $\epsilon > 0$. That is,
    $W^s(\beta_a^{'n}) \rightarrow W^s(\beta_a)$ in the $C^r$ sense as $V'_n \rightarrow V$ in the $C^{r+1}$ sense, and $\beta_a^{'n} \rightarrow \beta_a$. On $W^s(\beta_a)$, $\mu$ is written $\mu = \int_{W^s(\beta_a)} w_a \delta_{\beta_{a}} \left(d\text{vol}_g(x)\right)$ subject to the constraint $\mu(W^s(\beta_a))=1$. Therefore, $w_a=1$. The result follows immediately since $\beta_a^{'n} \rightarrow \beta_a$ and by \Cref{proposition:zero-noise-limit-MS}, $\mu'_n$ is similarly written $\mu'_n = \int_{W^s(\beta^{'n}_a)} w^{'n}_a \delta_{\beta^{'n}_{a}} \left(d\text{vol}_g(x)\right)$ with $w'_a=1$. We have that for any $f \in C^0_b(M)$,
    \begin{equation*}
        \int_{W^s(\beta_a^{'n})} f \,d\mu'_{n} = w_a^{'n} f(\beta_a^{'n}) \rightarrow w_a f(\beta_a) = \int_{W^s(\beta_a)} f \,d\mu.
    \end{equation*}
    Therefore, the subsequence $\{\mu'_{n}\}$ converges weakly to $\mu$. To show that $\{\mu'\}$ converges weakly to $\mu$, proceed by contradiction. Suppose that $\{\mu'\}$ does not converge to $\mu$. Then there exists a $g\in C^0_b(W^s(\beta_a))$, a number $\delta > 0$, and a subsequence $\{\mu'_{k}\}$ so that 
    \begin{equation*}
        \left|\int_{W^s(\beta_a)} g \,d\mu - \int_{W^s(\beta_a^{'k})} g \,d\mu'_{k} \right| \geq \delta
    \end{equation*}
    for all $k \geq 1$. Therefore no subsequence converges, which is a contradiction.
\end{proof}

If a gradient flow $\phi_t^X$ on a Riemannian manifold $(M,g)$ is Morse-Smale, its invariant measures, those of its random perturbations, and their zero-noise limits can be studied globally. Since it is Morse-Smale, it is structurally stable. Therefore, there is a neighborhood $\mathcal{N} \subset \text{Diff}^1(M)$ of $\phi_t^X$ in which each flow $\phi_t^Y$ is topologically equivalent to it by a homeomorphism $h_Y$. Let $(M,\mathcal{B})$ be a measurable space with the natural Borel $\sigma$-field $\mathcal{B}$ on $(M,g)$. When $(M,\mathcal{B})$ is equipped with a probability measure $\mu \in \mathcal{M}$, it is a Borel probability space $(M,\mathcal{B},\mu)$. Recall that the image or \textit{pushforward} of a probability measure $\mu$ under $h_Y$ is given by $h_{Y\#}\mu(A) = \mu(h_Y^{-1}(A))$ for Borel sets $A\in \mathcal{B}$, and that a $\phi_t^X$-invariant measure $\mu$ is a measure that is invariant under its image. That is, $\phi_{t{\#}}^X \mu = \mu$. A flow with an invariant measure is then a flow over the measure space $(M, \mathcal{B}, \mu)$\footnote{See \cite{sinai1989dynamical}, Definition 1.4 for a definition of flows over measure spaces.}, and the pushforward measure allows for a generalization of topological equivalence using measure space isomorphisms, discussed below.

Any homeomorphism $h:M \rightarrow M$ is an automorphism of the measurable space $(M,\mathcal{B})$, since it is a measurable bijection with a measurable inverse.  Similarly, a measure space isomorphism $\psi: (M_X,\mathcal{B}_X,\mu_X) \rightarrow (M_Y,\mathcal{B}_Y,\mu_Y)$ is a measurable bijection with measurable inverse for which the measures agree: $\mu_X(A) = \mu_Y(\psi(A))$ for all $A \in \mathcal{B}_X$. Therefore, $h: (M,\mathcal{B},\mu_Y) \rightarrow (M,\mathcal{B},h_{\#}\mu_Y)$ is a measure space isomorphism, since it is an isomorphism of measurable spaces and $h_{\#}\mu_X(h(A)) = \mu_X(h^{-1}(h(A)))= \mu_X(A)$ for $A \in \mathcal{B}_X$.

\begin{definition}[Metrically isomorphic\footnote{\Cref{definition:metrically-isomorphic} adapts \cite{sinai1989dynamical}, Definition 1.6, where metrically conjugate flows are called metrically isomorphic. "Conjugate" and "equivalent" are distinguished to agree with topologically conjugate and equivalent deterministic flows.}]
\label{definition:metrically-isomorphic}
    Two flows on $\phi_t^X$ and $\phi_t^Y$ on $(M,\mathcal{B})$ with invariant measures $\mu_X$ and $\mu_Y$, are \textbf{\textit{metrically conjugate}} if there exists invariant subsets $M_1 \subset M$ and $M_2 \subset M$ with $\mu_X(M_X) = \mu_Y(M_Y) = 1$ and an isomorphism of measure spaces $h:(M_X,\mathcal{B}_X, \mu_X)\rightarrow (M_Y,\mathcal{B}_Y, \mu_Y)$ so that $\phi_t^Y(h(x)) = h(\phi_t^X(x))$ for all $t \in \mathbb{R}$ and $x \in M_X$. If there are times $t_1,t_2$ so that $h(\phi_{t_1}^X(x)) = \phi_{t_2}^Y(h(x))$ for all $t \in \mathbb{R}$ and $x \in M_X$, then $\phi_t^X$ and $\phi_t^Y$ are \textbf{\textit{metrically equivalent}}.
\end{definition}

Zero-noise limits are now globally described using the Morse-Smale conditions. Two lemmas, proven in \Cref{appendix:strict-isomorphism-weak-convergence}, are needed to prove \Cref{proposition:triangle-commutativity-zero-noise}.

\begin{lemma}[Homeomorphisms preserve weak convergence]
\label{lemma:strict-isomorphism-preserve-weak-convergence}
Let $\mu$ and $\nu$ be Borel probability measures on $(M,\mathcal{B})$, and let $\{\mu_n\}$ be a sequence converging weakly to $\mu$. If $h: M \rightarrow M$ is a homeomorphism and $h_{\#}\mu = \nu$, then $h_{\#}\mu_n$ converges weakly to $\nu$. 
\end{lemma}

\begin{lemma}[Topological equivalence implies metric equivalence]
\label{lemma:metric-equivalence-Morse-Smale}
Let $(M,g)$ be a compact Riemannian manifold, $X$ be a $C^r$ vector field on $M$, and $\mu$ be an invariant Borel probability measure under the flow generated by $X$. If $Y$ is topologically equivalent to $X$ under an orientation preserving homeomorphism $h:M \rightarrow M$, then $h_{\#}\mu$ is invariant under the flow generated by $Y$ and $h: (M, \mathcal{B}, \mu) \rightarrow (M, \mathcal{B}, h_{\#}\mu)$ is a metric equivalence between them as flows over measure spaces. 
\end{lemma}

In summary, homeomorphisms preserve weak convergence of probability measures and, when they define a topological equivalence between two flows, also push forward invariant measures accordingly. Thus, if a flow $\phi_t^X$ on $(M, \mathcal{B}, \mu)$ is topologically conjugate (resp. equivalent) to $\phi_t^Y$, then it is metrically conjugate (resp. equivalent) to $\phi_t^Y$ as a flow on $(M, \mathcal{B}, h_{\#}\mu_X)$.

\begin{proposition}[Continuous dependence of region of convergence]
\label{proposition:triangle-commutativity-zero-noise}
    Let $\phi^X_t$ be a Morse-Smale gradient flow on a Riemannian manifold $(M,g)$ and $\{\mathcal{X}^{\epsilon}\}_{\epsilon>0}$ small random perturbations of $\phi^X_t$ with zero-noise limit $\mu$. A region of convergence of invariant measures $\{\mu^{\epsilon}\}$ to the zero-noise limit of $\{\mathcal{X}^{\epsilon}\}_{\epsilon>0}$ depends continuously on the $C^r$ flow $\phi^X_t, r\geq1$ for the weak topology on measures and the compact-open $C^1$ topology on flows.
\end{proposition}

\begin{proof}
    Let $\mu$ be the zero-noise invariant measure of $\phi^{X}_t$. We aim to show that if $\phi^{X'}_t \rightarrow \phi^X_t$ in the $C^1$ sense, a region of convergence of the sequence $\{\mu^{\epsilon}\}$ to $\mu$ varies continuously in the weak topology on measures and the $C^1$ topology on flows. That is, if $\phi^{X'}_t \rightarrow \phi^X_t$, and the sequence of invariant measures $\{\mu^{\epsilon}\}_{\epsilon >0}$ converges weakly to $\mu$ as $\epsilon \rightarrow 0$, then a sequence $\{\nu^{\epsilon}\}$ converges weakly to a $\phi^{X'}_t$-invariant measure $\nu$. 
    
    Since $\phi^X_t$ is Morse-Smale, it is structurally stable. Therefore, there exists a neighborhood $\mathcal{V}$ of $\phi^X_t$ so that any $\phi^Y_t \in \mathcal{V}$ is topologically equivalent to $\phi_t^X$. Since $\phi^{X'}_t \rightarrow \phi^X_t$ uniformly, for any neighborhood $\mathcal{U}$ of $\phi^X_t$ there exists an $N$ so that for all $n \geq N$ the flow $\phi_t^{X'_{n}}$ is in $\mathcal{U}$. Set $\mathcal{N} = \mathcal{U} \cap \mathcal{V}$. Then each $\phi_t^{X'_{m}} \in \mathcal{N}$ is topologically equivalent to $\phi_t^X$. Therefore, each $\phi_t^{X'_{m}} \in \mathcal{N}$ can be equipped with a homeomorphism $h_{m}:M \rightarrow M$ so that for times $t$ and $t_m$, the following holds: $h_m(\phi_t^{X}(x)) = \phi_{t_m}^{X'_{m}}(h_m(x))$. Since $h_m$ is a topological equivalence and $\mu$ is $\phi_t^X$-invariant, \Cref{lemma:metric-equivalence-Morse-Smale} implies that the image of $\mu$ under $h_m$, denoted $\nu_{m} = h_{m\#}\mu$, is invariant under $\phi_{t_m}^{X'_{m}}$ so that $h_m$ establishes a metric equivalence on flows over measure spaces $(M,\mathcal{B},\mu)$ and $(M,\mathcal{B},\nu_m)$. Since $\mu$ is the zero-noise limit of $\{\mathcal{X}^{\epsilon}\}_{\epsilon >0}$, the sequence $\{\mu^{\epsilon}\}$ converges to $\mu$ weakly, by definition. Given a convergent subsequence $\{\mu^{\epsilon_{k}}\}$, define $\nu_m^{\epsilon_{k}} = h_{m\#}\mu^{\epsilon_{k}}$. By \Cref{lemma:strict-isomorphism-preserve-weak-convergence}, for each $\nu_m$, we have that $\nu_m^{\epsilon_{k}} \rightarrow \nu_m$. Application of the subsequence principle from \Cref{lemma:zero-noise-continuous-dependence} gives that $\nu_m^{\epsilon} \rightarrow \nu_{m}$ for each $\nu_m$.
\end{proof}

Let $X = -\nabla_g V$ be a Morse-Smale gradient field and $\{ \mathcal{X}^{\epsilon}\}_{\epsilon > 0}$ be a family of small random perturbations of the flow $\phi_t^X$ with zero-noise limit $\mu$. \Cref{proposition:triangle-commutativity-zero-noise} says that there is a neighborhood $\mathcal{N}$ of $\phi_t^X$ in which each $\phi_t^Y \in \mathcal{N}$ is topologically equivalent to $\phi_t^X$ and can be equipped with a homeomorphism that induces a commutative diagram on probability measures:
\begin{equation*}
    \begin{tikzcd}[column sep=6.5em, row sep=3.em]
    \mu^{\epsilon} \arrow[r, "w", "\epsilon \rightarrow 0" swap] \arrow[d, "h_{\#}"]
    & \mu \arrow[d, "h_{\#}"]  \\
    h_{\#}\mu^{\epsilon}  \arrow[r, "w"]
    & \nu
    \end{tikzcd} 
\end{equation*}
where the measures on the left-hand side are invariant for stochastic, or generative models, those on the right-hand side are invariant for associative memory models, and horizontal arrows $\xrightarrow{w}$ indicate weak convergence in the limit of vanishing noise.

\begin{remark}
    The image of $\mu^{\epsilon}$ under $h$, that is, $\nu^{\epsilon} = h_{\#}\mu^{\epsilon}$, at a given noise-level $\epsilon>0$, may not be a Boltzmann-Gibbs distribution or even absolutely continuous with respect to Lebesgue measure on $(M,g)$.
\end{remark}

% The properties of $\nu$ in \Cref{proposition:triangle-commutativity-zero-noise} and whether it is a zero-noise limiting measure are given as follows. Since $M$ is compact and $V$ is Morse, the non-wandering set $\Omega(X)$ is a finite number of hyperbolic critical points, $\{\beta_1,...,\beta_n\}$. By \Cref{proposition:zero-noise-limit-MS}, the zero-noise limit $\mu$ is a finite sum of delta functions: $\mu = \sum_{i=1}^n w_i \delta_{\beta_i}$, with $\sum_{i=1}^n w_i = 1$. Since $\phi_t^X$ and $\phi_t^Y$ are topologically equivalent, the restriction of $h$ to the non-wandering set of $X$ gives $h|_{\Omega(X)} = {\Omega(Y)}$ where $\Omega(Y)$ is also a union of hyperbolic critical elements. That is, $\Omega(Y) =  \bigcup_{i=1}^n\{ \beta'_i\} =  \bigcup_{i=1}^n h(\{\beta_i\})$. Consequently, $Y$ is a gradient system $Y = -\nabla_{g'}V'$ for a Riemannian metric $g'$ on $M$ and a Morse function $V'$. 

The properties of the measure $\nu$ in \Cref{proposition:triangle-commutativity-zero-noise}, and whether it arises as a zero-noise limit, are summarized as follows. Since $M$ is compact and $V$ is Morse, the non-wandering set $\Omega(X)$ consists of a finite set of hyperbolic critical points, $\{\beta_1, \dots, \beta_n\}$. By \Cref{proposition:zero-noise-limit-MS}, the zero-noise limit $\mu$ of small random perturbations $\{\mathcal{X}^{\epsilon}\}_{\epsilon > 0}$ of the flow $\phi_t^X$ is a discrete measure supported on these points: $\mu = \sum_{i=1}^n w_i \, \delta_{\beta_i}$, with $\sum_{i=1}^n w_i = 1$. Since $\phi_t^X$ and $\phi_t^Y$ are topologically equivalent, the homeomorphism $h$ maps the non-wandering set of $X$ onto that of $Y$: $h|_{\Omega(X)} = \Omega(Y)$, where $\Omega(Y)$ also consists of finitely many hyperbolic critical points, say $\{\beta_1', \dots, \beta_n'\}$, with $\beta_i' = h(\beta_i)$. Therefore, $Y$ is also a gradient field $Y = -\nabla_{g'} V'$, for some Riemannian metric $g'$ and Morse function $V'$ on $M$.

By definition of pushforward measure, $h_{\#}\mu = \mu(h^{-1}(A))$ for $A \in \mathcal{B}$. Consequently, $h_{\#}\mu = \sum_{i=1}w'_i\delta_{h(\beta_i)}$ where $w'_i$ is equal to the coefficient $w_i$ for which $h^{-1}(\beta'_i) = \beta_i$. In addition, $V$ is generically an excellent Morse function, so the critical points of $X$ can be ordered by their critical values: $V(\beta_1) > V(\beta_2) > ... > V(\beta_n)$, and similarly for $V'(\beta_i')$. This correspondence yields the following results:
\begin{enumerate}[itemsep=0pt]
    \item $\nu$ is the zero-noise limit of small random perturbations $\{\mathcal{Y}^\epsilon\}_{\epsilon > 0}$ of the flow $\phi_t^Y$ if and only if $V'(h(\beta_n)) = \inf_i V'(\beta_i')$; that is, if and only if the critical point $\beta_n$ of $X$, which has uniformly minimal energy under $V$, is mapped by $h$ to the global minimum of $V'$.
    \item On the stable manifold $W^s(\beta_i)$ of a critical element $\beta_i$ of $X$, the limiting measure $\nu$ coincides with the zero-noise limit $\mu$ of the family of small random perturbations $\{ \mathcal{X}^{\epsilon}\}_{\epsilon>0}$ since the pushforward $h_{\#}\mu$ preserves supports.
\end{enumerate}

%------------------------------------------
% Generic arcs connecting Morse-Smale gradient flows
%------------------------------------------
\section{Generation, learning, and stable families of gradients}
\label{section:generic-arcs}
\Cref{section:Morse-Smale-gradient-assocaiative-memory} and \Cref{section:diffusion-zero-noise-limit} use the generic properties of associative memory implied by the Morse-Smale conditions to jointly characterize: (i) a generation-to-memory transition at vanishing noise levels, and (ii) the robustness of the critical point structure of these systems to stochastic and deterministic perturbation. This section uses the classification of bifurcations that occur when these conditions are violated to capture the learning and generation dynamics of associative memory and generative diffusion models. Let $(M,g)$ be a Riemann manifold and denote the space of smooth symmetric $(0,2)$-tensors on $M$ by $C^{\infty}(M,S^{2}(T^{*}M))$. The space of Riemannian metrics is an open subset of this space and is denoted by $\mathcal{R}(M)$ (\cite{tuschmann2015moduli}). 

Let $\theta = \{\theta_g, \theta_V, \theta_t\}$ parameterize a Riemannian metric $g_{\theta} \in \mathcal{R}(M)$, a real-valued function $V_{\theta} \in C^\infty(M,\mathbb{R})$, and an optional time parameter $t \in \mathbb{R}$. Let $\mathcal{L}: \theta \rightarrow \mathbb{R}$ be a smooth loss function. We consider four scenarios defined by families of vector fields $X_{\theta} := -\nabla_{g_{\theta}} V_{\theta}$ and diffusion processes with gradient drift like \eqref{equation:riemannian-small-random-perturbation-equation}:

\begin{example}[Optimizing time-independent systems]
\label{example:one-parameter-family-associative-memory-optimization}
    Gradient-descent in $\eta \in \mathbb{R}$ with respect to the parameters $\theta$,
    \begin{equation*}
        \dot{\theta_\eta} = -\nabla_{\theta_\eta} \mathcal{L}(\theta_\eta) \, ,
    \end{equation*}
    produces a one-parameter family of gradients $\{ X_{\theta_{\eta}} \}_{\eta\in \mathbb{R}}$, with $X_{\theta_{\eta}} := -\nabla_{g_{\theta_{\eta}}} V_{\theta_{\eta}}$, that vary smoothly with $\eta$. For $\eta_1 < \eta_2$, the family $\{ X_{\theta_{\eta}} \}_{\eta\in [\eta_1,\eta_2]}$ traces a curve in the space of smooth vector fields. For a fixed $\epsilon >0$, a one-parameter family $\{\mathcal{X}^{\epsilon}_{\eta}\}_{\eta \in \mathbb{R}}$ is also obtained with
    \begin{equation*}
        d{x_t^{\epsilon}} = -\nabla_{g_{\theta_{\eta}}}V^{\epsilon}_{{\theta_{\eta}}}(x_t^{\epsilon})\,dt + \epsilon\sigma_{\eta}(x_t^{\epsilon})\,dw_t \, .
    \end{equation*}
\end{example}  

\begin{example}[Generation with time-varying drift]
\label{example:one-parameter-family-generation}
    Let $\{g_t \}_{t\in\mathbb{R}}$ and $\{V_t \}_{t\in\mathbb{R}}$ vary smoothly with time. Define $X_t:=-\nabla_{g_t}V_t$. Solutions to the time-dependent equation,
    \begin{equation*}
        d{x_t^{\epsilon}} = -\nabla_{g_t}V^{\epsilon}_t(x_t^{\epsilon})\,dt + \epsilon\sigma_t(x_t^{\epsilon})\,dw_t \, ,
    \end{equation*}
     define a family $\{ \mathcal{X}^{\epsilon}_{t}\}_{t\in\mathbb{R}}$. At $\epsilon=0$, a one-parameter family of gradients is obtained.
\end{example}  

\begin{example}[Optimizing time-varying associative memory]
\label{example:two-parameter-family-associative-memory-optimization}
    When $\theta_g, \theta_V$ vary smoothly in time $t\in \mathbb{R}$ they produce a one-parameter family of gradients, $\{X_{\theta_t}\}_{t\in \mathbb{R}}$, where $X_{\theta_t}\ := -\nabla_{g_{_{\theta_t}}} V_{\theta_t}$. Gradient descent in $\eta$ gives a two-parameter family of gradients $\{ X_{\theta_{t_{\eta}}} \}_{t,\eta\in \mathbb{R}}$ varying smoothly with both the (gradient descent) time index $\eta$ and $t$. For $\eta_1 < \eta_2$ and $t_1 < t_2$ the family $\{ X_{\theta_{t_{\eta}}} \}_{(t,\eta) \in [t_1,t_2] \times [\eta_1,\eta_2]}$ sweeps out a two-parameter subset in $\chi^{\infty}(M)$.
\end{example}  

\begin{example}[Optimizing time-varying diffusion models]
\label{example:two-parameter-family-explicit}
    Building on \Cref{example:one-parameter-family-generation}, gradient descent on $\theta$ produces a two-parameter family indexed by the "inner" time $t\in \mathbb{R}$ and "outer" optimization time $n\in\mathbb{R}$. At a parameter value $\eta$, the dynamics are written,
    \begin{equation*}
        d{x_t^{\epsilon}} = -\nabla_{g_{\theta_{t_{\eta}}}}V^{\epsilon}_{\theta_{t_{\eta}}}(x_t^{\epsilon})\,dt + \epsilon\sigma_{\theta_{t_{\eta}}}(x_t^{\epsilon})\,dw_t \, .
    \end{equation*}        
     At $\epsilon = 0$, a two-parameter family of gradients is obtained, with $X_{\theta_{t_{{n}}}} := -\nabla_{g_{\theta_{t_{\eta}}}}V_{\theta_{t_{\eta}}}$. 
\end{example}  

These examples indicate that learning processes of associative memory models are one-parameter families of gradients, and the generation process of diffusions with time-varying drift are one-parameter families. The learning processes of associative memory models and diffusions with time-varying drift are two-parameter families. One arrives at the following question: \textit{Are one- and two-parameter families of gradients stable?}

This section describes that stable families are abundant -- one- and two-parameter families of gradients are generically stable, and their instability is generically characterized by small sets of bifurcations. A statement is made for diffusion models but requires a comparison of their trajectories to deterministic ones at vanishing noise levels.

% --- Paths large deviations --- 

\subsection{Large deviations and stochastic flows}
\label{subsection:paths-large-deviations}
It is convenient to view solutions to equations like \eqref{equation:riemannian-small-random-perturbation-equation} as \textit{stochastic flows}, denoted by $\Phi^{X,\epsilon}$ for an $\epsilon >0$; see \cite{kunita1986lectures}. Let $(M,g)$ be a Riemannian manifold and $\Omega$ the space of continuous paths $C^0([0,\infty),M)$ equipped with the Wiener measure associated with Brownian motion on $M$. If $0 \leq s \leq t \leq T$ with $T>0$ and $\omega \in \Omega$, then $\Phi^{X,\epsilon}_{s,t}(\cdot, \omega)$ is a measurable map from $M$ to itself. Consequently, $\Phi^{X,\epsilon}$ is a stochastic flow of measurable maps, here, more strictly, of homeomorphisms or diffeomorphisms. 

A \textit{path} $\{x_t \}_{t=0}^T$ is a realization of the stochastic flow $\Phi^{X,\epsilon}(x_0,\omega)$ generated by $\mathcal{X}^{\epsilon}$ for an initial condition $x_0\in M$ and $\omega \in \Omega$. That a family $\{\mathcal{X}^{\epsilon}\}_{\epsilon>0}$ is a small random perturbation implies that the probability of observing a particular path should concentrate as $\epsilon \rightarrow 0$. Original work on this topic by \cite{ventsel1970small} is formulated through an action principle, where an action functional $J[x]$ is defined by
\begin{equation*}
    J[x] = \frac{1}{2}\int_0^T || \dot{x}_t - b({x}_t) ||^2 \, dt \, ,
\end{equation*}
where $|| \cdot ||$ is the Riemannian norm, and the bracket notation indicates that $J[x]$ is a functional over all of $x_t$; see also \cite{touchette2009large}, Chapter 6. 

The large deviation principle states that the probability $P_{\epsilon}[x]$ of observing a path $\{x_t \}_{t=0}^T$ that deviates from an enclosing $\delta$-tube is given by $P_{\epsilon}[x] \asymp e^{-a_{\epsilon}J[x]}$, where the notation $\asymp$ indicates that the dominant part of $P[x]$ is decaying exponentially; $a_{\epsilon}$ determines the rate at which $P[x]$ decays and is such that $a_{\epsilon} \rightarrow \infty$ as $\epsilon \rightarrow 0$. For any number $\delta > 0$,
\begin{equation*}
    P\left( \sup_{0\leq t \leq T} | x^{\epsilon}_t - x_t| < \delta \right ) \asymp e^{-a_{\epsilon}J[x]} \,.
\end{equation*}
Importantly, $J[x]$ has a unique zero when the path is equal to the trajectory solving the deterministic equations, denoted $x^*_t$ for clarity. Since $P_{\epsilon}[x]$ is dominated by $e^{-a_{\epsilon}J[x]}$, which decays exponentially as $\epsilon \rightarrow 0$, the path converges in probability to the deterministic one in this limit. That is, for any $\delta>0$, $\lim_{\epsilon \to 0} P\left(|| x^{\epsilon}_t - x^*_t||_{\infty}\geq\delta\right) = 0$ in the supremum norm $||\cdot||_{\infty}$ on the space of continuous paths $C^0([0,T],M)$. 

This pathwise convergence is preserved by homeomorphisms that conjugate two limiting deterministic flows; see \Cref{appendix:strict-isomorphism-weak-convergence} for a proof:

\begin{proposition}[Preservation of pathwise convergence]
\label{proposition:pathwise-convergence-preservation}
    Let $(M,g)$ be a closed Riemannian manifold; let $\{x_t\}_{t=0}^{T}$ and a path of a small random perturbation $\mathcal{X}^{\epsilon}$ of $\phi_t^X$ such that $x^{\epsilon}_t \xrightarrow[]{P} x^*_t$ in probability as $\epsilon \rightarrow 0$. Let $h:M \rightarrow M$ be a homeomorphism that conjugates $\phi_t^X$ and $\phi_t^Y$ for $Y$ a vector field on $M$. Then $h(x^{\epsilon}_t) \xrightarrow[]{P} y^*_t$ in probability as $\epsilon \rightarrow 0$.
\end{proposition}

The image of a diffusion process under a homeomorphism may not be a diffusion process, but it remains a continuous-time Markov process (\cite{burke1958markovian}, Corollary 3). Let $\overline{p}^{\epsilon}_t(\cdot \, | \, x)$ be the transition density of a transformed process satisfying the Chapman-Kolmogorov equation. It can be shown from the definition of a pushforward measure and homeomorphism that, if $\mu^{\epsilon}$ is invariant for $\mathcal{X}^{\epsilon}$ on $(M,g)$, then $h_{\#}\mu^{\epsilon}$ is invariant for the Markov process with transition density $\overline{p}^{\epsilon}_t(\cdot \, | \, x)$; see \Cref{proposition:triangle-commutativity-zero-noise} for properties of $h_{\#}\mu^{\epsilon}$.

\Cref{proposition:pathwise-convergence-preservation} suggests an equivalence of stochastic flows that is "pinned down" by their behavior at diminishing noise levels. A modest relaxation of pathwise conjugacy is to require that two stochastic flows converge in probability modulo orientation-preserving homeomorphisms without a conserved time parametrization. This gives a notion of equivalence at vanishing noise levels comparable to topological equivalence of deterministic flows:

\begin{definition}[Equivalent in the zero-noise limit]
    Let $(M,g)$ be a Riemannian manifold. Two stochastic flows \( \Phi_t^{X, \epsilon} \) and \( \Phi_t^{Y, \epsilon} \) are \textbf{equivalent in the zero-noise limit} if:
\begin{enumerate}[itemsep=0pt]
    \item (Convergence to deterministic flows) The stochastic flow \( \Phi_t^{X, \epsilon} \) converges in probability to \( \phi_t^X \) as \( \epsilon \to 0 \), and \( \{\Phi_t^{Y, \epsilon}\} \) converges in probability to \( \phi_t^Y \) as \( \epsilon \to 0 \). That is, for any \( t \), \( \delta > 0 \), and $x\in M$, ${\lim_{\epsilon \to 0} P \left( \|\Phi_t^{X, \epsilon}(x, \cdot) - \phi_t^X(x)\| \geq \delta \right) = 0} $, and similarly for \( \Phi_t^{Y, \epsilon} \); and
    
    \item (Convergence modulo homeomorphism) There exists a homeomorphism $h: M \rightarrow M$ preserving orientation and times $t_1,t_2 > 0 $ so that, for any $T > 0$ and $\delta > 0$, ${\lim_{\epsilon \to 0} P \left(  \| h(\Phi_{t_1}^{X, \epsilon}(x,\cdot)) - \Phi_{t_2}^{Y, \epsilon}(h(x),\cdot) \| \geq \delta \right) = 0 }$.
\end{enumerate}
\end{definition}
The definition is non-vacuous. It is well defined for stochastic flows that converge to deterministic ones at vanishing noise levels, satisfying (1). Topological equivalence also implies equivalence in the zero-noise limit; see \Cref{appendix:strict-isomorphism-weak-convergence} for a proof:

\begin{proposition}
\label{proposition:topological-eq-implies-zero-noise-equivalent}
    Let $\phi_{t_1}^{X},\phi_{t_2}^{Y} \in \text{Diff}^r(M), r\geq 1$ be topologically equivalent flows on $(M,g)$ under a homeomorphism $h: M \rightarrow M$. If $\mathcal{X}^{\epsilon}$ and $\mathcal{Y}^{\epsilon}$ are small random perturbations, then the stochastic flows $\Phi_{t_1}^{X,\epsilon}$ and $\Phi_{t_2}^{Y,\epsilon}$ are equivalent in the zero-noise limit.
\end{proposition}
\noindent
The content of \Cref{proposition:topological-eq-implies-zero-noise-equivalent} is summarized by the following commutative diagram:
\begin{equation*}
    \begin{tikzcd}[column sep=6.5em, row sep=3.em]
    \Phi^{X,\epsilon}_{t_1} \arrow[r, "P", "\epsilon \rightarrow 0" swap] \arrow[d, "h", "P" swap]
    & \phi^{X}_{t_1} \arrow[d, "h"]  \\
    \Phi^{Y,\epsilon}_{t_2}  \arrow[r, "P", "\epsilon \rightarrow 0" swap]
    & \phi^{Y}_{t_2}
    \end{tikzcd} 
\end{equation*}
Horizontal arrows denote convergence in probability as $\epsilon \rightarrow 0$. The downward arrow between stochastic flows indicates convergence in probability under the homeomorphism $h$, and the downwards arrow between deterministic flows $\phi_t^X$ and $\phi_t^Y$ indicates topological equivalence. 

Conversely, if two stochastic flows are equivalent in the zero-noise limit, then the probability that the deterministic (limiting) flows are not diminishes. Consequently, the probability that structural stability does not hold vanishes. Stability and bifurcations in the vanishing noise limit are understood in this manner in the rest of the document.

% --- Stable families of gradient vector fields --- 

\subsection{Stable families and bifurcations}
\label{subsubsection:stable-families-of-gradients}
Recall that a $k$-parameter family of gradients on a Riemannian manifold $(M,g)$ is a family $\{ X_{\eta_1,...,\eta_k} \}$, $\eta \in N^k$, where $N$ is a compact $k-$manifold, for which there exists a family of functions $\{f_{\eta_1,...,\eta_k} \}$ and family of metrics $\{ g_{{\eta_1,...,\eta_k}}\}$ such that $X_{\eta_1,...,\eta_k} = -\nabla_{g_{\eta_1,...,\eta_k}} f_{\eta_1,...,\eta_k}$. Pairs $\left ( \{f_{\eta}\}, \{g_{\eta} \}\right )$ are given the compact-open $C^{\infty}$ topology. 

Let $\chi^g_k(M)$ be the set of $k$-parameter families on $M$ and $\pi: M \times N \rightarrow N$ the canonical projection map. Two $k$-parameter families, $\{ X_{\eta} \},\{ Y_{\eta} \} \in \chi^g_k(M)$, are called \textit{equivalent} if there exist homeomorphisms, $H: M \times D \rightarrow M \times D$ and $\psi: N \rightarrow N$, so that $\pi H = \psi \pi$, and, for each $\eta \in N$, $h_{\eta}$ is an equivalence between $\{ X_{\eta} \}$ and $\{ Y_{\psi(\eta)} \}$, where $H(x,\eta) = \left( h_{\eta}(x), \psi(\eta)\right)$. A $k$-parameter family of gradients $\{ X_{\eta} \}$ is called \textit{stable} if there exists a neighborhood $\{ X_{\eta} \} \ni \mathcal{U}$ in which all elements are equivalent to $\{ X_{\eta} \}$. A value $\eta = (d_1,...,d_k)$ is called a \textbf{\textit{regular parameter value}} for $\{X_{\eta}\}$ if $X_{\eta}$ is (structurally) stable. Otherwise, it is called a \textbf{\textit{bifurcation value}}, and is denoted by $\bar{\eta}$.

% --- Comparison to universal unfoldings --- 

\subsubsection{Comparison to universal unfoldings}
\label{subsubsection:universal-unfoldings}
It is worth comparing bifurcations of families of gradients to Thom's work on universal unfoldings. Let $X = -\nabla V$ be a a gradient system with a degenerate singularity $p \in M$, where the Hessian determinant $| \, \partial^2 V / \partial x_i \partial x_j \, |$ vanishes. Under an arbitrary perturbation $\delta V$, the local behavior (formally, the \textit{germ}) of $V + \delta V$ has either infinite or finite topological types. In the finite case, there exists a $k-$parameter \textit{universal unfolding} of the potential,
\begin{equation*}
    V = V(p) + u_1g_1(p) + u_2g_2(p) + ... + u_kg_k(p) \, ,
\end{equation*}
from which any perturbation can be recovered up to topological equivalence (\cite{thom1969topological}, \S 1.D). These unfoldings, describing local behavior of parameterized potentials, form the basis of Thom’s approach, which underpins catastrophe theory (\cite{zeeman2006classification}), and are relevant to describe symmetry breaking phenomena observed in associative memory and diffusion models.

Contrastingly, loss of stability for families of gradients can occur due to local or global criteria. A degenerate singularity indicates a loss of stability of a gradient system, but it is local. Additional global bifurcations can occur, even for one-parameter families. Generically, these are due to non-transverse intersections of certain invariant manifolds, which are discussed in the following sections.

% --- Stable families of gradient vector fields --- 

\subsubsection{One-parameter families of gradients}
\label{subsubsection:one-parameter-families-gradients}

\noindent
Only three bifurcations are needed to build a Morse-Smale system from the trivial one (\cite{smale1960morse,rand2021geometry}). Generically, only two arise at bifurcation values for one-parameter families of gradients, described by the following theorem: the saddle-node bifurcation, a local bifurcation, and the heteroclinic flip, a global bifurcation.
\begin{theorem}[\cite{palis1983stability}]
    The following properties are generic, and hold on an open and dense subset, for one-parameter families of gradients:
    \begin{enumerate}[itemsep=0pt]
        \item The set of regular parameter values is open and dense in $\mathbb{R}$;
        \item At a bifurcation value, the gradient $X_{\bar{\eta}}$ has either exactly one (nonhyperbolic) saddle-node singularity, while all other singularities are hyperbolic and all stable and unstable manifolds intersect transversally, or it has exactly one orbit along which a stable and an unstable manifold intersect nontransversally, while all singularities are hyperbolic.
    \end{enumerate}
\end{theorem}
\noindent
A metric rescales and reorients space to align a potential gradient to a flow, but the potential determines the non-degeneracy of its critical elements. Transversal intersections and orbits of tangency are determined by both the potential and the metric. 

\paragraph{Saddle-node bifurcation.}
In the supercritical case, a saddle-node bifurcation results in the appearance of a repelling or attracting critical element and an index 1 saddle as a parameter is varied. In the subcritical case, an attracting or repelling element is destroyed along with the index 1 saddle. Let $\overline{\eta}$ be a bifurcation value for a one-parameter family of gradients $\{X_{\eta}\}$ and $s \in M$ a nonhyperbolic singularity. For regular parameter values $\eta$ near $\overline{\eta}$, there exists a one-dimensional center manifold $W_{\eta}^c(s)$ dependent on $\eta$ (i.e., the zero eigenspace of $DX_s$)\footnote{See, e.g., \cite{guckenheimer2013nonlinear}, Chapter 3 on center manifolds.}. Restricted to $W_{\eta}^c(s)$, $X_{\eta}$ takes the form:
\begin{equation*}
    X_{\eta}(x) = \left( ax^2 - b(\eta - \overline{\eta}) \right) \frac{\partial}{\partial x} + \mathcal{O}\left( |x|^3 + |x \cdot (\eta - \overline{\eta})| + |\eta-\overline{\eta}|^2 \right) \, ,
\end{equation*}
with $a\neq 0$. If $b \neq 0$ then the saddle-node unfolds generically (compare to \Cref{subsubsection:universal-unfoldings}). Restricted to a center manifold passing through $s\in M$, a saddle-node has the normal form ($\eta=\overline{\eta}$) ${X_{\overline{\eta}}(x) = ax^2 \frac{\partial}{\partial x} + \mathcal{O}\left( |x|^3 \right)}$.

\begin{example}[Saddle-node bifurcation, \colorAutoref{fig:figure-intuition-saddle-node-bistable}]
\label{example:saddle-node}
Consider a one-parameter family of gradients $\{ X_{\eta} \}_{\eta\in [-1,1]}$, where $X_{\eta} := -\nabla_{g} V_{\eta}$ is a smoothly varying family of perturbations of an adapted dual well model from \Cref{example:bistable-memories}, with $g_{ij}=\delta_{ij}$ fixed:
\begin{equation*}
\label{equation:example-saddle-node-associative-memory-dual-well}
    {\dot{v^i}} = - \sum_{j=1}^{n} g^{ij} \frac{\partial V(v, \eta)}{\partial v^j} \hspace{0.5em} 
    \mathrm{with} \hspace{0.5em}  V(v, \eta) = \dfrac{1}{6} v_{1}^{4} + \dfrac{1}{6} v_{2}^{4} - \dfrac{1}{2}v_{2}v_{1}^{2} - \dfrac{1}{2}v_{2} + \sum_{i=3}^{n}v_{i}^{2} + \eta\left(\dfrac{6}{10}v_1\right) \, .
\end{equation*}
As $\eta$ varies from $\eta=-1$ to $\eta=0$, a supercritical fold bifurcation produces a new saddle point and attractor. From $\eta=0$ to $\eta=1$ a subcritical bifurcation eliminates the saddle point and the opposite attractor. The disappearance and appearance of attractors is reflected in the stationary measures of small random perturbations $\mathcal{X}^{\epsilon}_{\theta_{\eta}}$ at a noise level of $\epsilon=1$. 
\end{example}

\begin{figure}[t!]
        \centering
        \includegraphics[width=1\textwidth]{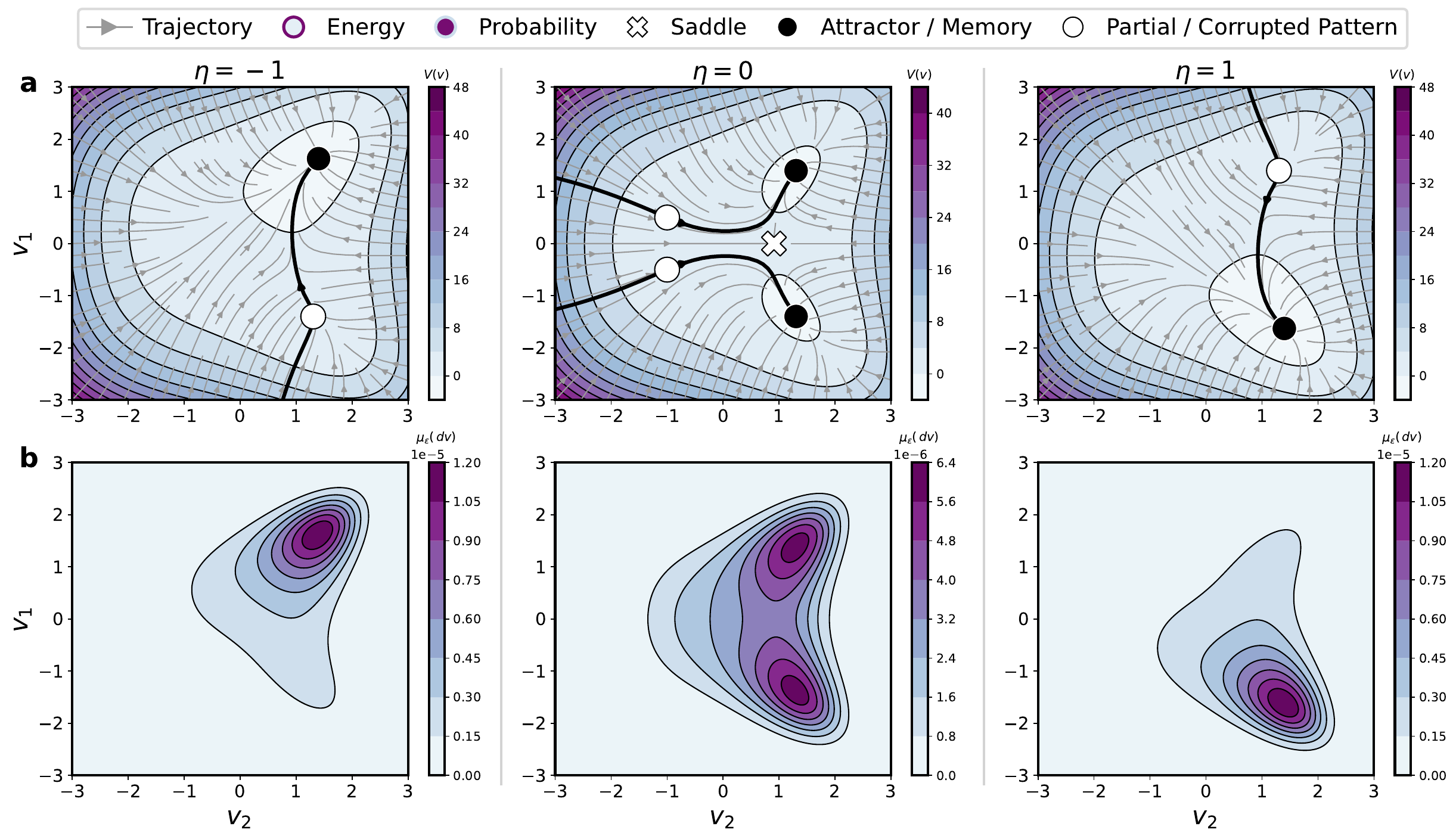}
        \caption{\textbf{Saddle-node (fold) bifurcation.} {\textbf{(a)} Trajectories (grey) representing solutions to the one-parameter family of gradients from \Cref{example:saddle-node} (left to right) overlayed on their respective energy surfaces. As the parameter value $\eta \in [-1,1]$ changes from $\eta=-1$ to $\eta=0$ (left to middle), an attractor and saddle are born indicating a supercritical fold bifurcation. From $\eta=0$ to $\eta=1$, the saddle and opposite attractor are destroyed, corresponding to the subcritical case. \textbf{(b)} Invariant measures $\mu^{\epsilon}(dv)$ of small random perturbations $\mathcal{X}^{\epsilon}$ of each gradient flow from $\eta=-1$ to $\eta=1$ with $\epsilon=1$. As $\epsilon \rightarrow 0$, the invariant measure will concentrate on the attractors.}}
        \label{fig:figure-intuition-saddle-node-bistable}
\end{figure}

\paragraph{Heteroclinic flip bifurcation.}
Let $p,q \in M$ be hyperbolic singularities and $\gamma \subset W^u(p) \, \cap \, W^s(q)$ an orbit of tangency between their unstable and stable manifolds. Generically, $\gamma$ is \textit{quasi-transverse}\footnote{See \cite{palis1983stability}, Chapter 2a or \cite{newhouse1973bifurcations}, Pages 306-307 about quasi-transverse submanifolds.}. In this case, $\text{dim} \, T_r \,W^u(p) + \text{dim} \, T_r \, W^s(q) = \text{dim} \, M - 1$ for $r \in \gamma$. Additionally, in local coordinates $(x_1,..., x_n)$ in a neighborhood of a point $r\in \gamma$,
\vspace{-0.5em}
\begin{align*}
    X_{\overline{\eta}} &= \frac{\partial}{\partial x_1} \, ,\\
    W^{u}(p) &= (x_1,...,x_{u}, 0,...,0) \, ,\text{ and}\\
    W^{s}(q) &= (x_1,...,x_k,0,...,0,x_{u_+1},...,x_{n-1}, f(x_2,...,x_k)) \, ,
\end{align*}
where $u = \text{dim}\,W^u(p)$, $k = \text{dim} \, (T_r W^u(p) \, \cap \, T_r W^s(q))$ and $f$ is a Morse function that controls the behavior of the intersection (\cite{dias1989bifurcations}, Section 1b). The Stable Manifold Theorem for a Morse function implies that $W^u(p)$ and $W^s(q)$ vary smoothly with respect to $\eta$ nearby $\overline{\eta}$. Consequently, for parameter values $\eta$ close to $\overline{\eta}$, $f$ can also be written as a smoothly varying function $f_\eta$ dependent on $\eta$ (\cite{palis1983stability}, Chapter 1, Section 2a). The orbit $\gamma$ unfolds generically if $\frac{\partial f_{\eta}}{\partial \eta}(r) |_{\eta = \overline{\eta}} \neq 0$.

A \textit{heteroclinic connection} between two critical elements $p, q$ is an orbit that lies on the stable manifold $W^s(q)$ and the unstable manifold $W^u(p)$. A heteroclinic flip occurs when these connections change, making it a global bifurcation. If a Morse-Smale gradient system undergoes this bifurcation, the unstable manifold $W^u(s)$ of a source saddle $s \in M$ changes between separate attractors by passing through a scenario where it intersects the stable manifold $W^s(\Tilde{s})$ of another saddle $\Tilde{s} \in M$. This intermediate state generically corresponds to a quasi-transverse orbit of tangency (\cite{palis1983stability}, Chapter 2). 

\begin{figure}[t!]
% \vspace{-1.5em}
        \centering
        \includegraphics[width=1\textwidth]{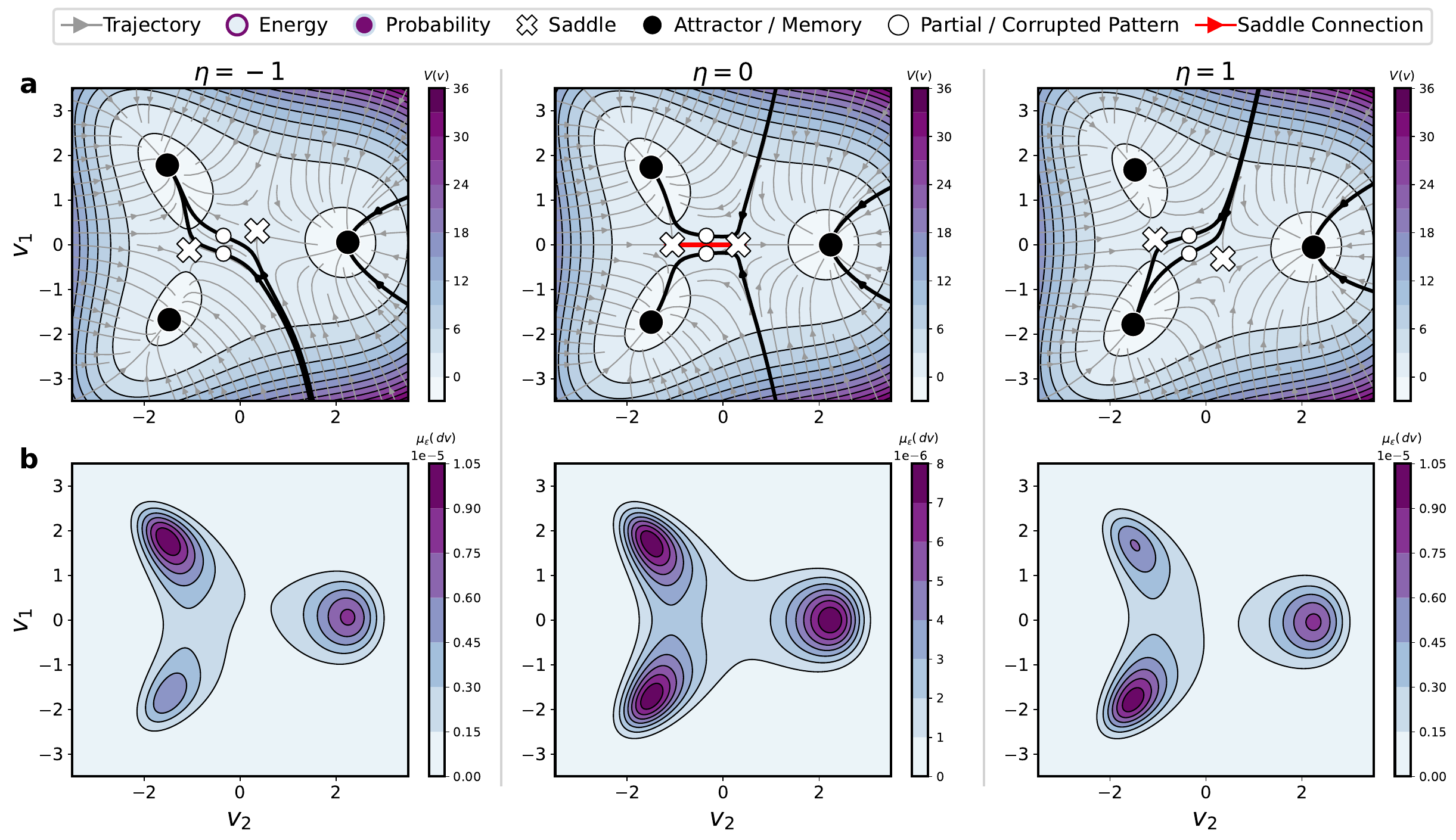}
        \caption{\textbf{Heteroclinic flip bifurcation.} {\textbf{(a)} Trajectories (grey) of the one-parameter family of gradients in \Cref{example:heteroclinic-flip} (left to right). As the parameter value $\eta$ changes from $\eta=-1$ to $\eta=0$ (left to middle), the unstable manifold of a saddle (right-hand side) shifts from intersecting the stable manifold of the top attractor to intersecting the stable manifold of the other saddle, creating a saddle-saddle connection (red line). From $\eta=0$ to $\eta=1$, the saddle connection is destroyed and the unstable manifold of the right-hand side saddle intersects the stable manifold of the opposite attractor (bottom). The arc of gradient fields from $\eta=-1$ to $\eta=1$ encounters a heteroclinic flip bifurcation. \textbf{(b)} Invariant measures $\mu^{\epsilon}(dv)$ of small random perturbations $\mathcal{X}^{\epsilon}$ with $\epsilon=1$.}}
        \label{fig:figure-intuition-heteroclinic-flip-bifurcation}
% \vspace{-1.5em}
\end{figure}

\begin{example}[Heteroclinic flip bifurcation, \colorAutoref{fig:figure-intuition-heteroclinic-flip-bifurcation}]
\label{example:heteroclinic-flip}
Let $\{ X_{\eta} \}_{\eta\in \mathbb{R}}$ be a one-parameter family of gradients where $X_{\eta} := -\nabla_{g_{\eta}} V_{\eta}$ is smoothly varying with respect to $\eta$. The heteroclinic flip geometry described by \cite{raju2024geometrical} is adapted with perturbations to both the metric and potential. Consider the one-parameter family of metrics,
\begin{gather*}
    g_{\eta_{ij}} = \delta_{ij} - 6 \cdot (1 - \delta_{ij}) \, \eta \cdot G(v_1; -1, 1) \cdot G(v_2; 0, 2) \\
    \mathrm{with } \qquad G(v_i; \mu, \sigma) = \frac{1}{\sqrt{2 \pi \sigma^2}} \exp\left(-\frac{(v_i - \mu)^2}{2 \sigma^2}\right) \, ,
\end{gather*}
where the function $G(v_i; \mu, \sigma)$ is Gaussian with mean $\mu$ and standard deviation $\sigma$. The one-parameter family of potential functions $\{V_{\eta}\}_{\eta \in [-1,1]}$ is given by,
\begin{gather*}
    \label{equation:example-heteroclinic-flip-associative-memory}
    V(v, \eta) = \dfrac{1}{5} \left( \frac{1}{2} v_{1}^{4} + \dfrac{1}{4} v_{2}^{4} - v_{2}^{3} + 2 v_{2}v_{1}^{2} - 2 v_{2}^{2} + \frac{3}{2}v_{2} + \sum_{i=3}^{n} v_{i}^{2} + \frac{1}{2}\eta v_{1} \right) \, .
\end{gather*}
From $\eta=-1$ to $\eta=1$, an orbit of tangency (a heteroclinic connection) between the saddle, $s$, on the left, and the saddle, $\Tilde{s}$, on the right, is born, altering the intersection of the unstable manifold $W^u(\Tilde{s})$ with the stable manifold $W^s(\beta_a)$ of the top left attractor $\beta_a$ to the stable manifold $W^s(s)$. From $\eta=0$ to $\eta=1$, the unstable manifold $W^u(\Tilde{s})$ "flips" and intersects the stable manifold $W^s(\beta_b)$ of the opposite attractor $\beta_b$.
\end{example}

% --- Stable families of gradient vector fields (2-parameters) --- 

\subsubsection{Two-parameter families of gradients}
\label{subsubsection:two-parameter-families-gradients}
Let $\{ X_{\eta} \}$ be a two-parameter family on $M$, with $\eta = (\eta_1,\eta_2) \in N$ and $N$ a compact surface. Analogous to one-parameter families, the main theorem of \cite{dias1989bifurcations} shows that two-parameter families of gradients are generically stable (on an open dense subset of $\chi_2^g(M)$)\footnote{One would hope that arbitrary $k-$parameter families are stable. This is not the case (\cite{takens1985moduli}).}. In addition to the (two) bifurcations for one-parameter families, the bifurcations of two-parameter families consist of their combinations, along with additional codimension two bifurcations.  There are eleven possibilities. 

Additional generic conditions are imposed to derive the codimension two bifurcations. The main results are recapitulated below -- refer to \cite{dias1989bifurcations} for a complete description and a rigorous treatment of the assumptions made. The first five bifurcations are derived from the codimension one bifurcations. Another three are due to the breakdown of transversality conditions on invariant manifolds that are generic in the one-parameter case. Finally, another case occurs as a codimension two singularity, and two remaining bifurcations are due to codimension two tangencies.

% Cusp bifurcation %
\paragraph{On codimension one bifurcations.}
Five possibilities at a bifurcation value $\overline{\eta}$ are derived from the bifurcations of one-parameter families; two are the one-parameter bifurcations; an additional three are simultaneous occurrences of two saddle-nodes, two quasi-transversal orbits of tangency-, or a quasi-transversal orbit of tangency and a saddle-node.

% Criticality %

\paragraph{Criticality.}
Generically, the multiplicity of the eigenvalues of $X_{\eta}$ at a hyperbolic singularity $p \in M$ are equal to one, so a smallest contracting and expanding direction can be identified. The \textit{strong unstable manifold}, $W^{uu}(p)$, and \textit{strong stable manifold}, $W^{ss}(p)$, correspond to all positive (resp. negative) eigenvalues except for the smallest. Generically, the strong stable manifolds (resp. unstable) and strong unstable (resp. stable) manifolds of distinct hyperbolic critical elements are transverse, or \textit{noncritical}. Two additional bifurcations are due to criticality, i.e., the nontransverse intersection of these invariant manifolds. 

In particular, a sixth possibility is first described by one quasi-transversal orbit of tangency between the unstable and stable manifolds of two hyperbolic critical elements $W^u(p)$ and $W^s(q)$. In addition, there is a singularity $s$ with $W^u(s)$ nontransverse to $W^{ss}(p)$ or $W^{uu}(q)$ along a unique quasi-transversal orbit of tangency.

A seventh possibility consists of a saddle-node whose strong stable (resp. unstable) manifold is nontransverse to the unstable manifold of another singularity. The nontransverse orbit in this scenario is quasi-transversal.

% Center manifold tangency %
\paragraph{Quasi-transverse stable and center-unstable manifolds.}
The eigenvalues of $X_{\overline{\eta}}$ at hyperbolic singularities generically have a multiplicity of one, so there is also a \textit{center-stable manifold} $W^{cu}$ defined by the smallest positive eigenvalue and all negative ones (similarly, stable $W^{cs}$). Generically, $W^{cu}$ is transverse to $W^{uu}$. An additional bifurcation occurs due to a quasi-transversal orbit of tangency $\gamma \subset W^u(p) \, \cap \, W^s(q)$ with $p,q \in M$ hyperbolic singularities, along which the center-unstable manifold $W^{cu}(p)$ is not transversal to $W^s(q)$ (it can be assumed quasi-transversal).

% Cusp bifurcation %

\paragraph{Cusp bifurcation.}
Let $\overline{\eta}$ be a bifurcation value for $\{X_{\eta}\}$ and $s \in M$ a nonhyperbolic singularity. A ninth possibility encompasses a codimension two singularity with a cusp geometry. As with the saddle-node, $DX_s$ generically has a zero eigenvalue with multiplicity one. On a center manifold $W^c(s)$ passing through $s$, the cusp has the normal form ${X_{\overline{\eta}}(x) = \left( x^3 + \mathcal{O}\left( |x|^4 \right) \right)\frac{\partial}{\partial x}}$. One has a standard picture for the cusp as two collapsing saddle-nodes\footnote{See \cite{guckenheimer2013nonlinear}, Chapter 7 and \cite{dias1989bifurcations}, \S7 for more information on the cusp bifurcation.}.

% Cubic contact %

\paragraph{Orbit of tangency with cubic contact.}
Generically, the Morse function $f$ in \Cref{subsubsection:one-parameter-families-gradients} describing the interaction between invariant manifolds for a quasi-transverse orbit of tangency is quadratic; the contact is parabolic (\cite{palis1983stability}). A tenth bifurcation is due to a unique orbit of tangency between the unstable manifold $W^u(p)$ and the stable manifold $W^s(p)$ of two hyperbolic critical elements $p,q \in M$ having cubic contact.

% codimension 2 tangency %

\paragraph{Codimension two orbit of tangency.}
Finally, an eleventh possibility occurs due to a genuine codimension two tangency. That is, there is an open and dense subset of $\chi_2^g(M)$ such that if an element $X_{\overline{\eta}}$ has an orbit of tangency $\gamma \subset W^u(p) \, \cap \, W^s(q)$ with 
\begin{equation*}
    \text{dim} \left[T_r \,W^u(p) +  T_r \, W^s(q) \right] = \text{dim} \, M - 2
\end{equation*}
for $r \in \gamma$, then 
\begin{equation*}
    \text{dim}\, W^u(p) + \text{dim}\, W^s(q)  = \text{dim} \, M - 1 \, .
\end{equation*}
The tangency results from a lack of dimensions of the stable and unstable manifolds: ${u+s = \text{dim} \, M-1}$, where $s = \text{dim} \, W^s(q)$ and $u = \text{dim} \, W^u(p)$ (see \cite{dias1989bifurcations}, Proposition 5).

%%%%%%%%%%%%%%%%%%%%%%%%%%%%%%%%%%%%%%%%%%%%%%%%
% Applications and examples %
%%%%%%%%%%%%%%%%%%%%%%%%%%%%%%%%%%%%%%%%%%%%%%%%
\section{Applications and examples}
\label{section:applications-examples}
Examples to verify the broad applicability of the theory described in this paper are now considered. Energy-based models are rather general and are first considered in \Cref{subsection:applications-energy-based-models}, followed by the Hopfield network and Boltzmann machine in \Cref{subsection:applications-hopfield-boltzmann-models}. Modern Hopfield networks are described in \Cref{subsection:applications-modern-hopfield-networks}. Finally, diffusion models are described in \Cref{subsection:applications-diffusion-models-modern-hopfield}. Providing conditions and methods to ensure structural stability of these models is not within the scope of this document. However, the Hopfield and modern Hopfield networks admit descriptions.

\subsection{Assumptions and generic conditions}
\label{section:assumptions}
One assumption and two generic conditions are imposed. The assumption is that flows remain in a closed manifold or globally attracting region diffeomorphic to a closed disc in $\mathbb{R}^n$ with $C^2$ boundary. The first condition requires that the boundary intersects the flows transversally, ruling out bifurcations that occur at the boundary. The second condition is used to obtain structural stability criteria for classic and modern Hopfield networks via small modifications to their energy functions in Morse charts, making the dynamics \textit{gradient-like}.

\subsubsection{Generic condition 1: Transversal intersection with boundary}
\label{subsubsection:assumption-boundary}
A compact subset of $\mathbb{R}^n$ with boundary transverse to a gradient field can always be constructed if there is a globally attracting region containing its critical points. Denote by $D_{\delta}= \{\,x\in \mathbb{R}^n \, : \, ||x|| \leq \delta \,\}$ the closed disc of radius $\delta > 0$ in $\mathbb{R}^n$. with boundary $\partial D_{\delta} = \{\,x \in \mathbb{R}^n \, : \, ||x|| = \delta \,\}$ a smooth $(n-1)$-dimensional sphere $S_{\delta}^{n-1} \subset \mathbb{R}^n$.
\begin{proposition}
\label{proposition:globally-attracting-region-existence}
     Let $X$ be a gradient field on $\mathbb{R}^n$ and $D_C \subset \mathbb{R}^n$ a globally attracting region with all critical points of $X$ in the interior of $D_C$. Then there exists a region $D_{C+\delta}$ containing $D_C$ with boundary $\partial D_{C+\delta}$ transverse to $X$.
\end{proposition}

The proof in \Cref{appendix:generic-conditions} is essentially follows from Thom's transversality theorem (\cite{palis2012geometric}, Theorem 3.4); specifically, that transversality is a dense property. That it is open also holds but is not necessary. \Cref{proposition:globally-attracting-region-existence} assumes that a disc with positive radius contains all critical points of $X$. This is assumed for energy-based models and diffusion models, as they do not have explicit energy functions.

% --- Trajectory spaces of negative gradient flows --- 

\subsubsection{Generic condition 2: Gradient-like dynamics}
\label{subsubsection:trajectory-spaces-negative-gradients}
A modern view of the Morse-Smale conditions is through trajectory spaces of negative gradient flows, which are central to Morse and Floer homology and provide a neat description of these conditions for a dense set of metrics given a Morse function. The following is adapted from \cite{schwarz1993morse}, Chapter 2, which we refer to for rigorous descriptions; see, e.g., \cite{milnor1974characteristic}, Chapters 2-3 about vector bundles.

Let ${\gamma_x(t) = \phi^X_t(x)}$ be a gradient flow map for a Morse function $V: M \rightarrow \mathbb{R}$. We consider the space of (compact) curves with $\lim_{t \rightarrow \pm \infty} \gamma(t) = x \in M$ and respectively, $y \in M$. Denote by $\mathcal{P}_{x,y}$ the set of curves that start at $x \in M$ and end at $y \in M$. That is,
\begin{equation*}
    \mathcal{P}_{x,y} = \left\{ \gamma : \mathbb{R} \cup \{\pm \infty\} \rightarrow M \, : \,  \lim_{t \rightarrow -\infty} \gamma(t) = x \, \text{ and } \, \lim_{t \rightarrow +\infty} \gamma(t) = y \right\} \,.
\end{equation*}

Let $\mathcal{P}^{1,2}_{x,y} \subset \mathcal{P}_{x,y}$ be those curves which are square integrable with weak first derivative and by convention, denote the extended real numbers by $\overline{\mathbb{R}} = \mathbb{R} \cup \{\pm \infty\}$. Then, given a function $V \in C^{\infty}(M,\mathbb{R})$ and critical points $x, y \in \text{Crit}(V)$ as endpoints, the gradient field induces a smooth section $F$ in the $L^2$-Banach bundle,
\begin{equation*}
    L^2_{\mathbb{R}} \left(\mathcal{P}^{1,2*}_{x,y} TM\right) = \bigcup_{s \in \mathcal{P}^{1,2}_{x,y}} L^2_{\mathbb{R}}(s^{*}TM) \, ,
\end{equation*}
where $L^2_{\mathbb{R}}$ is a contravariant functor associating to sections of a smooth vector bundle $\xi$ on $\overline{\mathbb{R}}$ a vector space of sections of $\xi$ along with a Banach space topology from $L^2(\mathbb{R}, \mathbb{R}^n)$\footnote{The functor $L^2_{\mathbb{R}}$ is described in detail in \cite{schwarz1993morse}, Appendix A.}. The section $F$ is given by,
\begin{align*}
    F: & \,\mathcal{P}^{1,2}_{x,y} \rightarrow L^2_{\mathbb{R}} \left(\mathcal{P}^{1,2*}_{x,y} TM\right) \, ,\qquad  s \mapsto \frac{d}{dt} \, s + \nabla V \, \circ \, s \, ,
\end{align*}
which can be understood as a map that sends a curve $s$ from $x$ to $y$ to the vector field $\dot{s} + \nabla_{s(t)} V$ over the image of $s$. That is, it is a section of the induced bundle $s^* TM$. The zeros of the section $F$ are smooth curves that solve the differential equation ${\frac{d}{dt} \, s = -\nabla V \, \circ \, s}$ and satisfy the condition that $\lim_{t\rightarrow + \infty}s(t) = y$ and $\lim_{t\rightarrow - \infty}s(t) = x$; see \cite{schwarz1993morse}, Proposition 2.8 and 2.9. To summarize, $F$ is a section of a fiber bundle with base space $\mathcal{P}^{1,2}_{x,y}$ and with fibers over a curve $s \in \mathcal{P}^{1,2}_{x,y}$ being $C^{0}$ sections of $s^{*}TM$.

\begin{definition}[Gradient-like field]
    Let $V \in C^2(M ,\mathbb{R})$ be a Morse function on a Riemannian $n$-manifold $(M,g)$. A vector field $X$ is a \textit{gradient-like field adapted to $V$} if,
    \begin{enumerate}[itemsep=0pt]
        \item $DV_pX(p) \leq 0$ throughout the complement $M \setminus \text{Crit}(V)$ with equality only at the set of critical elements of $V$ on $M$;
        \item For each critical point $\beta \in \text{Crit}(V)$ with index $\lambda$, there is a smooth coordinate chart around $\beta$ so that in local coordinates, $\nabla V = - \left(\sum_{i=1}^{\lambda} x_i \, \partial/\partial x_i \right) + \left( \sum_{j={\lambda+1}}^{n} x_j \, \partial/\partial x_j\right)$.
    \end{enumerate}
\end{definition}
If (2) holds for a critical point $\beta_i$ the vector field $X$ is said to be in \textit{standard form} near $\beta_i$. If $X$ is in standard form near all critical points, the metric is called \textit{nice}, or \textit{compatible with the Morse charts of $V$}. It is shown below that there is always a gradient-like field for a Morse function. The proof is in \Cref{appendix:Morse-theory} and is adapted from \cite{cohennotes}. Alternate proofs can be found in \cite{audin2014morse}, Chapter 2 or \cite{milnor2025lectures}, Lemma 3.2.
\begin{proposition}
\label{proposition:nice-metrics-dense}
     Let $(M,g_0)$ be a Riemannian manifold with metric $g_0$ and $V \in C^2(M, \mathbb{R})$ a Morse function. The set of nice metrics for the pair $(M,V)$ is dense in the $L^2$ space of Riemannian metrics on $M$.
\end{proposition}
\noindent
Define $\mathcal{F}_{x,y} \subset \mathcal{P}^{1,2}_{x,y}$ as a subset that consists of gradient flow lines. That is,
\begin{equation*}
    \mathcal{F}_{x,y} = \left\{ \gamma\in  \mathcal{P}^{1,2}_{x,y} \, : \,  \gamma'(t) = -\nabla_{\gamma(t)} V \right\} \,.
\end{equation*}
The difference between gradient and gradient-like fields is only at critical points. For gradient-like fields adapted to a Morse function, the Morse-Smale conditions are equivalent to the section $F$ being transverse to the zero section (in the $L^2$-bundle), i.e, to the surjectivity of the linearization of $F$ at $\mathcal{F}_{x,y}$; see \cite{audin2014morse}, Theorem 10.1.5\footnote{A version of the surjectivity of $F$ being equivalent to the Smale transversality condition using the Levi-Cevita connection for a given metric can be found in \cite{hutchings2002lecture}.}. Given local coordinates around the image of a curve $\gamma \in \mathcal{F}_{x,y}$, with endpoints $x$ and $y$,
\begin{equation*}
    F(\gamma)_i = \frac{d\gamma_i}{dt} \,  + \sum_{j} g^{ij}(\gamma(t)) \, \frac{\partial V}{\partial x_j} = 0 \,.
\end{equation*}
Linearizing the gradient flow equation gives, for $\varrho \in \mathcal{P}^{1,2}_{x,y}$,
\begin{equation}
\label{equation:Floer-operator}
    DF_A(\varrho)_i = \frac{d\varrho_i}{dt} \,  + \sum_{j,k} \frac{\partial g^{ij}}{\partial x_k} \varrho_k \, \frac{\partial V}{\partial x_j} + \sum_{j,k} g^{ij} \frac{\partial^2 V}{\partial x_j \partial x_k}\varrho_k = 0 \, ,
\end{equation}
where the notation $F_A$ indicates that the linearization $DF_A(\varrho)_i$ can be assigned a continuous family of endormorphisms of $\mathbb{R}^n$, $A \in C^0({\mathbb{R}},\text{End}(\mathbb{R}^n))$, that can be written as $n \times n$ matrices. For the local computation above,
\begin{equation*}
    A_{ik} = \sum_{j} \left( \frac{\partial g^{ij}}{\partial x_k} \, \frac{\partial V}{\partial x_j} + \sum_{j,k} g^{ij} \frac{\partial^2 V}{\partial x_j \partial x_k} \right) \, ,
\end{equation*}
and substituting into \eqref{equation:Floer-operator} for $DF_A$ gives $DF_A(\varrho)_i = \dot{\varrho}_i + \sum_k A_{ik} \varrho_k$. The asymptotic behavior of the family is given by the Hessian matrices at each critical point:
\begin{equation*}
    \lim_{t\rightarrow -\infty} A_{ik}(t) = \sum_{j} g^{ij}(x) \frac{\partial^2 V}{\partial x_j \partial x_k} (x) \hspace{1em} \text{and} \hspace{1em} \lim_{t\rightarrow +\infty} A_{ik}(t) = \sum_{j} g^{ij}(y) \frac{\partial^2 V}{\partial x_j \partial x_k} (y) \, ,
\end{equation*}
at which the transversality condition is equivalent to the Morse condition.

%--------------------------------------------------------
% Energy-based models
%--------------------------------------------------------

\subsection{Energy-based models}
\label{subsection:applications-energy-based-models}
Energy-based models are Boltzmann-Gibbs distributions, ${p_{\theta}(x)=({Z_{\theta}})^{-1} \cdot e^{{-V_{\theta}(x)}/{\epsilon^2}}}$ with $Z_{\theta} = \int e^{{-V_{\theta}(x)}/{\epsilon^2}} \, dx$, where $Z_{\theta}$ is the partition function, and $V_{\theta}$ is a (smooth) real-valued function parameterized by a neural network with parameters $\theta$ (\cite{pml2Book}, Chapter 24). The noise-scaling parameter $\epsilon$ is typically set to one, and this is implicitly assumed below.

A traditional approach to fitting $p_{\theta}(x)$ to a data distribution $p_{\text{data}}(x)$ is through maximum likelihood estimation by maximizing the expected log-likelihood over the data distribution: $\max_{\theta} \mathbb{E}_{x \sim p_{\text{data}}(x)}[\log p_{\theta}(x)]$. However, the likelihood $p_{\theta}(x)$ is intractable to compute, due to the partition function $Z_{\theta}$. Similarly, the gradient of the log-probability of the model,
\begin{equation*}
    \nabla_{\theta} \log p_{\theta} = -\nabla_{\theta} V_{\theta}(x) - \nabla_{\theta} \log Z_{\theta} \,
     ,
\end{equation*}
requires computing $Z_{\theta}$; however, Markov chain Monte Carlo (MCMC) can be used to obtain unbiased estimates thereof. Specifically, $\nabla_{\theta} \log Z_{\theta}$ can be rewritten as,
\begin{equation*}
    \nabla_{\theta} \log Z_{\theta} = \mathbb{E}_{x \sim p_{\theta}(x)}[-\nabla_{\theta}V_{\theta}(x)] \, .
\end{equation*}
Therefore, if an exact sample $\tilde{x} \sim p_{\theta}(x)$ can be drawn from the model, then an unbiased sample of the gradient $\nabla_{\theta} \log p_{\theta}$ can be obtained, allowing the use of gradient ascent to optimize the model parameters. Alternatives, such as contrastive divergence (\cite{hinton2002training}), have been developed to overcome the difficulties with training these models, but all are compatible with gradient-based optimization (see \cite{song2021train}).

\begin{figure}[t!]
% \vspace{-1.5em}
        \centering
        \includegraphics[width=1\textwidth]{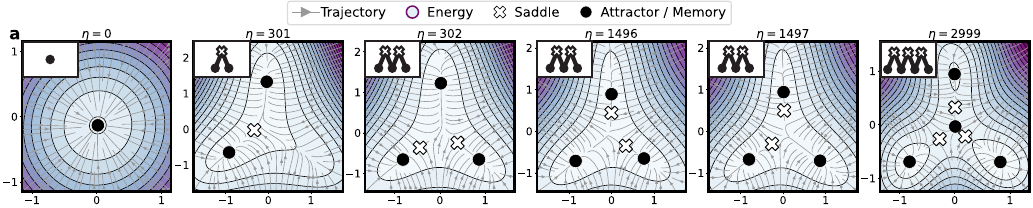}
        \caption{\textbf{Codimension one bifurcations during the learning dynamics of energy-based generative models.} {\textbf{(a)} Trajectories (grey) and critical points of a one-parameter family of gradients (left to right) derived from the potential of an energy-based model trained to generate four centroids in $\mathbb{R}^2$ using contrastive divergence. The model was pretrained to generate a centroid at the origin. The energy was parameterized by a three-layer multilayer perceptron with the softplus activation and a hidden dimensionality of 128, and was regularized by adding a quadratic term $\frac{1}{2}(\max_x ||x||)^2$ over the training data to encourage a barrier on the max norm of generated data. As the optimization index increases from $\eta_0 = 0$ to $\eta_{\text{final}} = 2999$, a sequence bifurcations occur; shown are representative bifurcations during training. Additional bifurcations occur and are not shown. DAGs at parameter values are shown as insets. A supercritical saddle-node occurs from $\eta=0$ to $\eta=301$, creating two attractors. Another saddle-node creates a third attractor and index 1 saddle from $\eta=301$ to $\eta=302$. A heteroclinic flip appears to occur between $\eta=1496$ to $\eta=1497$, causing the unstable manifold of the top saddle to change from intersecting the stable manifold of the bottom left attractor to the bottom right attractor. A final supercritical saddle-node creates the fourth attractor.}
        }
        \label{fig:figure-energy-based-model-learning}
\vspace{-0.5em}
\end{figure}

A popular method to draw samples from an energy-based model is Langevin MCMC, which produces samples from probability distribution $p(x)$ using the score function $\nabla_x \log p(x)$ in an overdamped Langevin equation (\cite{parisi1981correlation}). The score for energy-based models is the negative gradient of the energy: $\nabla_x \log p_{\theta}(x) = -\nabla_x V_{\theta}(x)$, making Langevin MCMC a discretization of the stochastic differential equation,
\begin{equation*}
    d \tilde{x}^{\epsilon}_t = \underbrace{\nabla_x \log p(x^{\epsilon}_t)}_{- \nabla_x V_{\theta}(x)} \,dt  + \epsilon \, dw_t \, ,
\end{equation*}
where $\epsilon >0$ is re-introduced as a noise-scaling term. Given an energy-based model $p_{\theta}(x)$, a family $\{p^{\epsilon} _{\theta}(x)\}_{\epsilon >0}$ parameterized by $\epsilon>0$ and $\theta$ yields a family of invariant measures $\{\mu^{\epsilon}\}_{\epsilon > 0}$ as in \Cref{section:diffusion-zero-noise-limit}. Langevin MCMC similarly produces a family of small random perturbations $\{\mathcal{X}^{\epsilon}\}_{\epsilon > 0}$. As $\epsilon \rightarrow 0$, trajectories approach those of a gradient dynamical system $X = - \nabla_x V_{\theta}(x)$.

If the potential is smooth, gradient ascent yields one-parameter families $\{\mu^{\epsilon}_\eta\}_{\epsilon > 0, \eta \in \mathbb{R}}$ and family $\{\mathcal{X}^{\epsilon}_\eta\}_{\epsilon > 0, \eta \in \mathbb{R}}$. By \Cref{subsection:paths-large-deviations}, when $\epsilon \rightarrow 0$ the probability that a path generated by $\mathcal{X}^{\epsilon}_\eta$ deviates from the flow defined by the gradient drift field vanishes. From \Cref{subsubsection:stable-families-of-gradients}, the process by which energy-based models learn to fit $p_{\text{data}}$ can be generically characterized at vanishing noise levels by ordered sequences of saddle-node or heteroclinic flip bifurcations (\colorAutoref{fig:figure-energy-based-model-learning}).

%--------------------------------------------------------
% Continuous Hopfield networks and Boltzmann machines
%--------------------------------------------------------

\subsection{Hopfield networks and Boltzmann machines}
\label{subsection:applications-hopfield-boltzmann-models}
The classic examples of associative memory and energy-based models, the Hopfield network and the Boltzmann machine, are now detailed.

% --- Stability properties of Hopfield networks and Boltzmann machines --- 

\subsubsection{Hopfield networks}
\label{subsubsection:hopfield-networks}
Consider the continuous Hopfield network with $N$ neurons (\cite{hopfield1984neurons}). Let $f_i$ be a smooth and monotonically increasing activation function, $v_i = f_i(u_i)$ be the output of a feature neuron given the input $u_i$ from a hidden neuron, and $f_i^{-1}(v_i) = u_i$ the inverse. The common sigmoid or hyperbolic tangent activation functions are considered here. Given a synaptic weight matrix $W$ that is symmetric with zero diagonal, the network's energy function is,
\begin{equation}
\label{equation:Hopfield-energy}
    V(v) = -\frac{1}{2} \sum_{i,j} W_{ij} v_i v_j + \sum_{i} B_i v_{i} + \sum_{i} R^{-1}_{i} \int_{0}^{v_{i}} f_{i}^{-1} (v) \,dv \,.
\end{equation}    
The term $R^{-1}$ introduces delays in the output of feature neurons relative to hidden neurons. Here, $R^{-1}_i = 1/\rho_i + \sum_j \tau^{-1}_{ij}$, where $\tau_{ij}$ is a resistance between neurons $i,j$, and $\rho_i$ is an input resistance. The bias term $B$ is set to zero to simplify arithmetic; the following results hold through similar calculations otherwise. The dynamics of the hidden neurons are,
\begin{equation*}
 {\frac{d{u_i}}{dt}} = \sum_{j} W_{ij} v_j  - R^{-1}_{i} f_i^{-1}(v_i) \, ,
\end{equation*}
which is the $i^\text{th}$ component of the negative gradient of the potential $V(v)$ with respect to feature neurons in the standard metric $g_{ij} = \delta_{ij}$. The dynamics of feature neurons are actually also gradient but in a metric determined by the activation function. By definition, $u_i = f_i^{-1}(v_i)$, and applying the chain rule gives
\begin{equation*}
    {\frac{d{v_i}}{dt}} = f_i'\left(u_i \right)\cdot \frac{du_i}{dt} = -f_i'\left(u_i \right) \cdot \frac{\partial V}{\partial v_i} = -f_i'\left(f_i^{-1}(v_i) \right) \cdot \frac{\partial V}{\partial v_i} \,.
\end{equation*}    
\noindent
The term $f_i'(\cdot)$ is always positive since $f_i$ is strictly increasing, so the first term in the final equality forms a diagonal matrix with positive entries. Therefore, the dynamics are gradient in a metric with inverse $g^{ij} = f_i'\left(f_i^{-1}(v_i) \right) \delta^{ij}$. A standard gradient system, ${\dot{v_i} = -\partial V/\partial v_i}$, can be obtained by a smooth change of coordinates\footnote{Similar descriptions for a more general class of models can be found in \cite{cohen1983absolute}.} ${y^i(v) = [\,{1}/{f_i'\left(f_i^{-1}(v_i)\right)}}\,]^{1/2} \,$.

\paragraph{Transverse boundary.}
The range of the sigmoid and hyperbolic tangent functions are the open intervals $(0,1)$ and $(-1,1)$. Consequently, the phase space for feature neurons is diffeomorphic to the open $N$-dimensional disc, denoted $D^N$. Fixed points in feature states often lie in the corners of a hypercube, so we consider the dynamics on the closure $\overline{D^N}$, which is compact. The vector field is vacuously transversal to the boundary at $v \in \partial \overline{D^N}$. 

Similarly, the inverse of these activation functions approach $\pm \infty$ as their input approaches the infimum and supremum of the open intervals $(0,1)$ or $(-1,1)$. If $R^{-1} > 0$ and the entries of $W$ are bounded, then the hidden state vector field points inwards, as the term $R^{-1}f^{-1}(v_i) \rightarrow \pm\infty$ as $v_i$ approaches the infimum or supremum of its range. In this case, a closed region that contains all critical points of the system exists, which generically satisfies \Cref{proposition:globally-attracting-region-existence}. In contrast, if $R^{-1} = 0$, the hidden-state dynamics become linear. Structural stability is equivalent to hyperbolicity for linear vector fields (\cite{palis2012geometric}, Chapter 2.2). This is reflected in \Cref{proposition:structural-stability-Hopfield-newtork}.

\paragraph{Linearized gradient flow equations.}
The structural stability of the Hopfield network is now discussed. Let $\gamma$ be a gradient flow line whose endpoints are two critical points on hidden states. Consider $F_u(\gamma): \gamma \mapsto \frac{d}{dt} \, \gamma + \nabla V \, \circ \, \gamma$ where $\nabla V$ is understood as the gradient with respect to the feature neurons in the standard metric. Locally, around the image of $\gamma$, the linearization $DF_{A_{u}}$ of $F_u(\gamma)$ is computed using \eqref{equation:Floer-operator}\footnote{The notation $F_u(\gamma)$ and $DF_{A_{u}}$ are used to clarify that hidden state dynamics are considered.}: 
\begin{align*}
    DF_{A_{u}}(\gamma)_i &= \frac{d\gamma_i}{dt} \,  + \sum_{j,k} \frac{\partial g^{ij}}{\partial v_k} \gamma_k \, \frac{\partial V}{\partial v_j} + \sum_{j,k} g^{ij} \frac{\partial^2 V}{\partial v_j \partial v_k}\gamma_k \nonumber\\
    &= - \frac{\partial V}{\partial v_i} + \sum_{j,k}\delta^{ij}\frac{\partial^2 V}{\partial v_j \partial v_k}f^{-1}_k(v_k) \nonumber\\
    &= \sum_{j} W_{ij} v_j  - R^{-1}_{i} f_i^{-1}(v_i) + \sum_k \left( - W_{ik} + R^{-1}_{i } \delta_{ik}\frac{1}{f_i'(f_i^{-1}(v_i))} \right) f^{-1}_k(v_k) \, ,
\end{align*}
where $\gamma_i(t) = u_i$ and $f_k^{-1}(v_k) = \gamma_k(t)$. The second equality is obtained by substituting $g^{ij} = \delta^{ij}$ and $\partial g^{ij} / \partial v_k = 0$ together with the equality $d\gamma_i /dt = -\partial V/\partial v_i$. The third equality comes from substituting the equality $d{u_i}/dt = -\partial V/ \partial v_i$ and computing $\partial^2V/\partial v_j \partial v_k$ using the equality of $[f_i^{-1}]'(v_i)$ with $1/f_i'(f^{-1}_i(v_i))$.

Let $\alpha$ be a gradient flow line whose endpoints are two critical points on feature neuron states. Consider $F_v(\alpha): \alpha \mapsto \frac{d}{dt} \, \alpha + \nabla_g V \, \circ \, \alpha$ where $\nabla_g V$ is the Riemannian gradient in the metric with inverse $g^{ij} = f'(f^{-1}(v_i))\delta^{ij}$. The metric is not constant, so the term with partial derivatives of its inverse does not vanish in the local formula for the linearization of $F_v(\alpha)$. Since $g^{ij}$ is diagonal, this term is only nonzero when $i=j=k$. Consequently, at the image of $\alpha$, where $\alpha_i(t) = v_i$, the linearization of $F_v(\alpha)$ is,
\begin{align*}
    DF_{A_{v}}(\alpha)_i &= \frac{d\alpha_i}{dt} \,  + \frac{\partial g^{ii}}{\partial v_i}\alpha_i \frac{\partial V}{\partial v_i} + \sum_{j,k} g^{ij} \frac{\partial^2 V}{\partial v_j \partial v_k}\alpha_k \nonumber\\
    &=  f_i'(f_i^{-1}(v_i)))\left(\sum_{j} W_{ij} v_j  - R^{-1}_{i} f_i^{-1}(v_i)\right) \\
    &+  \frac{f''(f^{-1}(v_i))}{f'(f^{-1}(v_i))} v_i \left(\sum_{j} -W_{ij} v_j  + R^{-1}_{i} f_i^{-1}(v_i)\right)\\
    &+ f_i'(f_i^{-1}(v_i))\sum_k \left( - W_{ik} + R^{-1}_{i } \delta_{ik}\frac{1}{f_i'(f_i^{-1}(v_i))} \right) v_k \,.
\end{align*}
The first term in the second equality is again obtained from $d\alpha /dt$ being a gradient flow line. The second term follows from an application of the chain rule on the inverse metric entries followed by the equality of $[f_i^{-1}]'(v_i)$ with $1/f_i'(f^{-1}_i(v_i))$. The final equation is expanded by substituting the connection matrix and resistance terms as above. 

\paragraph{Morse condition.}
Recall that $\lim_{t\rightarrow + \infty}\gamma(t) = u^{*}_x$ and $\lim_{t\rightarrow - \infty}\gamma(t) = u^{*}_y$ for critical points $u^{*}_{x,y}$ on hidden states and $\gamma(t)$ a gradient flow line between them. Similarly, for a curve $\alpha(t)$ on feature neurons, we have $\lim_{t\rightarrow + \infty}\gamma(t) = v^{*}_a$ and $\lim_{t\rightarrow - \infty}\gamma(t) = v^{*}_b$ for critical points $v^{*}_{a,b}$. The asymptotic behavior of $DF_{A_{u}}$ and $DF_{A_{v}}$ gives the Hessian matrices at these critical points. On hidden neurons,
\begin{align*}
\label{equation:hopfield-hessian-hidden-states}
    \lim_{t\rightarrow -\infty} A_{u;ik}(t) = \sum_{j} \delta^{ij} \frac{\partial^2 V}{\partial v_j \partial v_k} (u^*_x) = -W_{ik} + R^{-1}_{i } \delta_{ik} \frac{1}{f_i'(u^*_{x;i})} \, ,
\end{align*}
where $u^{*}_{x;i}$ denotes the $i^\text{th}$ component of the critical point $u^{*}_{x}$. A similar expression is obtained for $\lim_{t\rightarrow +\infty} A_{v;ik}(t)$ by replacing $u^{*}_{x;i}$ with $u^{*}_{y;i}$. On feature neurons,
\begin{align*}
    \lim_{t\rightarrow -\infty} A_{v;ik}(t) = \sum_{j} g^{ij}(v^*_a) \frac{\partial^2 V}{\partial v_j \partial v_k} (v^*_a) =f_i'\left(f_i^{-1}(v^*_{a;i}) \right)\left( - W_{ik} + R^{-1}_{i } \delta_{ik} \frac{1}{f_i'(v^*_{a;i})}\right) \, .
\end{align*}
A similar expression for $\lim_{t\rightarrow +\infty} A_{v;ik}(t)$ follows from substituting $v^{*}_{a;i}$ with $v^{*}_{b;i}$.

The Morse condition is satisfied if these Hessian matrices at each critical point are full rank. This is easily checked when $R^{-1} \rightarrow 0$: if the resistance term $R^{-1}$ is zero and the weight matrix $W$ has a zero eigenvalue, then the network is not structurally stable. This regime is common in applications that lack physical or biological motivation.

\begin{proposition}
\label{proposition:structural-stability-Hopfield-newtork}
    The continuous Hopfield network with $R_i^{-1} = 0$ is not structurally stable if $\lambda_i = 0$ for any $i=1,...N$, where $\lambda_i$ are eigenvalues of the weight matrix $W$.
\end{proposition}

It follows that the number of linearly independent patterns must be at least equal to the number of neurons for a Hopfield network with $R^{-1}=0$ to be structurally stable.

\begin{remark}
    Demanding that $W$ be full rank is easy to satisfy. Suppose that a network has a connection matrix $W$ with eigenvalues $\lambda_1,...,\lambda_N$ with $\lambda_i = 0$ for some $i = 1,...,N$. The eigendecomposition of $W$ is $W = Q \Lambda Q^T$ where $\Lambda$ is a diagonal matrix with $\Lambda_{ii} = \lambda_i$ for $i = 1,..., N$. For any $\epsilon > 0$, $\Lambda$ and $W$ can be modified by setting $\overline{\Lambda}_{ii} = \max(\lambda_i,\epsilon)$ and $\overline{W} = Q \overline{\Sigma} Q^T$. The matrix $\overline{W}$ is clearly full rank and it can be confirmed that it is symmetric (i.e., $\overline{W}^T = \overline{W}$) from the property that $(Q^T)^T = Q$ and $\overline{\Lambda}^T = \overline{\Lambda}$. Since $W$ and $\overline{W}$ are symmetric square matrices, their eigenvalues coincide with their singular values. The Frobenius norm between $\overline{W}$ and $W$ is then 
    \begin{equation*}
    || W - \overline{W}|| = \sqrt{\sum_{i=1}(\lambda_i - \overline{\lambda}_i)^2} = \sqrt{\sum_{j : \lambda_j < \epsilon}(\lambda_j-\epsilon)^2} \, ,
    \end{equation*}
    where the last equality sums over eigenvalues less than $\epsilon$ and can be arbitrarily small. 
\end{remark}

\paragraph{Smale condition.}
Let $\{ u^*_1,...u^*_n\}$ and $\{ v^*_1,...v^*_m\}$ be sets of hyperbolic singularities in hidden and feature states, respectively. On hidden neurons, choose non-overlapping Morse charts $(U_1, h_1),..., (U_n, h_n)$ in the neighborhoods of $u^*_1,...u^*_n$. Let $g_{E_1},...,g_{E_n}$ be the standard Euclidean metric with respect to the coordinates in each chart, and let $g_{E}$ be the standard metric. Let $B(u^*_1,\delta)\subset U_1$ be a coordinate ball of radius $\delta >0$ and $B(u^*_1,\epsilon)\subset B(u^*_1,\delta)$ with $\epsilon < \delta$. Choose a $C^{\infty}$ bump function $\Lambda: M \rightarrow [0,1]$ with $\Lambda(x) = 1$ for $x \in  B(u^*_1,\delta)$ and $\Lambda(x)=0$ for $x$ in a neighborhood of the complement $M \setminus B(u^*_1,\epsilon)$. The metric
\begin{equation*}
    g_u = g_E(y)(1-\Lambda(y)) + g_{E_{1}}(y)\Lambda(y)
\end{equation*}
is in standard form near $u^*_1$ on the Morse chart $(U_1, h_1)$ and can be extended globally as in the proof of \Cref{proposition:nice-metrics-dense}. Proceed iteratively on the remaining Morse charts so that $g_u$ is compatible with $(U_1, h_1),..., (U_n, h_n)$. The same procedure applied to suitable Morse charts for $\{ v^*_1,...v^*_m\}$ gives a gradient-like field with metric $g_v$ for feature neurons. In this case, replace the standard metric $g_E$ in the equation above with $g_0$, the metric on feature neurons determined by the activation function $f$, for $y$ in a Morse chart of $v^*_i$. The structural stability of these gradient-like systems, which have the same fixed points as the original network and differ by an arbitrarily small amount, are equivalent to the surjectivity of the linearizations $DF_{A_u}$ and $DF_{A_v}$.

% --- Boltzmann machine --- 

\subsubsection{Boltzmann machines}
The Boltzmann machine is an energy-based model. Its energy function is exactly the Hopfield network energy, and the probability of a network being in a configuration is given by $P(v) = {Z}^{-1}\cdot e^{{-{V(v)}/{T}}}$ with $Z = \int e^{{-{V(v)}/{T}}}dv$,
where $T > 0$ is a noise scaling, or "temperature", parameter. The classic formulation considers binary states $v_i = \{0,1\}$ with an update rule that describes the probability of a neuron firing (\cite{ackley1985learning}). Following \Cref{subsection:applications-energy-based-models}, a continuous-state formulation can be defined using an overdamped Langevin equation: $d{u}^{\epsilon}_t = - \nabla V(u^{\epsilon}_t)\,dt + \sqrt{T}dw_t \,.$

It is known that, as $T \rightarrow 0$, the Boltzmann machine's update rule is equivalent to the update rule of the Hopfield network. When the Hopfield network is structurally stable (\Cref{subsubsection:hopfield-networks}), the results of this paper apply in the vanishing noise limit. 

% --- Learning dynamics of Hopfield networks and Boltzmann machines--- 

\subsubsection{Learning dynamics of Hopfield networks and Boltzmann machines}
\label{subsubsection:learning-dynamics-hopfield}
Consider a Hopfield network with $N$ neurons, and let $\bm{\xi}=\{\zeta^1,...,\zeta^M\}$, $m\in \{1,...,M\}$, with each $\zeta^m \in \mathbb{R}^N$, be a set of $M$ patterns. The connection matrix $W$ of the Hopfield network is often trained to store $\bm{\xi}$ using a Hebbian learning rule (i.e., the outer product rule): $ W = \frac{1}{M} \sum_{m=1}^{M} \zeta^m (\zeta^m)^T$. This rule was originally biologically motivated, but can be derived as an unconstrained quadratic program. 

\begin{figure}[t!]
        \centering
        \includegraphics[width=1\textwidth]{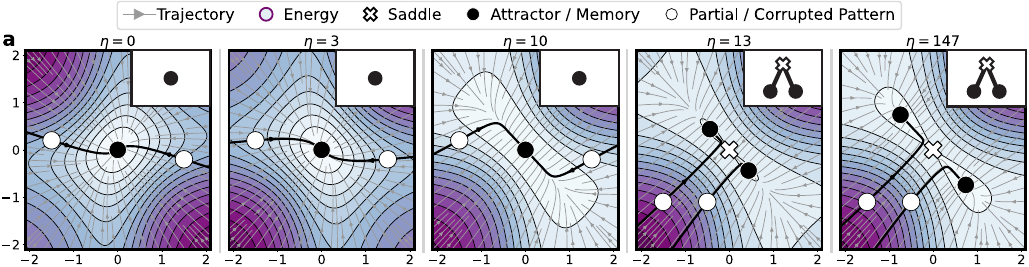}
        \caption{\textbf{Memory formation through Hebbian learning is generically described by codimension one bifurcations in Hopfield networks.} {\textbf{(a)} Hidden neuron trajectories (grey) and critical points of a one-parameter family of gradients (left to right, indexed by $\eta$) produced by training a 2-neuron Hopfield network using \Cref{algorithm:hebbian-learning-pgd} to store two memories, $(-1,1)$ and $(1,-1)$, in $\mathbb{R}^2$. Trajectories and critical points are overlayed on the energy surface given by \eqref{equation:Hopfield-energy} using $R^{-1}=0.85$. The learning rate and convergence tolerance of \Cref{algorithm:hebbian-learning-pgd} were set to $0.1$ and $1e^{-12}$, respectively. As the optimization index $\eta$ increases from the initial state $\eta_0 = 0$ to the final state $\eta_{\text{final}}=147$, a supercritical saddle-node bifurcation occurs, creating the stored memories. This bifurcation is reflected in the DAGs (insets) at each parameter value.}
        }
        \label{fig:figure-hebbian-bifurcations-hopfield}
\vspace{-0.5em}
\end{figure}

Consider minimizing the objective $\mathcal{L}(W) = - \frac{1}{2}\sum_{m=1}^{M} (\zeta^m)^T W \zeta^m$. The function $\mathcal{L}$ is convex, and the (negative) gradient $\partial L / \partial W = \sum_{m=1}^{M} \zeta^m(\zeta^m)^T$ is the Hebbian rule up to rescaling. There is no constraint to bound $\mathcal{L}(W)$ below, but it is natural to constrain the Frobenius norm $||W|| \leq c$ for $c>0$, resulting in the convex problem,
\begin{align}
\label{equation:hebbian-pgd-problem}
\min_W \quad & \mathcal{L}(W) = - \frac{1}{2}\sum_{m=1}^{M} (\zeta^m)^T W \zeta^m \hspace{1em}
\textrm{s.t.} \hspace{1em} ||W|| \leq c \,.
\end{align}
The Frobenius norm constraint is natural in that the unique solution to \eqref{equation:hebbian-pgd-problem} is scale-invariant. For any $c>0$, the solution is $cW^{*}$, where $W^{*}$ is the solution with unit norm. Similarly, the energy $V(v) = -\frac{c}{2} \sum_{i,j} W^*_{ij} v_i v_j$ is a scalar multiple of the energy corresponding to the weight matrix with unit norm, when $R^{-1}=0$. \Cref{algorithm:hebbian-learning-pgd} describes a projected gradient descent approach to train a Hopfield network in this manner.

Even if the bias and resistance terms, $B$ and $R^{-1}$, are nonzero, Hebbian learning produces a one-parameter family of matrices $\{W\}_{\eta \in \mathbb{R}}$ and gradients $\{X_{\eta}\}_{\eta \in \mathbb{R}}$. Generically, if $\overline{\eta}$ is a bifurcation value, $X_{\overline{\eta}}$ has one nonhyperbolic singularity corresponding to a saddle-node or a heteroclinic orbit of tangency. Consequently, memory formation is described by ordered sequences of these bifurcations. An example demonstrating a saddle-node bifurcation for a two-memory system trained using the Hebbian learning rule is shown in \colorAutoref{fig:figure-hebbian-bifurcations-hopfield}, and a higher-dimensional example is described in \colorAutoref{fig:supp-fig-algorithm-1}. Gradient-based approaches to training Boltzmann machines are standard, and this characterization applies to them in the zero-noise limit, as with energy-based models.

%--------------------------------------------------------
% modern Hopfield networks
%--------------------------------------------------------

\subsection{Modern Hopfield networks}
\label{subsection:applications-modern-hopfield-networks}
Modern Hopfield networks were introduced in \cite{krotov2016dense}, where the energy of the classic Hopfield network with binary states was modified to $V = - \sum_{m=1}^{M} F(v^T \xi^m)$, where $F(x) = x^n$ is an $n^{\text{th}}$ order polynomial. The Hopfield model corresponds to $n=2$. These models were shown to achieve a maximum storage capacity proportional to $N^{n-1}$ where $N$ is the number of neurons. \cite{demircigil2017model} further extended the model with polynomial interactions by setting $F(x) = \exp(x)$, which yields exponential storage in the number of neurons ($2^{N/2}$).

The model with exponential interactions was further extended in \cite{ramsauer2020hopfield}. Let $\bm{\xi} = (\xi^1,..., \xi^m)$ be a matrix containing a set of $m$ patterns, where $\xi^m \in \mathbb{R}^d$, and let $C = \max_m || \xi^m ||$. The energy of this network is
\begin{gather*}
    V = \underbrace{-\text{lse}(\beta,\bm{\xi}^Tv)}_{V_1} + \underbrace{\frac{1}{2}v^Tv + \frac{1}{\beta} \log M+ \frac{1}{2} C^2}_{V_2} \\
    \mathrm{where } \qquad \text{lse}(\beta,\bm{\xi}) = \frac{1}{\beta}\log \left( \sum_{m=1}^M \exp (\beta \xi^m)\right) \, .
\end{gather*}
We consider the update rule equivalent to the attention mechanism (\cite{ramsauer2020hopfield}, Section 2),
\begin{equation*}
    v_{t+1} = \bm{\xi}\text{softmax}\left(  \beta \bm{\xi}^T v_t\right) = - \nabla \left( -\text{lse}(\beta,\bm{\xi}^Tv) \right) = -\nabla (V_1) \, ,
\end{equation*}
obtained from a convex-concave procedure, where the energy is a sum of $V_1$ and $V_2$ (\cite{yuille2001concave} and \cite{ramsauer2020hopfield}, Appendix A.1.4).

\paragraph{Transverse boundary.} The update rule $v_{t+1}$ converges to a globally attracting region diffeomorphic to the closed $d$-dimensional disc: $D_C = \{\,x\in \mathbb{R}^d \, : \, ||x|| \leq C \,\}$\footnote{See \cite{ramsauer2020hopfield}, Appendix A.1.4 on global convergence of the update rule $v_{t+1}$.}. From \Cref{proposition:globally-attracting-region-existence}, a compact subset $\tilde{D}_{C+\delta} \subset \mathbb{R}^d$ with boundary $\partial \tilde{D}_{C+\delta}$ transverse to the gradient flow for this update rule can always be constructed from a dense set of arbitrarily small modifications to $\partial D_{C+\delta}$ for $\delta > 0$. 

\paragraph{Morse condition.}
The Jacobian matrix $D(v_{t+1})$ of the update rule is
\begin{align*}
    D(v_{t+1}) &= \beta \bm{\xi}\left(\text{diag}(\bm{p})- \bm{p}\bm{p}^T \right)\bm{\xi}^T=\bm{\xi} J_s \bm{\xi}^T \hspace{1em} \text{with} \hspace{1em} \bm{p} = \text{softmax}(\beta\bm{\xi^T v}) \, ,
\end{align*}
where $J_s$ the Jacobian of the softmax (\cite{ramsauer2020hopfield}, Appendix A.1.5). The rank of $D(v_{t+1})$ is bounded above by $\text{rank}(\bm{\xi})$ and $\text{rank}(J_s)$. Consequently, if the number of linearly independent patterns, say $M$, is less than the dimensionality, $d$, then singularities are degenerate and the network is not structurally stable. The codomain of the softmax is the probability simplex, which requires only $M-1$ parameters to describe. This implies that the nullity of $J_s$ is at least $1$, and by the rank-nullity law, $\text{rank}(J_s)\leq M-1$ when $J_s: \mathbb{R}^M \rightarrow \mathbb{R}^M$ is viewed as a linear map. Suppose that $M=d$ and the patterns are linearly independent so that $\bm{\xi}$ is full rank. Then $\text{rank}(J_s) = d-1$, which implies that $\text{rank}(D(v_{t+1}))\leq d-1$. Therefore, to ensure that all critical points are nondegenerate, it is necessary that there be at least $d+1$ patterns with $d$ of them linearly independent. This is actually sufficient to satisfy the Morse condition, since in this case, $\text{rank}(\bm{\xi} J_s \bm{\xi}^T)=d$.

\paragraph{Smale condition.}
Suppose that $V_1$ is Morse and $\{\beta_1,...,\beta_n\}$ is a set of hyperbolic singularities. The update rule $v_{t+1}$ can be approximated by a gradient-like field adapted to $V_1$ by choosing suitable Morse charts at each critical point and following the proof of \Cref{proposition:nice-metrics-dense}, which is not recapitulated for brevity. The stable and unstable manifolds $W^s(\beta_j)$ and $W^u(\beta_i)$ are transverse when the section $F: \mathcal{P}_{\beta_{i},\beta_{j}}^{1,2} \rightarrow L^2(\mathcal{P}_{\beta_{i},\beta_{j}}^{1,2*}T\tilde{D}_{C+\delta})$ intersects transversally the zero section, where $\tilde{D}_{C+\delta}$ is the closed disc with boundary given by an inclusion map $\tilde{\Psi}_{C+\delta}$. In a coordinate patch at the image of a gradient flow line $\gamma \in \mathcal{F}^{1,2}_{\beta_{i}, \beta_{j}}$ outside adapted subsets of Morse charts for $\beta_{i}$ and $\beta_{j}$, the linearization $DF_A$ is 
\begin{align*}
    DF_A(\gamma)_i &= \frac{d\gamma_i}{dt} + \sum_{j,k} g^{ij} \frac{\partial^2 V_1}{\partial v_j \partial v_k} \gamma_k = \sum_j \xi_{ij} \text{softmax}\left(  \beta \bm{\xi}^T v_t\right)_j + \sum_{k} \left(\bm{\xi} J_s \bm{\xi}^T \right)_{ik} v_{t;k} \,,
\end{align*}
where
\begin{equation*}
    \text{softmax}\left(  \beta \bm{\xi}^T v_t\right)_j = \frac{e^{\left(  \beta \bm{\xi}^T v_t\right)_j}}{\sum_m \left(  \beta \bm{\xi}^T v_t\right)_m} \,
\end{equation*}
and $\gamma_k(t) = v_{t;k}$. The first term of $DF_A(\gamma)_i$ is obtained by substituting the update rule $v_{t+1}$ for $d\gamma_i/dt$, and the second is obtained using $D(v_{t+1})$ in place of $\partial^2V/\partial v_j \partial v_k$.

\begin{figure}[t!]
% \vspace{-1.5em}
        \centering
        \includegraphics[width=1\textwidth]{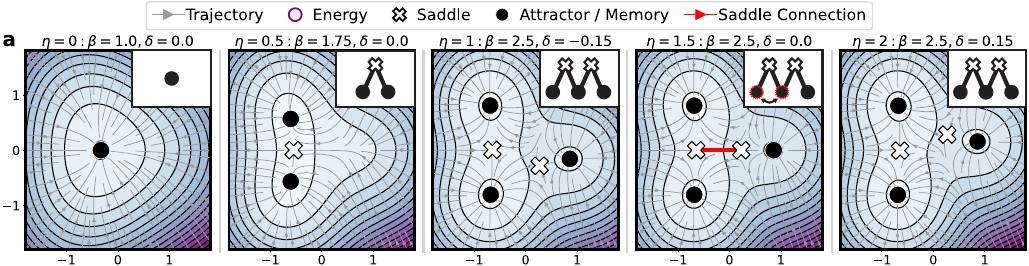}
        \caption{\textbf{Codimension one bifurcations of a modern Hopfield network.} {\textbf{(a)} Trajectories (grey) and critical points of a one-parameter family of networks (left to right) storing three patterns in $\mathbb{R}^2$: $(0.95,0 + \delta(\eta))$, $(-0.7,\sqrt{3}/2)$, and $(0.7,\sqrt{3}/2)$. The inverse temperature $\beta(\eta)$ and perturbation $\delta(\eta)$ are made to vary smoothly with $\eta \in [0,2]$ and constant for $\eta >2$ using a combination of Gaussian cumulative distribution (step) functions. As $\eta$ varies in $[0,2]$, a sequence of bifurcations occur, reflected in the DAGs at each parameter value (insets). Two supercritical saddle-nodes occur from $\eta=0$ to $\eta=1$ as $\beta$ increases, causing memories to form. From $\eta=1$ to $\eta=2$ a heteroclinic flip occurs by passing through an unstable intermediate state with a saddle-saddle connection along which the orbit is tangential (red line, DAG indicates node switching).}
        }
        \label{fig:figure-modern-hopfield-bifurcations}
\end{figure}

\paragraph{Learning dynamics.}
This model is typically an intermediate layer of a deep neural network, where the dynamics are defined in an embedding space. The dynamics can be considered on $\tilde{D}_{C'+\delta}$, where $D_{C'+\delta}$ is a closed disc with radius $C'$, where $C' = \max_{\eta} C$ is the uniform maximum norm of the patterns to store over the gradient descent index $\eta$, and $\tilde{D}_{C'+\delta}$ is constructed as in \Cref{proposition:globally-attracting-region-existence}. Understood as a dynamical system on the closed region $\tilde{D}_{C'+\delta}$, the learning process is generically characterized by the bifurcations of one-parameter families.

\paragraph{Metastability and temperature.} The "inverse temperature" parameter $\beta$ is analogous to a noise scaling parameter. When $\beta \rightarrow 0$ (high temperature), the softmax operation in the update $v_{t+1}$ is evenly distributed over the patterns, producing a single attractor. As $\beta$ increases, the softmax operation becomes sharper, mimicking a low temperature regime. An illustration of the two generic bifurcations that modern Hopfield networks can undergo as pattern positions and the inverse temperature smoothly vary with a parameter (e.g., optimization index or other) is given in \colorAutoref{fig:figure-modern-hopfield-bifurcations}.

%--------------------------------------------------------
% Diffusion with affine drift and modern Hopfield networks
%--------------------------------------------------------

\subsection{Denoising diffusion models}
\label{subsection:applications-diffusion-models-modern-hopfield}
Models in this section learn the reverse to diffusion processes that iteratively add noise to samples from a data distribution $p_{\text{data}}$ to some prior distribution $p_T$, for $T \in [0,\infty)$. We consider models whose forward noising processes are Ornstein–Ulhenbeck processes that are solutions to stochastic differential equations in $\mathbb{R}^n$ like
\begin{equation}
\label{equation:Ornstein-Uhlenbeck-forwards}
    d{x^{\epsilon}_t} = -\frac{1}{2}\beta_t x^{\epsilon}_t \,dt + \sqrt{\epsilon_t}\,dw_t \,  ,\hspace{1em} x_0 \sim p_{\text{data}} \, ,
\end{equation}
where $t \mapsto \beta_t$ is a positive weight function and $\epsilon_t$ is a time-dependent noise scale. Written in the form of \eqref{equation:stochastic-ode}, the diffusion coefficient is the $n \times n$ identity matrix. 

Denoising diffusion models (\cite{ho2020denoising,sohl2015deep}) are shown in \cite{song2020score} to be discretizations of \eqref{equation:Ornstein-Uhlenbeck-forwards}, where $\epsilon_t = \beta_t$, and are called variance-preserving. Song and colleagues also define sub-variance preserving models by defining the noise schedule
$\epsilon_t = \beta_t \, \left(1-e^{-2 \int_0^t \beta_s \, ds} \right)$. Likewise, they define variance-exploding models as SDEs with zero drift and a noise schedule $\epsilon_t = d[\sigma^2(t)]/dt$ where $\sigma^2(t)$ is obtained as the continuous limit of a Markov chain with a sequence $\{ \sigma_i \}_{i=1}^N$ of noise scales when $N \rightarrow \infty$ (\cite{song2020score}, Appendix B).

\paragraph{Noising dynamics.} An additional noise scaling parameter $\delta>0$ can be added to the right-hand sides of the stochastic equations. The drift term $\beta_t$ is affine (in fact, linear), and at a fixed time, it is constant, so its only fixed point is the origin. These forwards (noising) processes do not undergo nonlinear bifurcations in the $\delta \rightarrow 0$ limit, but changes in the sign of the eigenvalues of $D\beta_t(x_t)$ can alter the system's stable and unstable directions. The structural stability of these systems as linear vector fields $L: \mathbb{R}^n \rightarrow \mathbb{R}^n$ is simple: \textit{they are structurally stable if and only if they are hyperbolic}.

\paragraph{Generation dynamics.} Reverse processes to equations like \eqref{equation:stochastic-ode} are solutions to reverse stochastic differential equations of the form
\begin{equation}
\label{equation:reverse-diffusion}
    d{x_t^{\epsilon}} =  \left\{ b(x_t^{\epsilon}) - \nabla \cdot \left[ \epsilon^2 \sigma(x_t^{\epsilon})\sigma^*(x_t^{\epsilon}) \right] - \epsilon^2\sigma(x_t^{\epsilon})\sigma^*(x_t^{\epsilon}) \nabla \log p^{\epsilon}_t(x_t^{\epsilon})\right\}\,dt + \epsilon \sigma(x_t^{\epsilon}) \,d\overline{w}_t \, .
\end{equation}
The reverse SDE of \eqref{equation:Ornstein-Uhlenbeck-forwards} becomes
\begin{align}
\label{equation:reverse-diffusion-models-eqn}
    d{x_t^{\epsilon}} &=  \left\{ -\frac{1}{2}\beta_t x^{\epsilon}_t - \epsilon_t  \nabla \log p^{\epsilon}_t(x_t^{\epsilon})\right\}\,dt + \sqrt{\epsilon_t} \,d\overline{w}_t \nonumber \\
    &= - \nabla_x \left( \frac{1}{4} \beta_t||x^{\epsilon}_t||^2 + \epsilon_t \log p^{\epsilon}_t(x^{\epsilon}_t) \right) \, dt + \sqrt{\epsilon_t} \,d\overline{w}_t \, ,
\end{align}
after substituting terms into \eqref{equation:reverse-diffusion}. Evidently, the drift is the negative gradient of a time-varying potential, $V_t(x) = \frac{1}{4} \beta_t||x^{\epsilon}_t||^2 + \epsilon_t \log p^{\epsilon}_t(x^{\epsilon}_t) + c\, $, where $c$ is a constant. In the limit where $\epsilon_t \rightarrow 0$ uniformly, a time-dependent gradient system is obtained.

\begin{figure}[t!]
% \vspace{-1.5em}
        \centering
        \includegraphics[width=1\textwidth]{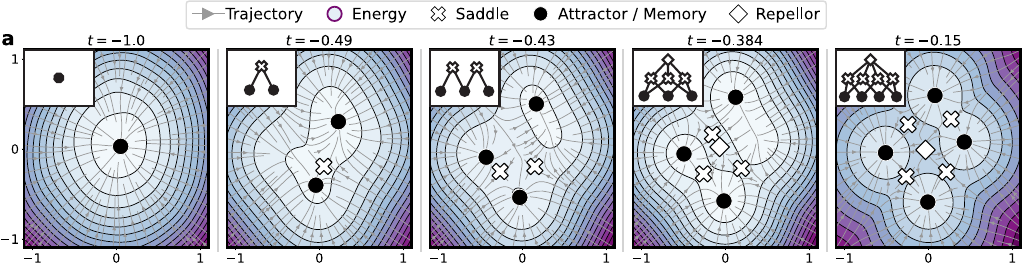}
        \caption{\textbf{Codimension one bifurcations characterize generation dynamics of denoising diffusion models.} {\textbf{(a)} Trajectories (grey) and critical points of a one-parameter family of gradients (left to right) derived from the probability flow ODE of a variance-preserving denoising diffusion model with linear noise schedule trained to generate four centroids in $\mathbb{R}^2$. The model was pretrained to generate a single centroid centered at the origin. As time runs backwards from $-1$ to $0$, a sequence supercritical saddle-node bifurcations occur, reflected in the DAGs at each parameter value (insets). Two supercritical saddle-nodes occur from $t=-1$ to $t=-0.43$. As $t$ continues to increase, another saddle-node occurs, creating a repellor and an index 1 saddle, visualized at $t = -0.384$, followed by an additional saddle node, visualized at $t=-0.15$.}
        }
        \label{fig:figure-diffusion-model-generation}
\vspace{-0.5em}
\end{figure}

It is more common to draw samples from $p_{\text{data}}$ using an ordinary differential equation derived from the forwards Kolmogorov equation, called the probability flow ODE (\cite{song2020score}, Appendix D). For \eqref{equation:stochastic-ode}, this deterministic ODE is
\begin{equation}
\label{equation:probability-flow-ode}
    d{x_t^{\epsilon}} =  \left\{ b(x_t^{\epsilon}) - \frac{1}{2}\nabla \cdot \left[ \epsilon^2 \sigma(x_t^{\epsilon})\sigma^*(x_t^{\epsilon}) \right] - \frac{1}{2}\epsilon^2\sigma(x_t^{\epsilon})\sigma^*(x_t^{\epsilon}) \nabla \log p^{\epsilon}_t(x_t^{\epsilon})\right\}\,dt \, ,
\end{equation}
whose solutions are distributed according to the induced marginal densities $p_t$ of \eqref{equation:stochastic-ode} for all $t$. The probability flow ODE for models with reverse SDEs like \eqref{equation:reverse-diffusion-models-eqn} is
\begin{align}
\label{equation:probability-flow-ode-VP-models}
    d{x_t^{\epsilon}} &=  \left\{ -\frac{1}{2}\beta_t x^{\epsilon}_t  - \frac{1}{2}\beta_t \nabla \log p^{\epsilon}_t(x_t^{\epsilon})\right\}\,dt  \nonumber \\
    &=-  \nabla_x \left( \frac{1}{4} \beta_t||x^{\epsilon}_t||^2 + \frac{1}{2} \epsilon_t \log p^{\epsilon}_t(x^{\epsilon}_t) \right) \, dt \, ,
\end{align}
which is the negative gradient of $V_t(x) =  \frac{1}{4} \beta_t ||x^{\epsilon}_t||^2 + \frac{1}{2} \epsilon_t \log p^{\epsilon}_t(x^{\epsilon}_t)  + c $, whose regularity is determined by the neural network parameterizing the score function, and which is evidently a model of associative memory. These equations describe variance and sub-variance preserving models, while variance-exploding models are defined by removing the quadratic norm-squared terms in \eqref{equation:probability-flow-ode-VP-models}, giving $V_t(x) =\frac{1}{2} \epsilon_t \log p^{\epsilon}_t(x^{\epsilon}_t)  + c $.

Data generation consists of solving the probability flow ODE backwards in time. When the time-varying score function is smooth, samples evolve via a smoothly-varying one-parameter family of gradients. Therefore, the generative process is generically characterized by ordered sequences of the bifurcations for one-parameter families. A denoising diffusion model trained using the score matching objective demonstrates this characterization of the generative process in \colorAutoref{fig:figure-diffusion-model-generation}.

\paragraph{Learning dynamics.}
"Learning" to generate is generically characterized by the bifurcations of two-parameter families. An illustration of the training and generation dynamics of the above model as a two-parameter family is given in \colorAutoref{fig:figure-training-generation-diffusion}.

%%%%%%%%%%%%%%%%%%%%%%%%%%%%%%%%%%%%%%%%%%%%%%%%
% Discussion and conclusion %
%%%%%%%%%%%%%%%%%%%%%%%%%%%%%%%%%%%%%%%%%%%%%%%%
\section{Conclusion and discussion}
\label{section:discussion-and-conclusion}
This paper describes a global geometric theory to characterize a transition from generation to memory at vanishing noise levels. We show that universal approximation, robustness to parameter and input perturbations, and the learning dynamics of associative memory models can be described generically in terms of Morse-Smale dynamical systems. In parallel, we extend this theory to generative diffusion models by analyzing their behavior in the zero-noise limit, thereby linking deterministic and stochastic formulations through a geometric perspective.

Reliable and stable associative memory satisfies the conditions for gradient fields to be Morse-Smale. Conversely, current models can be universally approximated by this class of systems. Memory retrieval is therefore generically downhill, and associated with each memory landscape is an invariant DAG that defines the connections between rest points (\colorAutoref{fig:figure-generic-properties}). The Morse lemma implies that these models generically satisfy the asymptotic stability property for memory storage. Moreover, they are structurally stable -- their fixed-point and orbit structure is unchanged in a neighborhood of a model in use. This takes the local form of stability global to describe the robustness of these systems to small model perturbations. Pragmatically, Morse-Smale associative memory models can be written as an inverse Riemannian metric times the gradient of a potential, which broadens their applicability to more geometrically structured domains. 

Generative diffusion appears as small random perturbations of associative memory, defined by stochastic differential equations with gradient drift (\colorAutoref{fig:figure-intuition-example-zero-noise}). Their invariant measures are Boltzmann-Gibbs distributions that, in the zero-noise limit, concentrate on the stable manifolds -- specifically, memories -- stored by the drift field. The Morse condition implies that these measures are stable to stochastic and deterministic perturbations on stable manifolds. The Morse-Smale assumptions extend these results globally; though, the image of a Boltzmann-Gibbs distribution under a homeomorphism between topologically equivalent systems may not preserve absolute continuity. Large deviation results from FW theory similarly imply that stochastic generative flows converge flows governing memory recall at vanishing noise levels. We introduced a definition for stochastic flows in the zero-noise limit analogous to the topological equivalence of deterministic flows to capture the stability of generation in this regime (\Cref{subsection:paths-large-deviations}). Together, these results show that the global fixed-point structure of these models generically persists under small stochastic or deterministic perturbations.

Memory formation occurs when the Morse-Smale assumptions are violated. The learning processes of associative memory models appear as one-parameter families of gradients tracing arcs in the space of gradient fields, and memory formation is generically described by only two bifurcations (\colorAutoref{fig:figure-intuition-saddle-node-bistable} and \colorAutoref{fig:figure-intuition-heteroclinic-flip-bifurcation}). Time-dependent models, described by two-parameter families, are characterized by eleven bifurcations (\Cref{subsubsection:two-parameter-families-gradients}). These results also apply to diffusion models that use deterministic probability flow ODEs for generation. Analogously, the definition introduced for equivalence of stochastic flows in the zero-noise limit implies that a broad class of models with time-varying drift generate data and learn through the bifurcations of one- and two-parameter families of gradients. These bifurcations correspond to topological changes in the system’s phase portrait that induce discrete changes in the memory DAG. Notably, memory formation can be understood as a sequence of graph edits: bifurcations create or eliminate rest points (nodes), or modify heteroclinic connections (edges) among nodes. These heteroclinic bifurcations are global and cannot be described by local bifurcation theory. The evolution of a DAG thus provides a combinatorial encoding of the global dynamics across parameter space and offers a natural language for describing the organization and restructuring of memory and generation landscapes.

Our results apply to classic models, including the Hopfield network and the Boltzmann machine, along with more general energy-based models and modern diffusion models (\Cref{section:applications-examples}). We give explicit computations that determine when Hopfield-type networks are not structurally stable (\Cref{subsection:applications-hopfield-boltzmann-models,subsection:applications-modern-hopfield-networks}). Although associative memory is generically structurally stable, even some simple scenarios violate the Morse-Smale conditions for these models, and the results of the modern Hopfield network are directly relevant to the attention mechanism. Nevertheless, it is not clear if the Morse-Smale constraints are computationally tractable to enforce, in general. It is also not clear if the connections made between generation and memory can be used to enhance model performance.

\subsection*{Acknowledgments}
J.H. and Q.M. thank Dmitry Krotov for helpful comments and feedback. J.H. is supported by a National Science Foundation Graduate Research Fellowship under grant no. 1746886.

\newpage
\appendix
\section*{Appendix}
\addtocontents{toc}{\protect\setcounter{tocdepth}{0}} % remove appendix from TOC

\section{Zero-noise limits and small random perturbations proofs}
\label{appendix:strict-isomorphism-weak-convergence}

Two lemmas from
\Cref{subsection:robustness-diffusion-models} are proven below.

\vspace{0.5em}
\noindent
{\bf Proof of \Cref{lemma:strict-isomorphism-preserve-weak-convergence}}. By definition of weak convergence, for any continuous bounded function $f:M \rightarrow \mathbb{R}$,
\begin{equation*}
    \int_M f(x) \,d\mu_{n}(x) \rightarrow \int_M f(x) \,d\mu(x) \,.
\end{equation*}
Set $\nu_n = h_{\#}\mu_n$ and recall that by assumption $\nu = h_{\#}\mu$. Consider a continuous bounded function $g : M \rightarrow \mathbb{R}$. Since $g$ is bounded and continuous, and $h:M \rightarrow M$ is a continuous bijection, the composition $g \circ h: M \rightarrow \mathbb{R}$ is bounded and continuous. Combined with the equalities $\nu_n = h_{\#}\mu_n$ and $\nu = h_{\#}\mu$,
\begin{equation*}
    \int_M g(y) \,d\nu_{n}(y) = \int_M ( g \circ h )(x) \,d\mu_{n}(x) \hspace{1em} \text{and} \hspace{1em} \int_M g(y) \,d\nu(y) = \int_M ( g \circ h )(x) \,d\mu(x) \, ,
\end{equation*}
from the definition of the pushforward measure. By assumption, $\mu_n \rightarrow \mu$ weakly. Therefore,
\begin{equation*}
    \int_M ( g \circ h )(x) \,d\mu_{n}(x) \rightarrow \int_M ( g \circ h )(x) \,d\mu(x) \, ,
\end{equation*}
which immediately implies that
\begin{equation*}
    \int_M g(y) \,d\nu_{n}(y) \rightarrow \int_M g(y) \,d\nu(y) \, .
\end{equation*}
\hfill\BlackBox

\vspace{0.5em}
\noindent
{\bf Proof of \Cref{lemma:metric-equivalence-Morse-Smale}}. That $h: (M, \mathcal{B}, \mu) \rightarrow (M, \mathcal{B}, h_{\#}\mu)$ is a measurable space isomorphism is immediate, since it is a bijective measurable map with measurable inverse. By definition of pushforward measure $h_{\#}\mu(h(A)) = \mu(h^{-1}(h(A)))= \mu(A)$ for $A \in \mathcal{B}$, so $h$ is a measure space isomorphism. It remains to show that $h_{\#}\mu$ is invariant under $\phi_t^Y$. That is, we need to show that $h_{\#}\mu(A)=h_{\#}\mu((\phi_t^Y)^{-1}(A))$. By definition,
\begin{equation}
\label{equation:pushforward-lemma-morse-smale}
    h_{\#}\mu((\phi_t^Y)^{-1}(A)) = \mu\left(h^{-1} ( (\phi_t^Y)^{-1}(A))\right)
\end{equation}
for any $t$. By definition of topological equivalence, there are $t_1,t_2$ so that $h(\phi^X_{t_{1}}(x)) = \phi^Y_{t_{2}}\left( h(x)\right)$ for all $x \in M$. By definition, 
\begin{equation*}
    (\phi_t^Y)^{-1}(A) = \{ y \in M \, : \, \phi_t^Y(y) \in A\} \, .
\end{equation*}
That is, $y \in (\phi_t^Y)^{-1}(A)$ if an only if $\phi_t^Y(y) \in A$. Since $y \in M$ and $h:M \rightarrow M$ is bijective, there exists $x \in M$ such that $y = h(x)$. Then $h(\phi^X_{t_{1}}(x)) = \phi^Y_{t_{2}}\left( h(x)\right) = \phi^Y_{t_{2}}\left( y\right)$. Then $y \in (\phi_t^Y)^{-1}(A)$ if and only if $h(\phi^X_{t_{1}}(x)) \in A$ -- that is, if and only if $\phi^X_{t_{1}}(x) \in h^{-1}(A)$. Applying again the preimage definition, we have that $y \in (\phi_{t_2}^Y)^{-1}(A)$ if and only if for $x$ with $y=h(x)$ we have that $x \in (\phi^X_{t_{1}})^{-1} (h^{-1}(A))$. By construction $x = h^{-1}(y)$, so 
\begin{equation*}
    h^{-1}((\phi_{t_2}^Y)^{-1}(x)) \in A \qquad \text{if and only if}\qquad (\phi^X_{t_{1}})^{-1} (h^{-1}(x)) \in A \,.
\end{equation*}
Therefore,
\begin{equation*}
    (\phi^X_{t_{1}})^{-1} (h^{-1}(A)) = h^{-1}((\phi_{t_2}^Y)^{-1}(A)) \,.
\end{equation*}
Substituting this into \eqref{equation:pushforward-lemma-morse-smale},
\begin{equation*}
    h_{\#}\mu((\phi_{t_2}^Y)^{-1}(A)) = \mu\left((\phi^X_{t_{1}})^{-1} (h^{-1}(A))\right) \,.
\end{equation*}
That $\mu$ is invariant under $\phi_t^X$ means that $\mu(A) = \mu((\phi_t^X)^{-1}(A))$ for $A \in \mathcal{B}$. Therefore,
\begin{equation*}
    h_{\#}\mu((\phi_{t_2}^Y)^{-1}(A)) = \mu\left( h^{-1}(A)\right) = h_{\#}\mu(A) \,.
\end{equation*}
\hfill\BlackBox

\vspace{0.5em}
\noindent
{\bf Proof of \Cref{proposition:pathwise-convergence-preservation}}.
    From the definition of convergence in probability, 
    \begin{equation*}
        \lim_{\epsilon \rightarrow 0}P(||x^{\epsilon}_t - x^*_t || \geq \delta) = 0
    \end{equation*}
    for $\delta >0$. We need to show that,
    \begin{equation*}
        \lim_{\epsilon \rightarrow 0}P(|| h(x^{\epsilon}_t) - y^*_t || \geq \delta') = 0
    \end{equation*}
    for $\delta' >0$. The Heine-Cantor theorem implies that $h$ is uniformly continuous since $M$ is compact. Uniform continuity implies that for any $\delta' > 0 $, there exists a $\delta > 0 $ such that for all $x,y \in M$:
    \begin{equation*}
        ||h(x) - h(y)|| < \delta' \qquad \text{with} \qquad ||x-y|| < \delta \,.
    \end{equation*} 
    Equivalently, for any $\delta'>0$ there exists a $\delta > 0$ so that $||x^{\epsilon}_t - x^*_t || < \delta$ implies that $|| h(x^{\epsilon}_t) - y^*_t || < \delta'$. The contrapositive of this implication implies a bound,
    \begin{equation*}
        P(|| h(x^{\epsilon}_t) - y^*_t || \geq \delta') \leq P(|| x^{\epsilon}_t - x^*_t || \geq \delta) \,.
    \end{equation*}     
    By convergence in probability of the right-hand side $P(|| x^{\epsilon}_t - x^*_t || \geq \delta \rightarrow 0$ when $\epsilon \rightarrow 0$.
\hfill\BlackBox

\vspace{0.5em}
\noindent
{\bf Proof of \Cref{proposition:topological-eq-implies-zero-noise-equivalent}}. 
    Since $\mathcal{X}^{\epsilon}$ are small random perturbations, paths of $x^{\epsilon}_{t_1}$ converge in probability to the trajectories of $\phi_{t_1}^X$; similarly for $\mathcal{Y}^{\epsilon}$. So, (1) is satisfied. We need to show (2), that $h(\Phi_{t_1}^{X, \epsilon}(x,\cdot)) \xrightarrow[]{P} \Phi_{t_2}^{Y, \epsilon}(h(x),\cdot)$. From the definition of topological equivalence, $h(\phi^X_{t_1}(x)) = \phi_{t_2}^Y(h(x))$ for $x\in M$. The left-hand side of this equality well-defined, so \Cref{proposition:pathwise-convergence-preservation} implies that $h(x_{t_1}^{\epsilon}) \xrightarrow[]{P} h(x^*_{t_1})$. Then, $h(\Phi^{X,\epsilon}_{t_1}(x,\cdot)) \xrightarrow[]{P} h(\phi^X_{t_1}(x))$. Moreover, by (1), $\Phi^{Y,\epsilon}_{t_2}(h(x),\cdot) \xrightarrow[]{P} \phi^Y_{t_2}(h(x))$. Substitute $h(\phi^X_{t_1}(x)) = \phi_{t_2}^Y(h(x))$ to obtain $h(\Phi^{X,\epsilon}_{t_1}(x,\cdot)) \xrightarrow[]{P} \phi^Y_{t_2}(h(x))$. So, for any $\delta >0$,
    \begin{align*}
        \lim_{\epsilon \to 0} P \left(  \| \Phi^{Y,\epsilon}_{t_2}(h(x),\cdot) - \phi_{t_2}^{Y, \epsilon}(h(x)) \| \geq \delta \right) &= 0, \hspace{1em} \text{and} \\
        \lim_{\epsilon \to 0} P \left(  \| h(\Phi_{t_1}^{X, \epsilon}(x,\cdot)) - \phi_{t_2}^{Y, \epsilon}(h(x)) \| \geq \delta \right) &= 0 \, .
    \end{align*}
    The norm in (2) satisfies the triangle inequality,
    \begin{equation*}
         \| h(\Phi_{t_1}^{X, \epsilon}(x,\cdot)) - \Phi_{t_2}^{Y, \epsilon}(h(x),\cdot) \| \leq  \| \Phi^{Y,\epsilon}_{t_2}(h(x),\cdot) - \phi_{t_2}^{Y, \epsilon}(h(x)) \| +  \| h(\Phi_{t_1}^{X, \epsilon}(x,\cdot)) - \phi_{t_2}^{Y, \epsilon}(h(x)) \| \,.
    \end{equation*}
    Consequently, the probability $P \left(\| h(\Phi_{t_1}^{X, \epsilon}(x,\cdot)) - \Phi_{t_2}^{Y, \epsilon}(h(x),\cdot) \| \geq \delta \right)$ obtained from the left-hand side is less than or equal to,
    \begin{align*}
        & \qquad P \left( \| \Phi^{Y,\epsilon}_{t_2}(h(x),\cdot) - \phi_{t_2}^{Y, \epsilon}(h(x)) \| +  \| h(\Phi_{t_1}^{X, \epsilon}(x,\cdot)) - \phi_{t_2}^{Y, \epsilon}(h(x)) \|\ \geq \delta \right) \\
        & \leq P \left( \| \Phi^{Y,\epsilon}_{t_2}(h(x),\cdot) - \phi_{t_2}^{Y, \epsilon}(h(x)) \| \geq \alpha \right) +  P\left(\| h(\Phi_{t_1}^{X, \epsilon}(x,\cdot)) - \phi_{t_2}^{Y, \epsilon}(h(x)) \|\ \geq \kappa \right) \, ,
    \end{align*}
    where the second inequality comes from the union bound (Boole's inequality) for some $\alpha,\kappa > 0$. Set $\alpha, \kappa = \delta/2$. Then the two terms in the last inequality are equal to $0$ when $\epsilon \rightarrow 0$. Then $\lim_{\epsilon \to 0}P \left(\| h(\Phi_{t_1}^{X, \epsilon}(x,\cdot)) - \Phi_{t_2}^{Y, \epsilon}(h(x),\cdot) \| \geq \delta \right) = 0$ as desired.
\hfill\BlackBox

\section{Proofs on Morse theory}
\label{appendix:Morse-theory}

A proof of the Stable Manifold theorem for a Morse function is now given. First, a couple of lemmas adapted from \cite{milnor2025lectures} are stated. 

\begin{lemma}
\label{lemma:good-critical-points-milnor}
    Let $K$ be a compact subset of an open set $U \subset \mathbb{R}^n$ and $1 \leq r \leq \infty$. If $f\in C^r(U,\mathbb{R})$ has only nondegenerate critical points in $K$, then there exists an $\eta >0$ such that if $g\in C^r(U,\mathbb{R})$ and $|| \, D^i(f-g)(x) \, || < \epsilon$ for all $x\in K$, then $g$ has only nondegenerate critical points in $K$.
\end{lemma}

\begin{lemma}
\label{lemma:density-critical-points-milnor}
    Let $d: K \rightarrow K'$ be a diffeomorphism of two compact subsets $K\subset U$ and $K'\subset U'$ for $U \subset \mathbb{R}^n$ and $U' \subset \mathbb{R}^n$. For any $\epsilon>0$, there exists a $\delta>0$ such that if a smooth map $f\in C^{\infty}(U', \mathbb{R})$ satisfies $|| \, D^i(f)(x') \, || < \delta$
    at all points $x'\in K'$, then $f \circ d$ satisfies $|| \, D^i(f\circ d)(x) \, || < \epsilon$
    at all points $x\in K$ for $i=0,1,2,...,r$.
\end{lemma}

\vspace{0.5em}
\noindent
{\bf Proof of \Cref{theorem:stable-manifold-theorem-morse}}.
  (1): Let $M^{a} = V^{-1}(-\infty,a] = \{q \in W^s(p) \, : \, V(q) \leq a\}$ and $a,b \in \mathbb{R}$ with $a < b$. Then $M^a$ is diffeomorphic to $M^b$ so long as $V^{-1}[a,b]$ is compact and contains no critical points\footnote{see, e.g., \cite{milnor2016morse}, Theorem 3.1 for an alternative proof which reparametrizes the gradient flow to be constant within $[a,b]$ and vanish elsewhere.}. Set $t_a(q)$ as the unique time where $V(\phi_{t_a(q)}^X(q)) = a$ for $q \in M^b$. Then the assignment $q \mapsto \phi_{t_a(q)}^X(q)$ is a diffeomorphism from $M^b$ to $M^a$ pushing $M^b$ to $M^a$ along the trajectories of the gradient flow.

    Let $B_{\delta} \subset M$ be a coordinate neighborhood of $p$ with radius $\delta > 0$. Let $E^s \oplus E^u$ be the invariant splitting of $T_pM$ into contracting and expanding directions defined by the orthonormal basis and adapted norm (above). We consider (1) for $W^s(p)$. The proof for $W^u(p)$ is analogous. Set $B^s_{\beta} = B_{\beta} \cap E^s$ and $B^u_{\beta} = B_{\beta} \cap E^u$. Choose $\beta$ sufficiently small so that the local stable manifold theorem holds. Choose $\beta >0$ so that for $x \in B^s_{\beta}$, $x_1 = \, \cdot \cdot \cdot \, = x_{\lambda} = 0$. Denote this subset of $B_{\beta}$ by $W_{B_{\beta}}^s$. Since this basis $\left( \frac{\partial}{\partial x_1}, ..., \frac{\partial}{\partial x_n} \right)$ for $T_pM$ is orthonormal, the tangent space $T_p \, W_{B_{\beta}}$ at $p$ is the positive eigenspace of the Hessian $\left( \partial ^2 V / \partial x_j \partial x_i \right)$ at $p$. 
    
    Let $\gamma_x(t) = \phi^X_t(x)$ be a smooth gradient flow map. Order the eigenvalues of the Hessian so the first $\lambda_p$ ones are negative $(\alpha_1,...\alpha_{\lambda_p})$ and $n-\lambda_p$ are positive $(\beta_{\lambda_{p+1}},...,\beta_{\lambda_{n}})$. In $W_{B_{\beta}}^s$,
    \begin{equation*}
        \gamma_i(t) =\left\{ \begin{array}{cc} 
                \gamma_i(0) \,e^{|\alpha_{\lambda_i} |\,t}, & \hspace{5mm} \lambda_i \leq \lambda_p \\
                \gamma_i(0)\,e^{-|\beta_{\lambda_{i}} |\,t}, & \hspace{5mm} \lambda_i > \lambda_p \\
                \end{array} \right. 
    \end{equation*}
    where $\gamma(0)=x$ for $x \in W_{B_{\beta}}^s$ and $\gamma_i$ denotes the $i$-th component of $\gamma$.
    
    This explicit formula for $\gamma_i(t)$ in $W_{B_{\beta}}^s \subset W^s(p)$ shows that it is a smoothly embedded submanifold of $M$ (in accordance with the local stable manifold theorem). Any such open set extends smoothly throughout $W^s(p)$. Let $M^a_s = W_{B_{\beta}}^s \, \cap \, M^a$ and let $W_t^s = \phi_{-t}^X (x)$ for $x \in W_{B_{\beta}}^s$. Then by the above, $\phi^X_t$ applied backwards in time is a diffeomorphism from $M^a_s = W_{B_{\beta}}^s$ to some $W_t^s = M^b_s$. When $t \rightarrow \infty$, the gradient flow traces out larger portions of $W^s(p)$. Since each element of $W^s(p)$ is eventually in $W_{B_{\beta}}^s$,
    \begin{equation*}
        W^s(p) = \bigcup_{b>V(p)} M^b_s = \bigcup_{t} W^s_t \,.
    \end{equation*}
    Since $W_{B_{\beta}}^s \subset W^s(p)$ is an embedded submanifold of the same dimension, we have that it is an open neighborhood of $p$ contained in $W^s(p)$ with tangent space $E^s$. That the tangent space of $W^s(p)$ is $E^s$ follows. That $W^s(p)$ is an embedded submanifold of $M$ with the same regularity as $X$ follows from the regularity of the local stable manifolds composed with $\phi_t^X$. The dimensionality similarly follows. This proves (1).

    (2): The proof follows \cite{milnor2025lectures}, Theorem 2.7 which shows that Morse functions are open and dense. Let $(U_1,h_1),...,(U_k,h_k)$ be a finite covering of $M$ (by compactness). Choose compact sets $K_i \subset U_i$ such that $\bigcup_i K_i$ covers $M$. By \Cref{lemma:good-critical-points-milnor}, there is a neighborhood $\mathcal{N}_i$ of $V \in C^{r+1}(M,\mathbb{R})$ whose elements have no degenerate critical points on $K_i$. Let $\mathcal{N} = \bigcap_i^m \mathcal{N}_i$. Then any $g \in \mathcal{N}$ has no degenerate critical points on $M = \bigcup_i^n K_i$. Therefore, Morse functions are open in $C^{r+1}(M,\mathbb{R})$.

    Now choose a neighborhood $\mathcal{N}$ of $V \in C^{r+1}(M,\mathbb{R})$ and a smooth bump function $\Lambda: M \rightarrow [\,0,1 \,]$ with $\Lambda(x) = 1$ for $x$ in a neighborhood of $K_1$ and $\Lambda(x) = 0$ for $x$ in a neighborhood of the complement $M \setminus U_1$. Then generically, through an application of Sard's theorem, for almost all linear maps $L: \mathbb{R}^n \rightarrow \mathbb{R}$, the map,
    \begin{equation*}
        V_1(x) = V(x) + \Lambda(x) L(h_i(x))
    \end{equation*}
    has nondegenerate critical points on $K_1 \subset U_1$; see \cite{milnor2025lectures}, Lemma A, Page 11.

    The difference between $V_1(x)$ is supported on the compact subset $K = \text{supp}\, \Lambda(x) \subset U_1$. $L(x) = \sum_i l_i x_i$ is linear, so this difference is written,
    \begin{equation*}
       id \circ V_1 \circ h^{-1}_1(x)  - id \circ V \circ h^{-1} (x)
    \end{equation*}
    for $x \in h_1(K)$. Setting the coefficients $l_i$ small enough, it can be ensured that this difference, along with the $r^{\text{th}}$ order derivatives are smaller than any $\epsilon > 0$ on $h_1(K)$. If $\epsilon > 0$ is small enough, $V_1 \in \mathcal{N}$ by \Cref{lemma:density-critical-points-milnor}. Therefore $V_1$ has no degenerate critical points on $K_1$, and by \Cref{lemma:good-critical-points-milnor}, we may choose a neighborhood $\mathcal{N}_1 \subset \mathcal{N}$ so that any $g \in \mathcal{N}_1$ has nondegenerate critical points in $K_1$. This process can be repeated to obtain a $V_2 \in \mathcal{N}_2 \subset \mathcal{N}_1$ that has no degenerate critical points on $K_2 \cup K_1$. Applying this process iteratively, $V_k \in \mathcal{N}_k \subset \mathcal{N}_{k-1} \subset... \subset \mathcal{N}_1 \subset \mathcal{N}$ which has no degenerate critical elements on $K_1 \, \cup ... \cup \, K_k$.
    
    Let $\mathcal{N}$ be an $\eta$-neighborhood of $V \in C^{r+1}(M)$. That is, for each $V'\in \mathcal{N}$ and $i=0,...,r+1$,    
    \begin{equation*}
        || D^i(V' \circ h^{-1})(x)  - D^i( V \circ h^{-1}) (x)|| < \eta \,.
    \end{equation*}
    From above it follows that $\eta$ can be made sufficiently small so that each $V' \in \mathcal{N}$ has nondegenerate critical points on $\bigcup_i^k K_k= M$. The gradient fields $X = -\nabla_gV$ and $X' = -\nabla_gV'$ in the (fixed) Riemannian metric $g$ also satisfy
    \begin{equation*}
        \left|\left| D^i \left(g^{ij}\frac{\partial V \circ h^{-1}}{\partial x^j}\right)(x)  -D^i \left(g^{ij}\frac{\partial V' \circ h^{-1}}{\partial x^j}\right)(x) \right|\right| < \delta
    \end{equation*}
    for $i=0,...,r$ and some $\delta >0$, since derivatives of the metric on compact sets are bounded and so are the partial derivatives $\partial V \circ h^{-1}/\partial x^j$ and $\partial V' \circ h^{-1}/\partial x^j$. This difference is clearly controlled by the difference in the partial derivatives of $V$ and $V'$; both are supported on compact sets. Then as $V' \rightarrow V$ in the $C^{r+1}$ sense, $X' \rightarrow X$ in the $C^r$ topology. By definition, the gradient flow $\phi_t^X$ for $X = -\nabla_gV$ is determined by the unique solutions to $\frac{d}{dt}\phi^X_t(p) = X(\phi^X_t(p))$ given initial conditions; similarly for $\phi_t^{X'}$. From (1), $W^s(p) = \bigcup_{b>V(p)} M^b_s = \bigcup_{t} W^s_t$. That is, $W^s(p)$ is an embedded submanifold, and the inclusion $W^s(p) \hookrightarrow M$ can be taken as gradient flow lines $\gamma_x(t)$ (from the definition of stable manifold). Since any $V' \in \mathcal{N}$ has no degenerate critical points, the stable manifolds for the corresponding gradient system $X'$ are well-defined. Since $V'$ can be made arbitrarily close to $V$ in $\mathcal{N}$, so are their flow lines, and $X'$ can be made so that it has exactly one critical point arbitrary close to that of $X$ with the same index. Denote by $p'$ such a critical point and $W^s(p')$ the corresponding stable manifold for the flow $\phi_t^{X'}$. Then the gradient flow lines of $\phi_t^{X'}$ can be made within $\epsilon$ in the $C^r$ topology to $\phi_t^{X}$ for any $\epsilon>0$. That is, $W^s(p') \rightarrow W^s(p)$ in the $C^r$ sense as $V' \rightarrow V$ in the compact-open $C^{r+1}$ topology. This shows (2). 
\hfill\BlackBox

\section{Applications}
\label{appendix:hopfield}

\subsection{Proofs on generic conditions}
\label{appendix:generic-conditions}
The following lemma is needed for a proof of \Cref{proposition:globally-attracting-region-existence}:
\begin{lemma}[\cite{palis2012geometric}, Chapter 1.3]
\label{lemma:vectors-transverse-palis}
    Let $S \subset \mathbb{R}^n$ be a submanifold and $f \in C^{\infty}(M, \mathbb{R}^n)$. The set of vectors $u \in \mathbb{R}^n$ for which $f + u$ is transversal to $S$ is residual.
\end{lemma}
\vspace{0.5em}
\noindent
{\bf Proof of \Cref{proposition:globally-attracting-region-existence}}. Denote by $\Psi_{\delta}: \partial D_{C+\delta} \hookrightarrow \mathbb{R}^n$ the inclusion of the boundary $\partial D_{C+\delta}$ in $\mathbb{R}^n$. We will show that, given a $\delta>0$, the map $\Psi_{\delta}$ can always be modified so that it intersects the gradient flow $\phi_t^X$ transversally. For any $\delta$, the set $D_{C+\delta} \setminus D_C$ contains no critical points. Therefore, defining $\Phi(x,t_x):= \gamma_x(t_x)$ for $x \in \partial D_{C+2\delta}$ gives a diffeomorphism $\Phi(x,t_x): \partial D_{C+2\delta} \rightarrow \partial D_C$\footnote{That $\Phi(x,t_x)$ is a diffeomorphism follows from, e.g., \cite{milnor2016morse}, Theorem 3.1.}. 

\begin{wrapfigure}{r}{0.42\textwidth}
    \vspace{1em}
    \centering
    \includegraphics[width=0.42\textwidth]{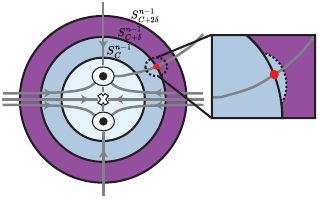}
    \caption{\textbf{Intuition for \Cref{proposition:globally-attracting-region-existence}}. {Critical points (filled black circles, white cross) lie in the interior of the disc $D_C$ (blue, magenta circles) with boundary $\partial D_C \cong S^{n-1}_{C}$ (black exterior). For a $\delta>0$, the disc $D_{C+2\delta}$ is diffeomorphic to $D_C$; flow lines (grey) intersect $D_{C+\delta}$. In a neighborhood of an intersection point (red circle), a bump function modifies $\partial D_{C+\delta}$} (inset) to intersect flow lines transversally.}
    \label{figure:transverse-boundary-construction}
    \vspace{0em}
\end{wrapfigure}  
Then for any $x \in \partial D_{C+2\delta}$, there is a time $t_x$ that the flow line $\gamma_{x}$ intersects $\partial D_{C+\delta}$. That is, $\gamma_{x}(t_x) \, \cap \, S^{n-1}_{C+\delta} = \{p \}$ for some $p \in S^{n-1}_{C+\delta}$. The sphere $S^{n-1}_{C+\delta}$ is compact, so choose a finite cover $(U_1,h_1), ... (U_n,h_n)$ and compact sets $K_i \subset U_i$ so that $K_1,...,K_n$ cover $S^{n-1}_{C+\delta}$. Choose a chart, say $(V_1, y_1)$ with $y_1 = id$ and $y_1: V_1 \rightarrow \mathbb{R}^n$ so that $\Psi_{\delta}(K_1) \subset V_1$. Choose a smooth bump function $\Lambda: M \rightarrow [0,1]$ with $\Lambda(x)=1$ for $x$ in a neighborhood of $K_1$ and $\Lambda(x)=0$ for $x$ in a neighborhood of the complement $S^{n-1}_{C+\delta} \setminus U_1$. Consider the map 
    \begin{equation*}
        \tilde{\Psi}_{\delta}(x) = \Lambda(x)\Psi^u_{\delta}(x) \, ,
    \end{equation*}
with $u \in \mathbb{R}^n$ a vector and $y_1 \circ \Psi^u_{\delta}(x) = y_1 \circ \Psi_{\delta}(x) + u = id \circ \Psi_{\delta}(x) + u $. By \Cref{lemma:vectors-transverse-palis}, there is a vector $u$ with arbitrarily small norm $|| u ||> 0 $ so that $y_1 \circ \Psi_{\delta}(x) + u$ is transversal to $y_1(\phi_t^X)  \subset \mathbb{R}^n$. Therefore $\tilde{\Psi}_{\delta}$ is transversal to $\phi_t^X$ on $K_1$. That $\tilde{\Psi}_{\delta}(x)$ remains one-to-one is clear. The difference between $\tilde{\Psi}_{\delta}(x)$ and $\Psi_{\delta}(x)$,
\begin{equation*}
     y_1 \circ \tilde{\Psi}_{\delta}\circ h_1^{-1}(x) - y_1 \circ \Psi_{\delta}\circ h_1^{-1}(x) = \Lambda \circ h^{-1}(x) + u \, ,
\end{equation*}
is supported on the compact set $K = \text{supp}\, \Lambda(x) \subset U_1$ for $x \in h_1(K)$. Given an initial neighborhood $\mathcal{N}_1$ of $\Psi_{\delta}$ in $C^{\infty}(S^{n-1}_{C+\delta},\mathbb{R}^n)$, it is easy to see that $|| u ||$ can be sufficiently small so that $\tilde{\Psi}_{\delta}(x) \in \mathcal{N}_1$. Repeating this process for a second compact set $K_2$ gives a map $\tilde{\Psi}_{\delta} \in \mathcal{N}_2$ transverse to $\phi_t^X$ on $K_2$. Since $\tilde{\Psi}_{\delta}$ is transverse to $\phi_t^X$ on $U_1$, $\tilde{\Psi}_{\delta} \in \mathcal{N}_1 \, \cap \, \mathcal{N}_2$ is transverse to $\phi_t^X$ on $K_1 \, \cup \, K_2$. Repeating this again produces a map $\tilde{\Psi}_{\delta} \in \mathcal{N}_1 \, \cap ... \cap \  \mathcal{N}_n$ which is transverse to $\phi_t^X$ on $K_1 \, \cup  ... \, \cup K_n$. \hfill\BlackBox

\vspace{0.5em}
\noindent
{\bf Proof of \Cref{proposition:nice-metrics-dense}}.
Since $M$ is compact, $\text{Crit}(V) = \{\beta_1,...,\beta_n\}$ is a finite number of critical elements. Let $(U_1, h_1),..., (U_n, h_n)$ be Morse charts in the neighborhoods of $\beta_1,...,\beta_n$ and $g_{E_1},...,g_{E_n}$ the standard euclidean metric with respect to the coordinates $(y^1,...,y^n)$ in each chart. Since these critical points are isolated, we may assume that $(U_1, h_1),..., (U_n, h_n)$ are non-overlapping. Let $B(\beta_1,\delta)\subset U_1$ be a coordinate ball of radius $\delta >0$ and $B(\beta_1,\epsilon)\subset B(\beta_1,\delta)$ with $\epsilon < \delta$ (which can be arbitrarily small). Choose a smooth ($C^{\infty}$) bump function $\Lambda: M \rightarrow [0,1]$ with $\Lambda(x) = 1$ for $x \in  B(\beta_1,\delta)$ and $\Lambda(x)=0$ for $x$ in a neighborhood of the complement $M \setminus B(\beta_1,\epsilon)$. Then the metric
\begin{equation*}
    g = g_0(y)(1-\Lambda(y)) + g_{E_{1}}(y)\Lambda(y)
\end{equation*}
is a metric on $U_1$, since the set of symmetric positive definite bilinear forms is a convex set and this is convex linear combination. Therefore, the metric $g$ is in standard form near $\beta_1$. By compactness, we may add additional charts $(U_j, h_j), 1\leq j \leq N$ with open images $\Omega_j$ under $h_j$ whose union with those of $(U_i,h_i)$ cover $M$. Then $g$ can be extended to all of $M$ by setting $g = g$ for $x$ in $U_1$ and $g = g_0$ otherwise. We now proceed in iterations. Apply the same procedure above on a Morse chart $(U_2, h_2)$ by choosing suitable coordinate balls $B(\beta_2, \epsilon) \subset B(\beta_2, \delta) \subset U_2$. Since $U_1$ and $U_2$ do not overlap, $g$ is automatically in standard form near $\beta_1$. After modifying it on $U_2$ with the bump function $\Lambda$, it is also in standard form near $\beta_2$. Applying this procedure to all remaining Morse charts produces a metric in standard form for all $\beta_1,...,\beta_n$. So, $g$ is compatible with the Morse charts of $V$.
\hfill\BlackBox

\subsection{Hopfield networks}
The following algorithm describes the projected gradient descent approach to training a Hopfield network using the Hebbian learning rule, followed by an application of it to a higher-dimensional system in \colorAutoref{fig:supp-fig-algorithm-1}.
\begin{figure}[t!]
        \centering
        \includegraphics[width=1\textwidth]{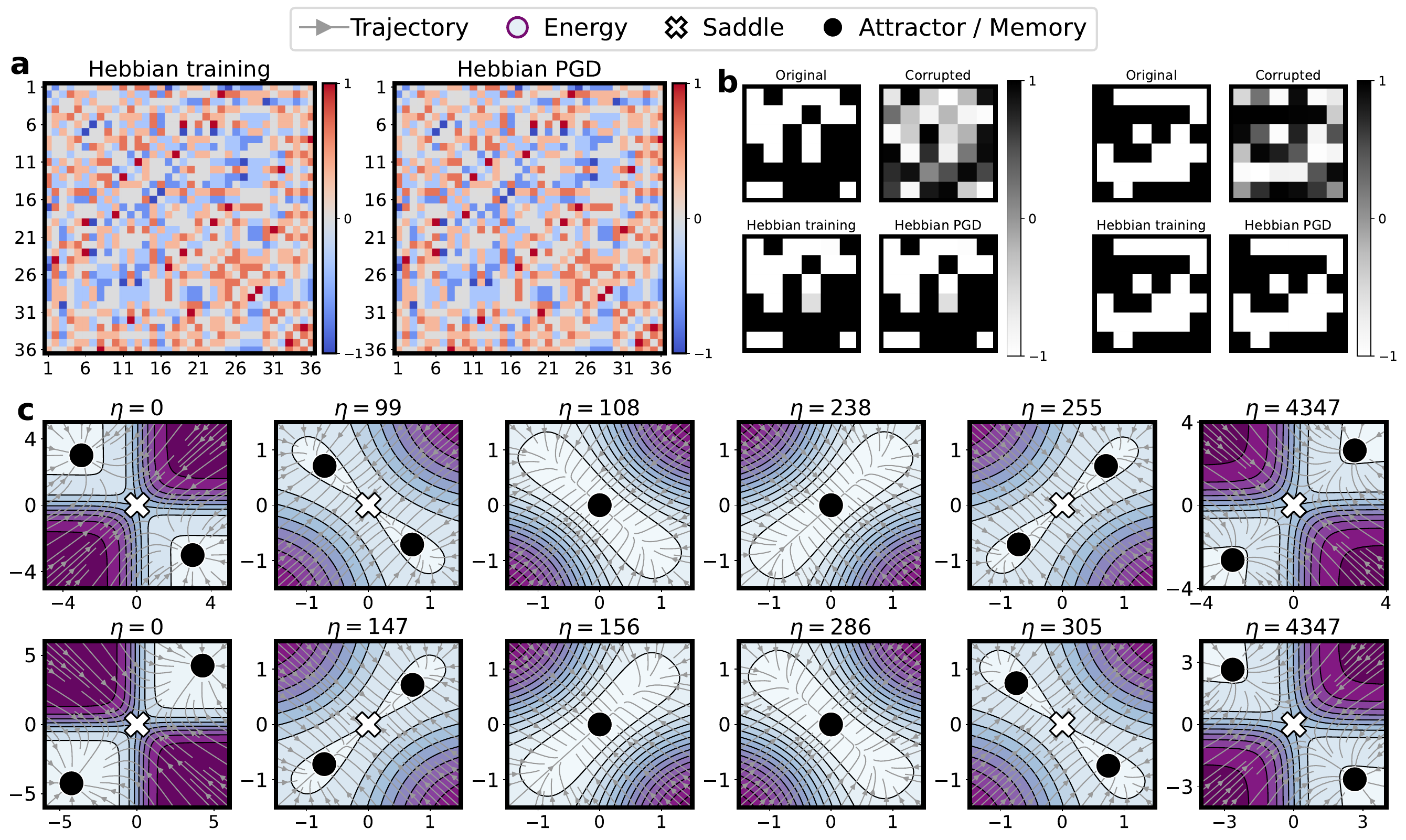}
        \caption{\textbf{Hebbian learning through projected gradient descent (\Cref{algorithm:hebbian-learning-pgd}).} \textbf{(a)} Synaptic connection matrix $W$ obtained by training a continuous-state Hopfield network with $36$ neurons to store six patterns in $\mathbb{R}^{36}$ using the standard Hebbian rule (left) and \Cref{algorithm:hebbian-learning-pgd} (right) -- the learning rate and convergence tolerance were set to $1e^{-3}$ and $1e^{-12}$, while the norm constraint was set as the norm of the weight matrix produced by the standard Hebbian rule. \textbf{(b)} Recall from two initial patterns corrupted with random Gaussian noise shows the concordance between the two approaches in \textbf{a}. Memory recall was performed by setting $R^{-1}=0.125$. \textbf{(c)} Hidden neuron trajectories (grey) and critical points of a one-parameter family of gradients (left to right, indexed by $\eta$) produced by training the network using \Cref{algorithm:hebbian-learning-pgd} and projecting on to two axes (top, axis 5 and 11; bottom, axis 21 and 22). Trajectories and critical points are projections by fixing the states of all other neurons to zero, and do not reflect the full dynamics. As the optimization index increases from the initial state $\eta_0 = 0$ to the final state $\eta_{\text{final}}=4347$, subcritical and supercritical saddle-node bifurcations occur.
        }
        \label{fig:supp-fig-algorithm-1}
\end{figure}

\vspace{1em}
\SetKwInput{KwInput}{Input}
\SetKwInput{KwOutput}{Output}
\begin{algorithm}[H]
\caption{Projected Gradient Descent Hebbian Learning}
\label{algorithm:hebbian-learning-pgd}
\SetAlgoLined
\KwInput{$\{\xi^1, \dots, \xi^M\}$: set of $M$ patterns; \\
         $\eta$: learning rate; \\
         $c$: Frobenius norm bound for projection constraint $\|W\|_F \leq c$; \\
         $\varepsilon$: tolerance for stopping condition.} 
\KwResult{Synaptic weight matrix $W$.}
\vspace{0.5em}
\textbf{Initialization:} \\
1. Randomly initialize weight matrix \( W \in \mathbb{R}^{N \times N} \). \\
2. Symmetrize $W$:
\hspace{1mm} $W \leftarrow \frac{W + W^T}{2}$ \\
3. Project $W$: \\
\If{$\|W\|_F > c$}{
        $W \leftarrow W \cdot \frac{c}{\|W\|_F}$ \tcp*{Project $W$ onto ball of radius $c$}
    }
\vspace{0.5em}
\While{not converged}{
    $W_{\text{prev}} \leftarrow W$ \tcp*{Store weight matrix}
    \vspace{0.5em}
    \textbf{Hebbian rule:} \\
    \For{each pattern $\xi^m$, $m=1, \dots, M$}{
        $W \leftarrow W + \eta \cdot \, \xi^m (\xi^m)^T$ \tcp*{Hebbian rule}
    }
    \vspace{0.5em}
    \textbf{Projection step:} \\
    \If{$\|W\|_F > c$}{
        $W \leftarrow W \cdot \frac{c}{\|W\|_F}$
    }
    \vspace{0.5em}
    \textbf{Stopping condition:} \\
    $\Delta W \leftarrow \|W - W_{\text{prev}}\|_F$ \\
    \If{$\Delta W < \varepsilon$}{
        \textbf{break}
    }
}
\end{algorithm}

\subsection{Denoising diffusion models}
The following figure shows the training and generation dynamics of the probability flow ODE of the denoising diffusion model from \colorAutoref{fig:figure-diffusion-model-generation} as a two-parameter family of gradients. The model was trained using the score matching objective, and the score function was parameterized by a three-layer multilayer perceptron with the softplus activation function and hidden dimensionality of 128. The score was calculated using automatic differentiation.

\begin{figure}[t!]
% \vspace{-1.5em}
        \centering
        \includegraphics[width=1\textwidth]{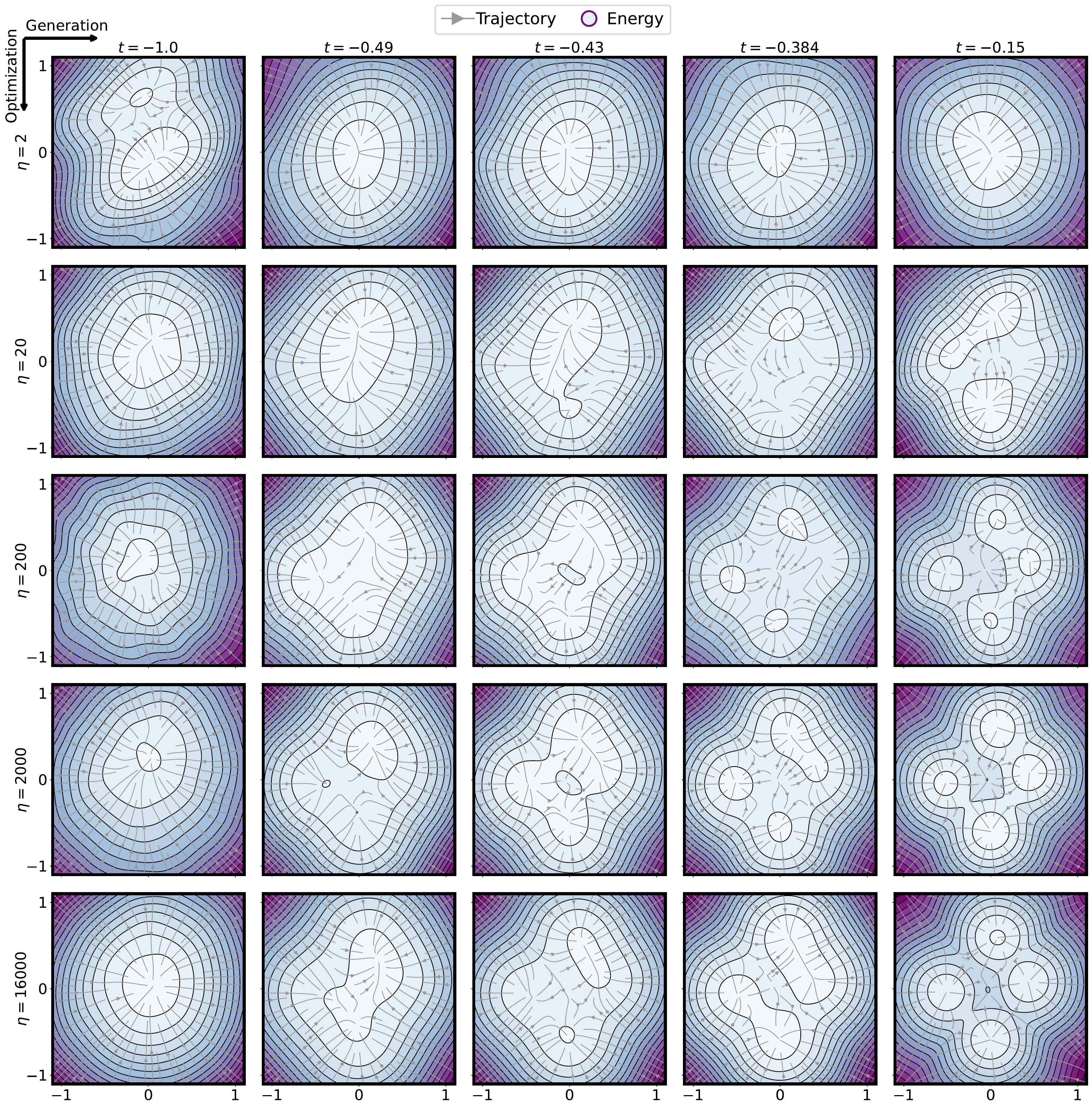}
        \caption{\textbf{Example of diffusion model generation and learning as a two-parameter family of gradients.} {Trajectories (grey) of a two-parameter family of gradients derived from the probability flow ODE of the variance-preserving diffusion model generating a data set with four centroids in $\mathbb{R}^2$ from \Cref{subsection:applications-diffusion-models-modern-hopfield}. The generation dynamics are obtained by solving the probability flow ODE backwards in time, with the energy of the time-varying potential depicted from left to right, indexed by time $t$, while the optimization dynamics are ordered from top to bottom by the gradient descent index $\eta$. Critical points are not shown.}
        }
        \label{fig:figure-training-generation-diffusion}
% \vspace{-1.5em}
\end{figure}
\newpage

\bibliography{main}
\end{document}